\documentclass[runningheads]{llncs}

\usepackage{eccv}
\usepackage{eccvabbrv}

\usepackage{graphicx}
\usepackage{booktabs}
\usepackage{array}

\usepackage{algorithm}
\usepackage{algpseudocode}

\usepackage[square,numbers]{natbib}

\usepackage[accsupp]{axessibility}

\usepackage{subcaption}   

\usepackage[pagebackref,breaklinks,colorlinks,citecolor=eccvblue]{hyperref}

\begin{document}

\title{SAMoE-VLA: A Scene Adaptive Mixture-of-Experts Vision-Language-Action Model for Autonomous Driving} 

\titlerunning{SAMoE-VLA}

\author{Zihan You\inst{1,2}\and
 Hongwei Liu\inst{1,3}\and
Chenxu Dang\inst{1,4} \and
Zhe Wang\inst{1}\and \\
Sining Ang\inst{1,5}\and
Aoqi Wang\inst{1,6}\and
Yan Wang\inst{1}\thanks{Corresponding author.}}
\authorrunning{Zihan You et al.}

\institute{Institute for AI Industry Research (AIR), Tsinghua University
\and
School of Instrument Science and Engineering, Southeast University
 \and
Zhili College, Tsinghua University
 \and
School of Artificial Intelligence and Automation, Huazhong University of Science and Technology
 \and
Department of Automation, University of Science and Technology of China
 \and
Department of Automation, University of Science and Technology Beijing}

\maketitle

\begin{abstract}
Recent advances in Vision-Language-Action (VLA) models have shown promising capabilities in autonomous driving by leveraging the understanding and reasoning strengths of Large Language Models (LLMs).
However, our empirical analysis reveals that directly applying existing token-level MoE mechanisms—which are inherited from LLM architectures—to VLA models results in unstable performance and safety degradation in autonomous driving, highlighting a misalignment between token-based expert specialization and scene-level decision-making.
To address this, we propose SAMoE-VLA, a scene-adaptive Vision-Language-Action framework that conditions expert selection on structured scene representations instead of token embeddings. Our key idea is to derive the MoE routing signal from bird’s-eye-view (BEV) features that encapsulates traffic scene context, enabling scenario-dependent expert weighting and merging tailored to distinct driving conditions. 
Furthermore, to support temporally consistent reasoning across world-knowledge, perception, language, and action, we introduce a Conditional Cross-Modal Causal Attention mechanism that integrates world state, linguistic intent, and action history into a unified causal reasoning process. 
Extensive experiments on the nuScenes open loop planning dataset and LangAuto
closed-loop benchmark demonstrate that SAMoE-VLA achieves state-of-the-art performance, outperforming prior VLA-based and world-model-based approaches with fewer parameters. Our code will be released soon.
\end{abstract}

\section{Introduction}
\label{sec:intro}
End-to-end autonomous driving, which aims to generate vehicle trajectories directly from raw sensor inputs, has achieved significant progress~\cite{jiang2023vad, hu2023planning, sun2025sparsedrive,liao2025diffusiondrive, xing2025goalflow}.
Supported by large-scale datasets and unified optimization objectives, these approaches produce driving behaviors that are both effective and smooth, achieving near–human-level performance in certain scenarios. However, when facing diverse urban layouts or unexpected interactions, these systems still lack the reasoning flexibility of human drivers~\cite{ chen2024end}.
Vision-Language Models (VLMs)~\cite{guo2025deepseek,bai2023qwen,touvron2023llama,chen2024internvl,team2023internlm,liu2023visual} and Vision-Language-Action (VLA)~\cite{kim2024openvla,intelligence2025pi_,zhou2025opendrivevla,zhou2025autovla} architectures extend this paradigm by introducing semantic reasoning and contextual understanding, enabling proactive decision-making. While most VLAs remain densely parameterized, the growing interest in conditional computation has motivated the adoption of Mixture-of-Experts (MoE) ,which has been proven highly effective in large language models~\cite{guo2025deepseek, shazeer2017outrageously}, to improve scalability and specialization while reducing model parameter size. 

\begin{figure*}[t] 
    \centering
    \includegraphics[width=0.75\textwidth]{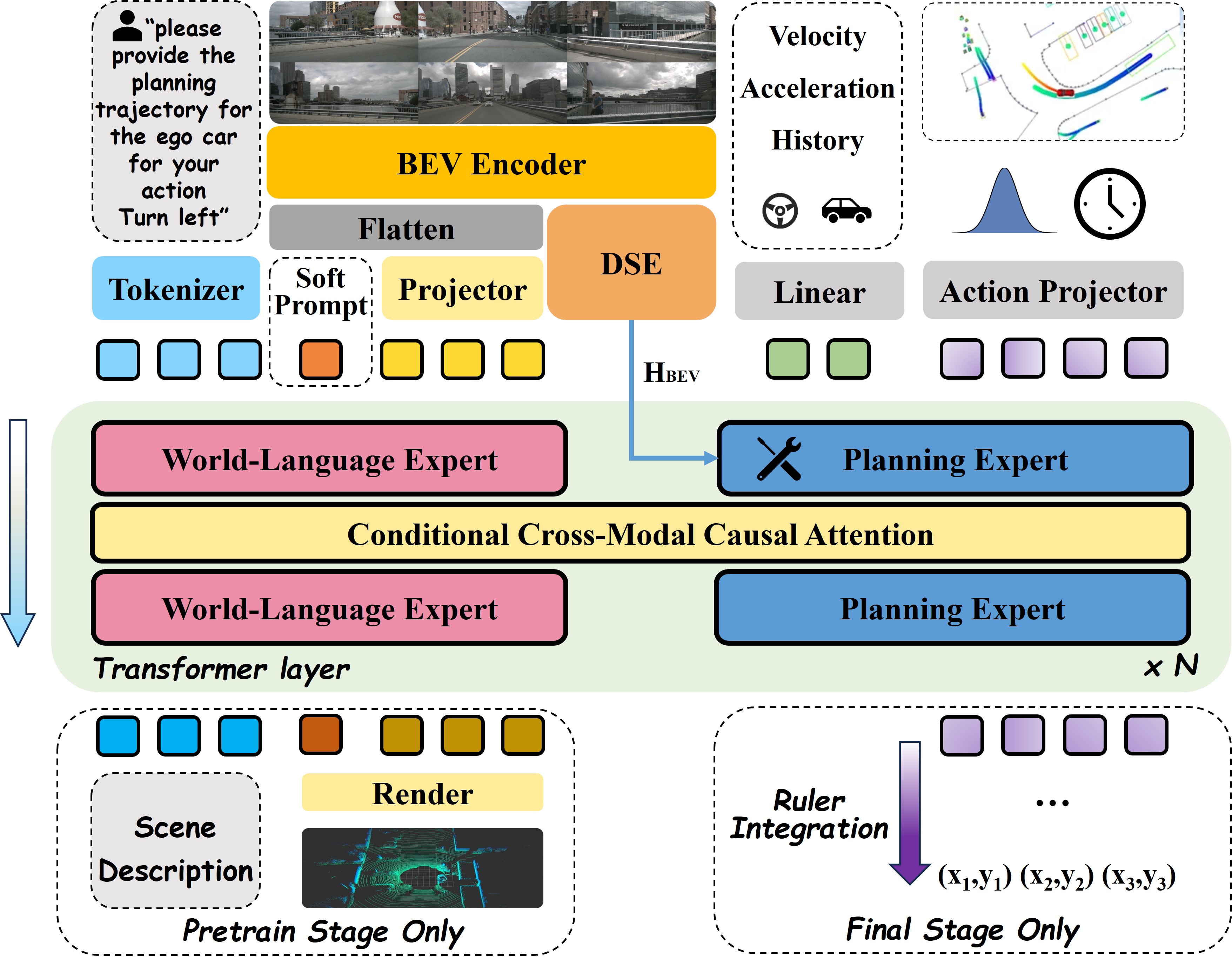} 
    \vspace{-3mm}
    \caption{Overview of our SAMoE-VLA. SAMoE-VLA employs two functional experts.
    \textbf{A World-Language Expert}: This module performs multimodal processing by integrating tokenized human instructions, Bird's-Eye-View (BEV) tokens and soft prompts for world embeddings.
    \textbf{A Planning Expert}: This expert utilizes a structure based on a scene adaptive Mixture-of-Experts (SAMoE) layers routed by the scene representation extracted from Deformable Scene Encoder and receives ego-state tokens and noisy action tokens as its input. Our model unifies these experts through Conditional Cross-Modal Causal Attention(CMCA).}
    \label{fig:pipeline}
\end{figure*}
In language modeling, token-level routing allows fine-grained expert specialization without violating structural constraints. In autonomous driving, however, decision-making is grounded in globally coupled scene semantics and temporally continuous world dynamics. 
Our experiments and empirical analysis in Fig.~\ref{fig:ladar} indicate that token-level routing and discrete expert selection, as used in soft and sparse MoE, can disrupt temporal causality and cross-modal coordination in driving planning. In particular, token-level sparse MoE increases the collision rate by 38.4\% compared to the dense baseline, resulting in inconsistent and unsafe trajectory generation in complex scenes. These findings suggest a granularity mismatch between token-level expert routing and scene-level embodied decision-making.
Moreover, previous explorations for autonomous driving~\cite{yang2025drivemoe,xu2025mose}  rely on manually defined router supervision, specialist annotations, and complex training regimes. While effective in controlled settings or benchmarks,  such designs couple expert behavior with predefined skill partitions, which may limit scalability and adaptability to diverse interaction patterns.

From this observation, we derive two design principles for MoE in  autonomous driving:
(i) expert routing should be conditioned on structured scene context rather than individual tokens;
(ii) cross-modal reasoning must preserve temporal causal consistency across world state, language intent, and action history.

Based on these principles, we propose SAMoE-VLA, a soft-weighted Mixture-of-Experts Vision-Language-Action model for scene-adaptive autonomous driving. Our Scene-Adaptive MoE(SA-MoE) performs differentiable expert fusion guided by bird’s-eye-view (BEV) scene representations that capture traffic geometry and interaction structure. This design enables coherent expert specialization at the scene level while maintaining smooth policy evolution. In addition, we introduce a Conditional Cross-Modal Causal Attention (CMCA) mechanism that unifies world, language, and planning representations under a temporally aligned causal mask, enabling consistent reasoning across modalities and time.

Experiments both on the nuScenes open loop planning dataset and LangAuto close loop benchmark show that SAMoE-VLA achieves state-of-the-art results while using  fewer parameters as shown in figure~\ref{fig:paramscla}, validating the benefits of BEV-based scene adaptive routing, soft-weight expert fusion, and causal world-language-action unification. 
We summarize our contributions as follows: 

\begin{itemize}
    \item We present SAMoE-VLA, a Mixture-of-Experts Vision-Language-Action framework that unifies the world, language, and planning spaces via a Conditional Cross-Modal Causal Attention mechanism.
    \item We introduce a BEV-guided Scene Adaptive MoE, routed by spatial traffic geometry captured from a lightweight Deformable Scene Encoder, enabling differentiable scene-aware expert fusion.
    \item Comprehensive experiments both on the nuScenes open loop planning dataset and LangAuto close loop benchmark demonstrate that SAMoE-VLA achieves state-of-the-art (SOTA) performance over existing VLAs and world model approaches.
\end{itemize}

\vspace{-6mm}
\section{Related Work}
\label{sec:related}
\vspace{-2mm}
\subsection{VLA for End-to-End Autonomous Driving}
\vspace{-1mm}
Recent progress in large vision-language models (VLMs) has boosted their use in end-to-end autonomous driving. Zhou et al.\cite{zhou2025opendrivevla} uses pre-trained VLMs for 3D-based driving actions with hierarchical alignment, leading in open-loop planning on nuScenes. Zhou et al.\cite{zhou2025autovla} adds action tokens to VLMs for policy learning, using adaptive reasoning and RL fine-tuning (GRPO) for better results. ORION~\citep{fu2025orion} includes a QT-Former for history, an LLM for reasoning, and a planner for actions. jiang et al.\cite{jiang2024senna} decouples VLM decisions (Senna-VLM) from E2E predictions (Senna-E2E). Li et al.\cite{li2025drive} and Li et al.\cite{li2025recogdrive} apply RL to VLMs for better reasoning, using rewards from trajectories and actions.
Advanced architectures like mixture-of-experts (MoE) handle variety. Yang et al.\cite{yang2025drivemoe} adds MoE to vision and actions for multi-view and rare cases. Instead of selected views input or relaying on predefined router labels, our model utilizes the BEV as input to provide global perspective, employs world modeling through the proposed cross attention mechansim, and softly weight and fuse the scene adaptive MoE without any pre-labeled router.
\vspace{-2mm}
\subsection{Mixture of Experts}
\vspace{-1mm}
Mixture-of-Experts (MoE)~\cite{cai2025survey} architectures enable efficient and scalable model capacity expansion by employing a gating function that adaptively routes or weights expert activations. Shazeer et al.\cite{shazeer2017outrageously} pioneered the introduction of the sparsely-gated Mixture-of-Experts layer, which significantly reduce resource consumption. Later extensions~\cite{zoph2022st,lepikhin2020gshard,riquelme2021scaling,lewis2021base,ma2018modeling}, such as Switch Transformer~\cite{fedus2022switch} and GLaM~\cite{du2022glam}, further refine this idea by simplifying routing, adding load-balancing losses and scaling to trillion-parameter regimes. 
Despite their success, sparse and Non-differentiable gating raises training instability~\cite{lewis2021base,rajbhandari2022deepspeed} and token dropping~\cite{antoniak2024mixture,fedus2022switch,zoph2022st}. SoftMOE~\cite{puigcerver2023sparse} illustrate soft assignment through mixing tokens to tackle those obstacles. Instead of merging tokens, SMEAR~\cite{muqeeth2023soft} and Lory~\cite{zhong2024lory} introduce expert merging to address the auto-regressive limitations often associated with SoftMOE. 
 In this work, we aim to propose a scalable and unrestricted Mixture-of-Experts (MoE) routing mechanism specifically tailored for autonomous driving.
\section{Method}
\vspace{-2mm}
In this section, we present the architecture and design details of our proposed SAMoE-VLA framework. We begin by describing the overall model architecture, followed by a detailed introduction to our Conditional Cross-Modal Causal Attention mechanism and soft weighted planning experts. Finally, we define our training objective and the specific training stages employed.
\vspace{-4mm}
\subsection{Method Architecture}
\vspace{-2mm}
Based on a pretrained large world vision language model~\cite{zhou2025hermes} and post-trained flow matching-based planning expert, our model unifies world space, language space and action space through Conditional Cross-Modal Causal Attention(CMCA) across a mixture of functional experts.  
SAMoE-VLA, as depicted in Figure \ref{fig:pipeline}, comprises two main specialized functional experts across the transformer backbone: (1) A world-language expert that processes language tokens from human instruction, BEV tokens generated from BEV encoder~\cite{li2024bevformer} and groups of soft prompts derived from BEV features and combined with learnable embeddings; (2) a planning expert structured with Mixture-of-Experts (MoE) layers that takes ego state tokens and noisy action tokens as inputs.
Instead of relying on a single feed-forward network (FFN) with limited representational capacity to capture semantic reasoning, 3D geometry, and control information simultaneously, the use of specialized experts prevents the formation of overly blended or averaged representations.

\textbf{World-Language Expert} takes Bird’s-Eye View (BEV) representation as input and provides powerful spatial semantic understanding, instruction following, and reasoning capabilities.
While previous works either restrict the input to particular camera viewpoints \cite{zhou2025autovla} or employ routers to select input views\cite{yang2025drivemoe}, which introduces security vulnerabilities, the use of BEV could provide a global perspective and avoid security challenges. This expert also accepts soft prompts as world tokens and enrich their representation to forecast future scenes 3D point clouds. The adoption of a World Vision-Language Model, in lieu of a common VLM, offers advantages by conferring the ability to predict future states and mitigating information losses typically associated with projection layer~\cite{li2025lost}. This capability is crucial for enabling the model to gain a comprehensive understanding of the generated driving behaviors and the subsequent consequences these actions may entail. Additional implementation details are provided in Appendix B.

\textbf{Planning Expert} structured with SA-MoE generates the planning trajectories via a flow matching mechanism, which further supervises the planning expert by requiring it to predict the velocity field that transports noisy intermediate actions toward ground-truth trajectories.  Further details are provided in Section~\ref{DSE}, Section~\ref{train_stages} and Appendix A.
\vspace{-4mm}
\subsection{Conditional Cross-Modal Causal Attention}
\vspace{-1mm}
Conditional Cross-Modal Causal Attention, also known as CMCA, is designed to enable temporally causal action generation while conditioning on heterogeneous contextual modalities, including world, language, and ego-state representations. CMCA treats conditioning tokens as a temporally static and globally visible context that does not participate in the autoregressive update dynamics of action tokens. This asymmetric information flow ensures that contextual representations serve as stable memory sources rather than evolving sequence elements during action generation.
Our conditional formulation decomposes the token sequence into two disjoint partitions: the \emph{conditioning context} (BEV tokens, world tokens, language tokens, ego state tokens) and the \emph{action tokens} to be generated.
Let $\mathbf{A} \in {0,1}^{B\times L\times L}$ denote the binary attention mask, where $\mathbf{A}_{b,i,j}=1$ indicates that the $i$-th token can attend to the $j$-th token in the $b$-th sequence.
We define the mask as
\begin{equation}
\mathbf{A}_{b,i,j} =
\begin{cases}
1, & j \in \mathcal{C}, \\
1, & i,j \in \mathcal{A} \text{ and } j \le i, \\
0, & \text{otherwise.}
\end{cases}
\end{equation}
where $\mathcal{C}$ indexes the conditioning tokens shared across all autoregressive steps.
This design allows every action token to attend to all conditioning tokens and its own historical actions while conditioning tokens neither attend to action tokens nor update through autoregressive decoding steps, thus realizing a temporally causal yet contextually grounded generation process.


To efficiently compute attention under structured visibility constraints, we integrate the masking operation directly into the similarity matrix prior to normalization. Specifically, the masked attention output is formulated as:
\vspace{-1mm}
\begin{equation}
\mathbf{O}
=
\operatorname{softmax}
\Big(
\mathbf{S}
+
(1 - \mathbf{A}) \cdot (-M)
\Big)
\, \mathbf{V},
\label{eq:masked_attention}
\end{equation}
\vspace{-1mm}
where $\mathbf{A}_{b,i,j} \in \{0,1\}$ denotes the structured visibility mask, and $M \rightarrow +\infty$ is a sufficiently large constant that effectively suppresses disallowed positions. Under this formulation, entries with $\mathbf{A}_{b,i,j}=1$ retain their original similarity scores, while entries with $\mathbf{A}_{b,i,j}=0$ are shifted toward negative infinity prior to the softmax operation, ensuring zero attention weight after normalization. $\mathbf{V}$ is represented as the concatenation results of each value states process from each domain expert. This causal alignment further ensures compatibility with the flow matching objective, which requires temporally coherent intermediate states to learn a smooth velocity field for trajectory refinement.

\vspace{-4mm}
\subsection{Scene Adaptive MoE with Deformable Scene Encoder }
\vspace{-2mm}
\label{DSE}
\begin{figure*}[t] 
    \centering
    \includegraphics[width=1.0\textwidth]{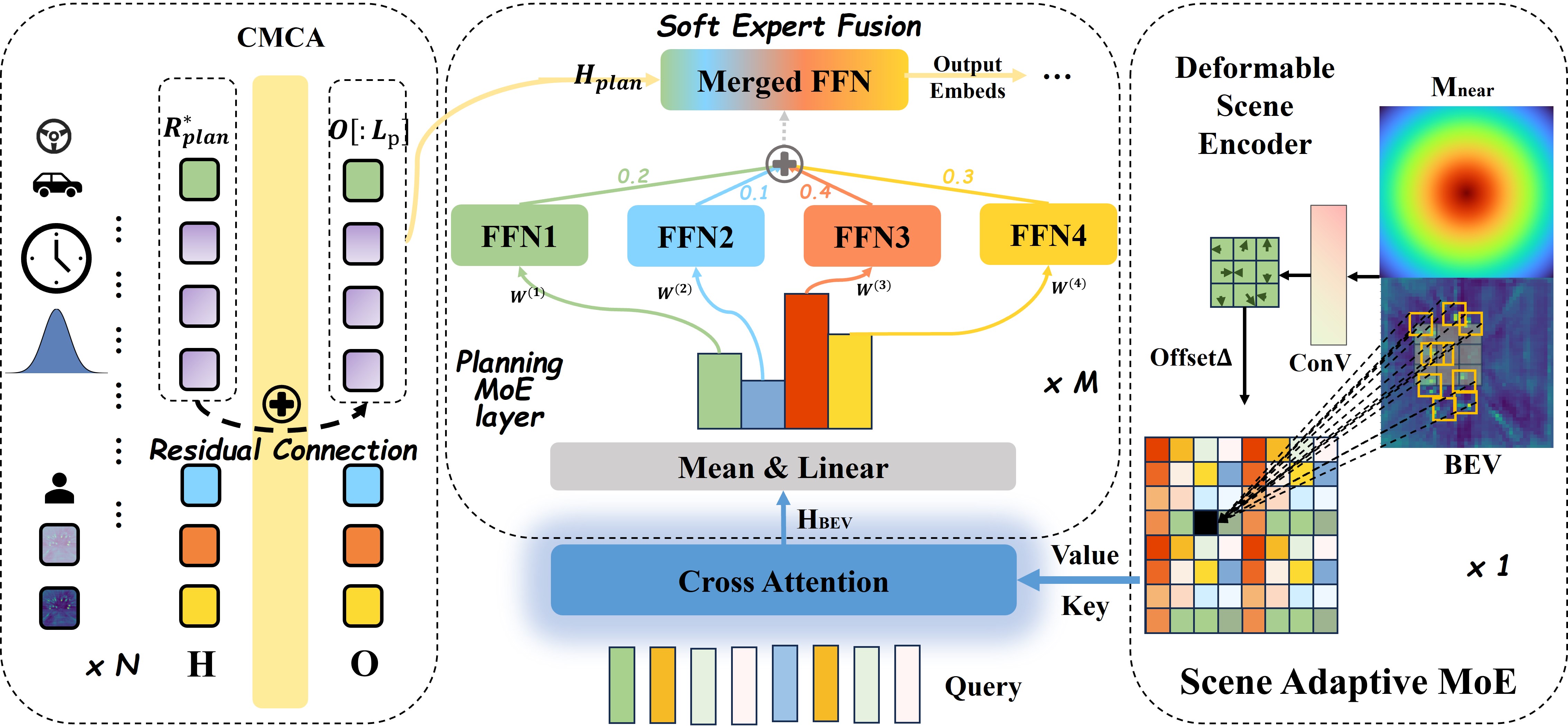} 
    \vspace{-1mm}
    \caption{Overview of our Scene Adaptive MoE guided by Deformable Scene Encoder. SA-MoE is the layer of our proposed planning expert shown in figure~\ref{fig:pipeline}. BEV hidden is calculated only once during inference, while expert weights need to be calculated in every layer. Each layer has its own Linear head in Eq. (8) so weights vary per layer.
}
    \label{fig:DSE}
    \vspace{-6mm}
\end{figure*}
To endow the planning expert with adaptive and specialized decision-making abilities across diverse driving situations, 
We propose a Planning Expert based on a Scene-Adaptive Mixture-of-Experts (MoE) architecture. 
Unlike conventional dense FFNs that average diverse behavioral patterns into a single representation, 
Our design decomposes the policy into multiple sub-experts, each specializing in a subspace of driving behaviors, and the planning expert's parameters during inference are dynamically weighted and fused based on different driving scenarios without any predefined labels.
The key contribution lies in a soft BEV-guided scene adaptive expert routing mechanism which is both fully differentiable and dynamic according to the current traffic context.

\textbf{Scene Adaptive Routing via Deformable Scene Encoding.} Given a BEV feature map 
$\mathbf{F}_{\text{BEV}} \in \mathbb{R}^{B \times H \times W \times C}$, 
we extract routing logits using a lightweight \textbf{D}eformable \textbf{S}cene \textbf{E}ncoder, 
termed \textbf{DSE}. 
This module augments BEV features with a distance-guided deformable convolution
to capture geometric and contextual cues of current traffic scenarios. Conventional convolutions apply a fixed kernel at each spatial location, assuming uniform information density across the feature map~\cite{dai2017deformable}. Such fixed sampling fails to capture the highly anisotropic spatial relevance present in driving scenes, so these
critical features near the ego vehicle are sampled with the same density as far-away background pixels. This leads to semantic dilution and inefficient feature utilization when routing MoE experts. To solve this problem, we demonstrate the deformable encoding mechanism as follows.

Given $\mathbf{F}_{\text{BEV}} \in \mathbb{R}^{B \times C \times H \times W}$, we generate spatial offsets $\mathbf{\Delta} \in \mathbb{R}^{B \times 2K^2 \times H \times W}$ 
through a learned predictor conditioned on both BEV features and a normalized distance map $\mathbf{M_{near}}$:

\begin{equation}
M_{\text{near}}(i,j) = 1 - 
\frac{\mathrm{dist}\!\left( (i,j), (c_x,c_y) \right)}
{\mathrm{dist}_{\max}} \in [0,1],
\end{equation}
where $(c_x,c_y)$ is the ego-center of the BEV and K is the size of the convolutional kernel used in the deformable convolution. $\mathbf{M_{near}}$ attains higher values near the ego location, thereby serving as a near-field attention prior.

The offset field for deformable convolution is predicted by:
\begin{equation}
\Delta = \mathcal{P}\left( [\mathbf{F}_{\text{BEV}}; M_{\text{near}}] \right),
\quad
\Delta \in \mathbb{R}^{B \times 2K^2 \times H_c \times W_c},
\end{equation}
where $\mathcal{P}(\cdot)$ is a small convolutional predictor initialized to zeros which ensures that the model begins with an identity convolution and gradually learns meaningful offset patterns during training. These learnable offsets allow the receptive field to dynamically adapt to scene geometry and local feature saliency,
enabling the network to focus on semantically informative regions.

The deformable convolution operation aggregates features under the learned spatial shifts:
\begin{equation}
\mathbf{S}_{\text{BEV}} = \text{Norm}(\text{Flatten}(\text{DeformConv}(\mathbf{F}_{\text{BEV}}, \mathbf{\Delta}))),
\end{equation}
capturing fine-grained geometry and semantic layouts across the driving scene.  
We then flatten and normalize the features to obtain a sequence representation 
$\mathbf{S}_{\text{BEV}} \in \mathbb{R}^{B \times N \times C}$ . 

To extract routing-aware embeddings, 
We introduce learnable query vectors 
$\mathbf{Q} \in \mathbb{R}^{1 \times T \times C}$ 
and apply multi-head cross attention:
\begin{equation}
\mathbf{H}_{\text{BEV}} = \text{MHA}(\mathbf{Q}, \mathbf{S}_{\text{BEV}}, \mathbf{S}_{\text{BEV}}),
\end{equation}
yielding the BEV-guided hidden states
$\mathbf{H}_{\text{BEV}} \in \mathbb{R}^{B \times T \times C}$, 
which summarize traffic semantics and spatial configurations relevant to driving decisions.
For the entire planning expert, $\mathbf{H}_{\text{BEV}}$  is computed only once through DSE outside the transformer layers, thus incurring almost no additional computation time.

\textbf{Soft Expert Weighted and Fused.} Previous works focus on LLM extending MoE by selecting tokens to the top-k expert or softly weighted the tokens to each experts\cite{puigcerver2023sparse}. While experiments in Figure~\ref{fig:ladar} suggest these mechanism are ill-suited for autonomous driving as they introduce the higher probability of collision. 
Instead of token-level specialization, we employ a soft weighted strategy that softly weights and fuses the experts. To be more specific, we calculate the weight of each expert routed by the signals encoded in DSE and merge these experts into a single FFN. Given $\mathbf{H}_{\text{BEV}}$ in each MOE layers, 
we compute routing logits as:
\begin{equation}
\mathbf{r} = \text{Linear}(\text{MeanPool}(\mathbf{H}_{\text{BEV}})) 
\in \mathbb{R}^{B \times E},
\end{equation}
where $E$ is the number of experts.  
Softmax normalization yields the expert weights:
\begin{equation}
\boldsymbol{\pi} = \text{softmax}(\mathbf{r}), 
\quad 
\pi_e = \frac{\exp(r_e)}{\sum_{e'} \exp(r_{e'})}.
\end{equation}
 Linear layer enables the expert weights in each MoE layer to be guided by not only current traffic scenario but also allowing flexible weighting based on the layer's position.
These weights determine the contribution of each expert to the final action embedding.

Each expert $e \in \{1, \dots, E\}$ 
comprises a lightweight feed-forward transformation parameterized by 
$\{ \mathbf{W}_1^{(e)}, \mathbf{W}_2^{(e)}, \mathbf{W}_3^{(e)} \}$, 
analogous to the MLPs in transformer blocks.  
For attention output $\mathbf{O} \in \mathbb{R}^{B \times L \times C}$ generated from CMCA as shown in formula~\ref{eq:masked_attention},  

The Hidden states of planning expert is computed as:

\begin{align}
\mathbf{H}_{\mathrm{plan}} &= \mathbf{O}{[:L_p]} + \mathbf{R_{\mathrm{plan}}^*}
\end{align}
where $L_p$ serves as the sum length of state tokens and action tokens for planning, $\mathbf{O}{[:L_p]}$ serves as the planning part of the attention output, $\mathbf{R_{\mathrm{plan}}^*}$ represents the residual connection,$\mathbf{H}_{\mathrm{plan}}$ represents the hidden states in planning expert.

We then calculate the weight of each expert and fuse these experts into a single FFN:
\begin{align}
\tilde{\mathbf{W}}_i &= \sum_{e=1}^{E} \pi_e\, \mathbf{W}_i^{(e)}, 
\quad i \in \{1,2,3\} 
\end{align}
The merged FFN $\tilde{\mathbf{W}}_i$ then replaces the original FFN layer and calculates the output embeds with the input of $\mathbf{H}_{\mathrm{plan}}$:
\begin{align}
\mathbf{H}_{\mathrm{plan}}' &= 
\sigma(\mathbf{H}_{\mathrm{plan}}\tilde{\mathbf{W}}_1)
\odot (\mathbf{H}_{\mathrm{plan}}\tilde{\mathbf{W}}_3)
\tilde{\mathbf{W}}_2
\end{align}
where $\sigma(\cdot)$ denotes the activation function SiLU
and $\odot$ represents element-wise multiplication. 
Further theoretical analysis and mathematical proofs, including representation capacity, temporal causality stability, convergence advantages, and gradient stability of SAMoE, is provided in Appendix E. Computational efficiency and deployment analysis is presented in Appendix I, with corresponding pseudocode given in Appendix F.
\vspace{-4mm}
\subsection{Training Stages and Objective }
\vspace{-1mm}

\label{train_stages}
We train SAMoE-VLA in two main stages with different training objectives. For the pretraining stage, the planning expert is frozen while the world-language expert is exclusively trained. This training is executed by masking all state and action tokens and proceeding through the sequential sub-steps outlined in \cite{zhou2025hermes}.  The overall training objective for the pretraining stage combines language modeling loss for unified scene understanding and the point cloud reconstruction loss for world prediction. 
The flow-matching loss \(\mathcal{L}_{\text{flow}}(\theta)\)used in final stage teaches the model to estimate velocity fields, guiding noisy action interpolations toward true action paths, all based on a multimodal setup \(\mathcal{C}\) that includes BEV tokens, language inputs, world details, and ego-state data. 
It supports creating smooth distributions of actions, boosting the model's skill in delivering accurate, situation-specific forecasts for step-by-step choices in changing settings.
The final training stage is partitioned into two distinct steps. In the first step, the model is trained without the Mixture-of-Experts (MoE) layer to ensure training stability. For the second step, the weights established during this initial phase are directly used to initialize the individual sub-expert weights within the soft-weighted MoE layer of the planning expert. We provide further training details and pseudo-code in Appendix D.1 and Appendix F.
\vspace{-2mm}

\vspace{-3mm}
\section{Experiments}
\vspace{-3mm}
\subsection{Datasets and Implementation Details}
\vspace{-1mm}

We evaluate SAMoE-VLA on challenging nuScenes~\cite{caesar2020nuscenes} dataset which is a large-scale, comprehensive collection comprising 1,000 diverse driving scenes, recorded across different cities, and featuring a total of 280,000 meticulously annotated frames. In addition to nuScenes, we conduct closed-loop evaluation on the LMDrive dataset and the LangAuto benchmark\cite{shao2024lmdrive} built on the CARLA simulator\cite{dosovitskiy2017carla}, which provide approximately 64K multimodal instruction–sensor–control clips collected across 8 towns under diverse weather and lighting conditions. LangAuto evaluates language-conditioned closed-loop driving by testing an agent’s ability to follow natural-language navigation and notice instructions across varying route lengths and adversarial settings.



The SAMoE-VLA architecture integrates OpenCLIP-ConvNext\cite{cherti2023reproducible} as the base image encoder, employing a CPFPN\cite{ma2023mrisnet} to generate multi-scale feature maps. Subsequently, BEVFormer is adopted as the Bird's-Eye-View (BEV) encoder, which processes a single input frame and encodes the resultant BEV features to a dimension of 256 and a spatial resolution of $\mathbf{200 \times 200}$. Further details about world-language expert align with \cite{zhou2025hermes}. The model is trained using 8 NVIDIA A800 GPUs. 
For the planning experts, the model parameters are initialized from the pre-trained InternVL2-2B.
Motivated by the design in \cite{fedus2022switch}, We replace the feed-forward layer with our proposed SA-MoE layer at every fourth Transformer layer. 
Action tokens are embedded through a linear projection, fused with sinusoidal time embeddings (with minimum period $4\times10^{-3}$ and maximum period $4.0$), and subsequently processed by a lightweight MLP to capture temporal dependencies.
During the final stages, we use full parameter fine-tuning instead of Lora with the learning rate of 1e-4. 
For flow matching, Gaussian noise $\boldsymbol{\varepsilon} \sim \mathcal{N}(0, \mathbf{I})$ and a time coefficient $\tau \sim \text{Beta}(1.5, 1.0)$ (clamped to $[0.001, 0.999]$) are sampled to form the noisy action representation. During inference, the ODE steps to produce the trajectories is 10. Additional implementation details about flow matching are provided in Appendix A.
\vspace{-4mm}
\subsection{Main Results}
\vspace{-1mm}

As shown in Table~\ref{tab:planning_comparison}, SAMoE-VLA achieves the lowest average L2 error (0.29 m) under our open-loop evaluation protocol, outperforming the previous VLA state of the art and the world model–based approach PreWorld by a $7\%$ relative reduction. While SAMoE-VLA is not consistently the best at the 1-second horizon, this behavior reflects its planning-oriented design, which prioritizes scene-level consistency and long-horizon feasibility over short-term motion fitting. Consequently, its advantages become increasingly evident at longer horizons. In particular, SAMoE-VLA achieves a 3-second L2 error of 0.35 m, which is 15\% lower than the best VLA baseline and 5\% lower than the strongest world model–based method, demonstrating its effectiveness in mitigating long-horizon error accumulation. Most VLA-based approaches outperform traditional end-to-end planners, while world model–based methods further improve long-horizon accuracy. Importantly, SAMoE-VLA maintains a favorable accuracy–safety trade-off, achieving the best average collision rate of 0.26\%, outperforming world model–based and traditional planning methods. Qualitative results are provided in Appendix~H.
\begin{table*}[t]
\centering
\footnotesize
\caption{Comparison of planning performance on the nuScenes dataset. L2 Error is measured in meters and Collision Rate in percentage. Lower is better for both metrics. SAMoE-VLA reports total parameter count.}
\setlength{\tabcolsep}{4pt}

\begin{tabular}{lcccc@{\hspace{8pt}}cccc}
\toprule

\textbf{Method} 
& \multicolumn{4}{c}{\textbf{L2 (m)}} 
& \multicolumn{4}{c}{\textbf{Collision (\%)}} \\

\cmidrule(lr){2-5} \cmidrule(lr){6-9}

& 1s & 2s & 3s & Avg & 1s & 2s & 3s & Avg \\

\midrule
\multicolumn{9}{c}{\textbf{Traditional End2End \& World Model-based}} \\
\midrule

Ego-MLP\cite{li2024ego} 
& 0.46 & 0.76 & 1.12 & 0.78 
& 0.21 & 0.35 & 0.58 & 0.38 \\

UniAD\cite{hu2023planning} 
& 0.20 & 0.42 & 0.75 & 0.46 
& 0.02 & 0.25 & 0.84 & 0.37 \\

VAD\cite{jiang2023vad} 
& 0.17 & 0.34 & 0.60 & 0.37 
& 0.04 & 0.27 & 0.67 & 0.33 \\

Drive-WM\cite{wang2024driving} 
& 0.43 & 0.77 & 1.20 & 0.80 
& 0.10 & 0.21 & 0.48 & 0.26 \\

Drive-occWorld\cite{yang2025driving} 
& 0.17 & 0.31 & 0.49 & 0.32 
& 0.02 & 0.24 & 0.62 & 0.29 \\

PreWorld\cite{li2025semi} 
& 0.22 & 0.30 & 0.40 & 0.31 
& 0.21 & 0.66 & 0.71 & 0.53 \\

\midrule
\multicolumn{9}{c}{\textbf{VLA \& VLM-based}} \\
\midrule

OpenEMMA (7B)\cite{xing2025openemma} 
& 1.45 & 3.21 & 3.76 & 2.81 
& -- & -- & -- & -- \\

ORION (7B)\cite{fu2025orion} 
& 0.17 & 0.31 & 0.55 & 0.34 
& 0.05 & 0.25 & 0.80 & 0.37 \\

Omnidrive (7B)\cite{wang2024omnidrive} 
& 0.14 & 0.29 & 0.55 & 0.33 
& 0.00 & 0.13 & 0.78 & 0.30 \\

EMMA\cite{hwang2024emma} 
& 0.14 & 0.29 & 0.54 & 0.32 
& -- & -- & -- & -- \\

\textbf{SAMoE-VLA (3.6B)} 
& 0.24 & \textbf{0.27} & \textbf{0.35} & \textbf{0.29} 
& 0.01 & 0.18 & 0.60 & \textbf{0.26} \\

\bottomrule
\end{tabular}

\label{tab:planning_comparison}

\vspace{-3mm}
\end{table*}

\begin{table}[ht]
\label{tab:planning_comparison_langauto}
\centering
\caption{Overall performance comparison on the close loop LangAuto benchmark under three evaluation settings:LangAuto, LangAuto-Short, and LangAuto-Tiny. Higher values indicate better performance.DS (Driving Score) reflects the overall driving quality by aggregating task completion and safety penalties; IS (Infraction Score) measures compliance by penalizing traffic violations; RC (Route Completion) quantifies the percentage of the assigned route successfully completed. SAMoE-VLA reports total parameter count.}
\small
\begin{tabular}{l|ccc|ccc|ccc}
\toprule
Method & \multicolumn{3}{c|}{LangAuto} & \multicolumn{3}{c|}{LangAuto-S} & \multicolumn{3}{c}{LangAuto-T} \\
\cmidrule(lr){2-4} \cmidrule(lr){5-7} \cmidrule(lr){8-10}
 & DS    & RC    & IS   & DS    & RC    & IS   & DS    & RC    & IS   \\
\midrule
LMdrive(7B)\cite{shao2024lmdrive}   & 36.2  & 46.5  & 0.81 & 50.6  & 60.0  & 0.84 & 66.5  & 77.9  & 0.85 \\
AD-H(7B)\cite{zhang2024ad}  & 44.0  & 53.2  & 0.83  & 56.1  & 68.0  & 0.78  & 77.5  & 85.1  & 0.91 \\
DSDrive(1B)\cite{liu2505distilling}  & 29.5  & 39.3  & 0.77  & 62.0  & 76.1  & 0.81  & 60.6  & 72.5  & 0.84 \\
BEVDriver(7B)\cite{winter2025bevdriver} & 48.9  & 59.7  & 0.82 & 66.7  & 77.8  & 0.87 & 70.2  & 81.3  & 0.87 \\

\textbf{SAMoE-VLA}(3.6B)    & \textbf{51.4} & \textbf{63.5}  & 0.81 & \textbf{69.5}  & \textbf{80.9}  & 0.86 & \textbf{79.5} & \textbf{86.4}  & \textbf{0.92} \\
\bottomrule
\end{tabular}
\vspace{-4mm}
\end{table}
Table 2 summarizes the closed-loop performance on LangAuto\cite{shao2024lmdrive} under three evaluation settings. On the full LangAuto benchmark, SAMoE-VLA achieves the best Driving Score (51.4) and Route Completion (63.5), outperforming all 7B baselines while maintaining competitive instruction alignment. The advantage persists in LangAuto-Short, where SAMoE-VLA again ranks first in DS (69.5) and RC (80.9), demonstrating stronger short-horizon planning stability. Under LangAuto-Tiny, it attains the highest Instruction Score (0.92) while achieve the best Driving Score of 79.5, indicating robust language grounding in simplified scenarios. Despite its smaller parameter scale, SAMoE-VLA consistently delivers superior or balanced performance across execution quality and instruction compliance.

\vspace{-4mm}
\subsection{Ablation Study}
\vspace{-2mm}

\label{ex:comparsion_between_Existing}
\textbf{Comparison to Existing MoE Mechanism.} To examine the impact of different existing MoE Mechanism, we adopt five distinct baseline architectures.
These baselines are systematically categorized into three groups based on their specialization approach: the non-MoE version, token-level specialization \cite{shazeer2017outrageously,muqeeth2023soft}, and expert-weighted specialization. As shown in figure~\ref{fig:ladar}, we find that deploying our soft-weighted SAMoE, which performs the best, could significantly decrease L2 error from 0.32 to 0.29 while nearly maintaining a moderate collision rate(0.25) compared with dense counterparts(0.25).  Although the sparse routing mechanism could decrease L2 error from 0.32 to 0.30, it increases 38.4\% of collision rate, thus raises a security challenge. Also, we observe that applying soft routing to action tokens severely degrades model performance, with the average L2 error increasing to 0.46 and the average collision rate rising to 0.57.

\begin{figure}[t]
    \centering
    \begin{minipage}[t]{0.48\textwidth}
        \centering
        \includegraphics[width=\textwidth]{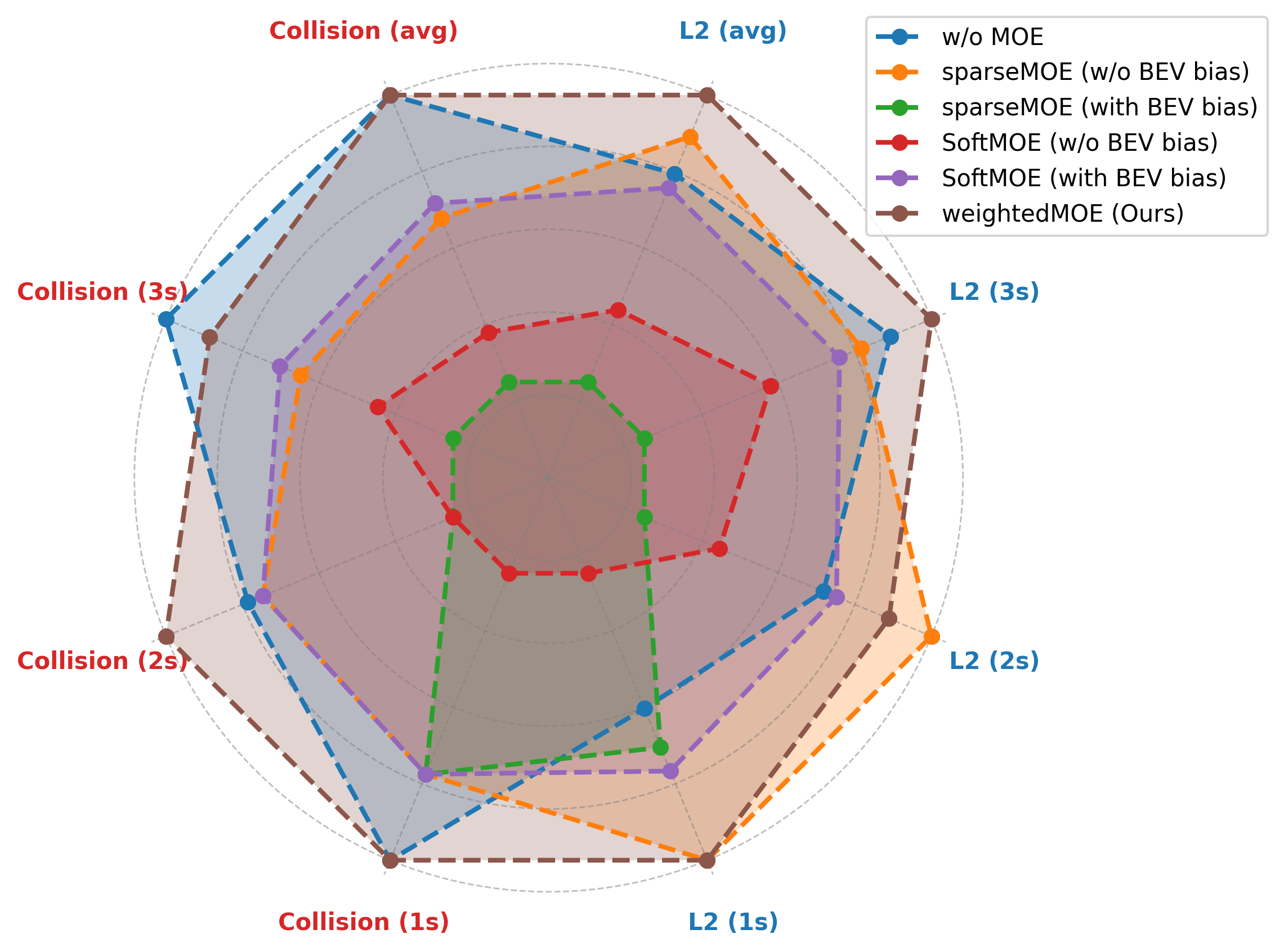}
        \vspace{-8mm}
        \caption{Radar chart of different MoE mechanism experiment results. Note that the distribution places smaller values near edges.}
        \label{fig:ladar}
    \vspace{-5mm}
    \end{minipage}
    \hfill
    \begin{minipage}[t]{0.50\textwidth}
        \centering
        \includegraphics[width=\textwidth]{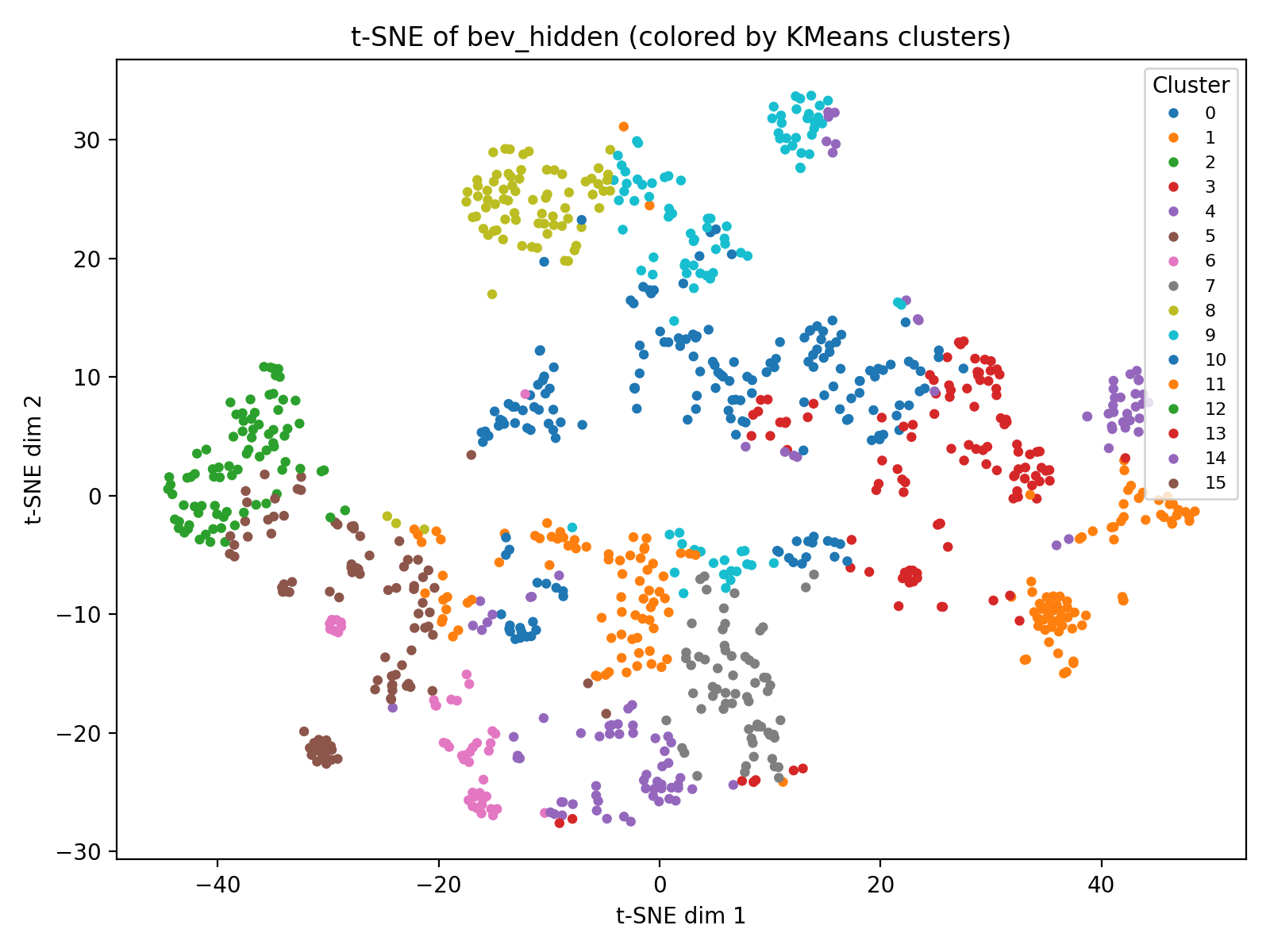}
        \vspace{-8mm}
        \caption{
        t-SNE visualization of mean-pooled BEV hidden representations, colored by KMeans clusters, showing distinct feature regimes and corresponding MoE routing preferences.}
        \label{fig:tnse}
        
    \vspace{-5mm}
    \end{minipage}
\end{figure}

Additionally, to analyze the impact of the scene adaptive for token-based MoE routing, we implement two baseline that add additional routing bias derived from DSE. We find that incorporating scene bias into sparse MoE leads to a drastic performance drop, with the average L2 error rising to 0.57 and the collision rate reaching 0.69. The data further indicate a notable performance trade-off: incorporating scene bias into the soft-MoE architecture allows the model to nearly maintain the L2 error performance, yet it results in a detrimental effect on safety, evidenced by a significant increase in the collision rate to 0.34. The aforementioned results indicate that token-level routing is ill-suited for planning tasks in autonomous driving. Details and numerical results comparing with existing MoE approaches are provided in Appendix C, with a computational cost results in Appendix I.

\textbf{Importance of scene adaptive routing Mechanism and DSE.} Further experimentation is conducted to investigate the importance of DSE and driving scene adaptive routing by integrate additional baseline. As illustrated in Table~\ref{tab:abl_route}, we first compare our scene adaptive routing strategy with an alternative prefix routing strategy. In prefix routing, expert merging is only guided by the first token of planning expert instead of BEV hidden vector captured by DSE. We find that the average L2 error for prefix routing is 0.31. While this demonstrates an improvement over Non-MoE version, its overall performance is inferior when compared directly to our proposed method by 6.8\%. Deformable convolution network(DCN) is also removed to validate its necessity. Compared with our proposed method, L2 error is 0.31, which is also 6.8\% lower.
We also performed an ablation study by removing the normalized distance map ($\mathbf{M_{near}}$), which yielded a comparable, albeit slightly inferior, performance with an L2 error of $0.30$ compared to our final model.
The results above indicates the importance of our proposed scene adaptive routing mechanism and deformable scene encoder.
\begin{table}[t]
\centering
\small
\setlength{\tabcolsep}{3.5pt}
\renewcommand{\arraystretch}{1.1}
\caption{Comparison of L2 distance and collision rate for different routers.}
\vspace{-2mm}
\begin{tabular}{lcccccccc}
\toprule
Method & \multicolumn{4}{c}{L2 $\downarrow$ (m)} & \multicolumn{4}{c}{Collision $\downarrow$ (\%)} \\
\cmidrule(lr){2-5} \cmidrule(lr){6-9}
 & 1s & 2s & 3s & Avg & 1s & 2s & 3s & Avg \\
\midrule
Prefix Routing        & 0.22 & 0.29 & 0.43 & 0.31 & 0.00 & 0.25 & 0.86 & 0.37 \\
w/o DCN           & 0.26 & 0.28 & 0.37 & 0.31 & 0.02 & 0.21 & 0.76 & 0.33 \\
w/o $\mathbf{M_{near}}$  & 0.25 & \textbf{0.26} & 0.38 & 0.30 & 0.01 & 0.20 & 0.66 & 0.29 \\
Ours                  & \textbf{0.24} & 0.27 & \textbf{0.35} & \textbf{0.29} 
                      & \textbf{0.01} & \textbf{0.18} & \textbf{0.60} & \textbf{0.26} \\
\bottomrule
\label{tab:abl_route}
\end{tabular}
\vspace{-8mm}
\end{table}

\textbf{Impact of number of experts.} We conduct experiments with different numbers of experts in our MoE architecture. As shown in Figure~\ref{fig:numexpert}, we find that 12 experts produce the lowest L2 error of 0.28, outperforming other results, while 4 experts achieve the lowest collision rate of 0.26.
The results indicate that employing either fewer or more experts than optimal fails to achieve notable performance improvements relative to the dense baseline model; only an appropriate number of experts can yield superior performance.
We observe that fewer experts enhance robustness against collisions.  4 experts version significantly reduced the collision rate (0.26) compared to the two-expert version and the average collision rate exhibited a subsequent increase with the further growth in the total number of experts. Conversely, L2 errors decreases steadily with increasing number of experts, but 16 experts configuration performs the average L2 error of 0.33.  
These results underscore the importance of expert scaling in MoE designs for planning tasks, where excessive specialization (e.g., 16 experts) may introduce instability, while moderate scaling (4--12 experts) optimizes multi-objective performance.


\textbf{Impact of Scene-Adaptive MoE on Challenging Scenarios.} To further validate the scene-adaptive capability, we conducted an ablation study across three representative driving scenarios, which is selected from nuScenes: complex intersections, narrow-turn roads, and close-range overtaking. As shown in Figure~\ref{fig:drivescene}, replacing the Scene adaptive MoE layer with a dense counterpart consistently leads to degraded planning accuracy and success rate across all settings. In complex intersections, the proposed soft weighted scene adaptive MoE reduces the average L2 distance from 0.33 m to 0.27 m and higher the success rate from 67.91\% to 77.41\%, indicating more stable trajectory prediction under high interaction density. The advantage becomes more pronounced in close-range overtaking, where the MoE version achieves a substantial reduction in average L2 (0.28 to 0.19 m) and 12\% of success rate improvement, reflecting enhanced adaptability to rapidly changing relative motion. These results confirm that the scene-aware routing mechanism effectively allocates expert capacity according to environmental complexity. This experiment also serves as an empirical study of scene-level generalization, showing that the proposed SAMoE-VLA maintains stable planning quality across structurally distinct scenarios, demonstrating robustness in complex layouts. Due to space limitations, supplementary selection detiails are deferred to Appendix G.

\begin{figure}[t]
    \centering
    \begin{minipage}[t]{0.48\textwidth}
        \centering
        \includegraphics[width=\textwidth]{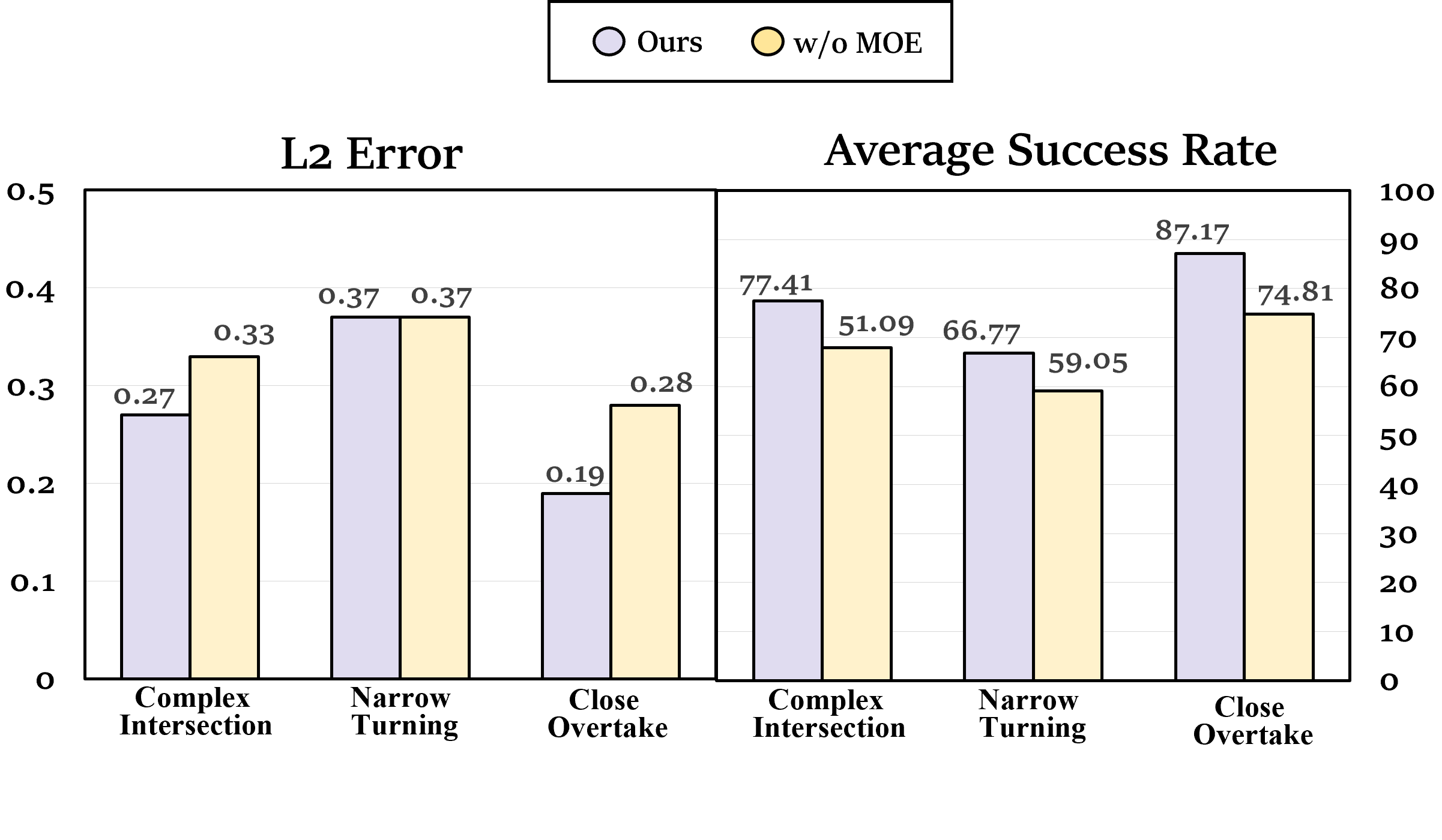}
        \vspace{-8mm}
        \caption{Comparison of average L2 distance and success rate with and without the scene adaptive soft weighted MoE layer.}
        \label{fig:drivescene}
    \vspace{-1mm}
    \end{minipage}
    \hfill
    \begin{minipage}[t]{0.50\textwidth}
        \centering
        \includegraphics[width=1.1\textwidth]{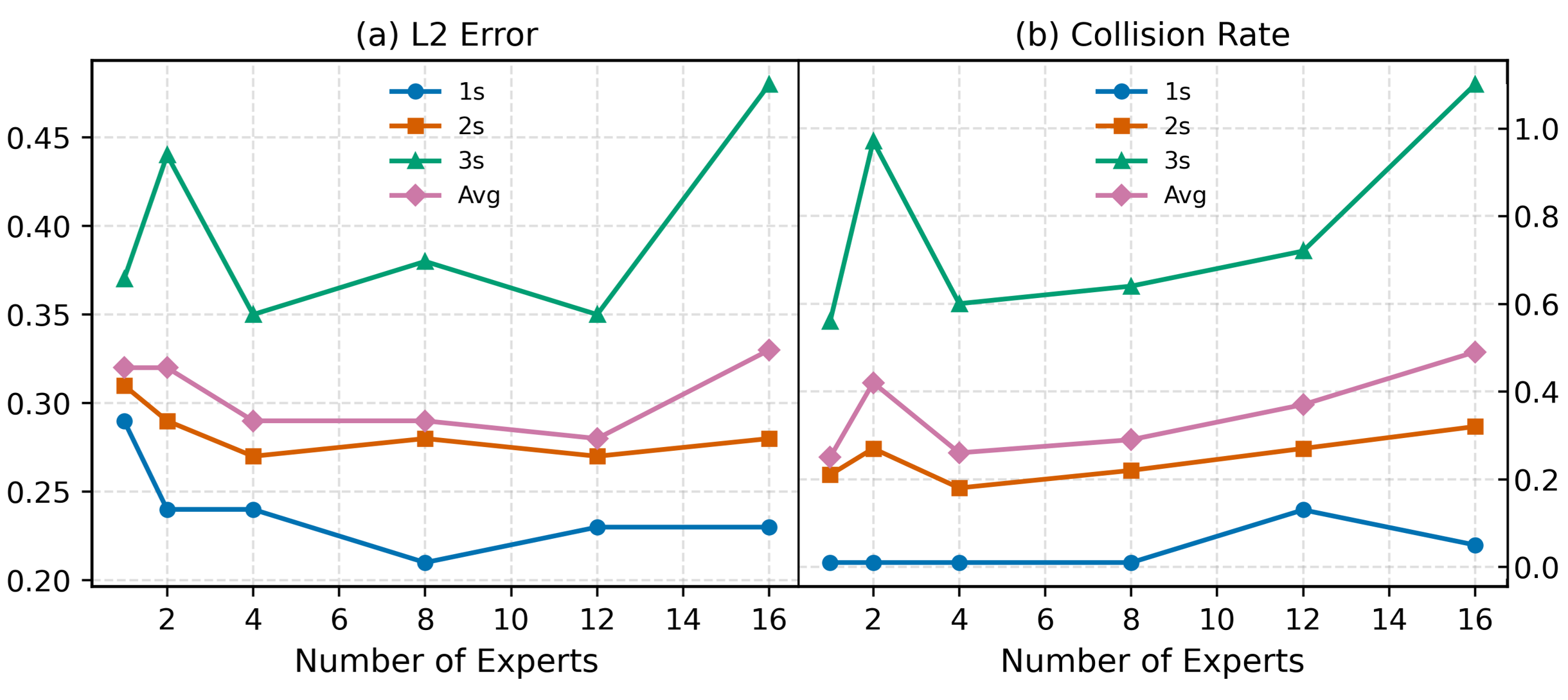}
        \vspace{-8mm}
        \caption{
        Performance comparison with varying numbers of experts. The left (a) reports the L2 error, while the right (b) illustrates the corresponding collision rates. }
        \label{fig:numexpert}
    \vspace{-1mm}
    \end{minipage}
\vspace{-7mm}
\end{figure}
\textbf{Importance of Conditional Cross-Modal Causal Attention.} We validate the importance of CMCA by introduce two additional experiments. 
Without employing a conditional autoregressive mechanism in which every action token could be seen by each other, our model will lead to an performance decline. The average L2 error exceed to 0.35, which indicates the success of our mask strategy. 
We also present a fully dense version where all the tokens (including BEV tokens and noisy action tokens) are processed in a single feed-forward layer in each transformer layer.
Experiments for the fully dense version suggest L2 error is 0.03 higher compared with our baseline, showing the importance of specialized experts tailored to each distinct spatial application.
To avoid interference, all the experiments have not deployed any SA-MoE layer.

Further experiment illustrate the importance of 3D-aware world modeling pretraining.
 Instead of training planning experts after future point cloud prediction, we simply finish the alignment steps and then train our planning expert. As shown in figure~\ref{fig:CMCA}, we find that the average L2 error drop to 0.33, further implcates the role in world modeling.

\textbf{Scene-Adaptive Route behavior Analysis via t-SNE.} To further interpret the scene adaptive route’s behavior and its ability to classify traffic scenes, we visualize and cluster the hidden representations extracted from the  Deformable Scene Encoder(DSE). Each sample-level feature vector is obtained by mean-pooling over the batch and temporal dimensions, followed by two-dimensional projection using t-SNE for local neighborhood visualization. 
Figure~\ref{fig:tnse} illustrates that samples form several loosely separated clusters, with dense intra-cluster regions and partially overlapping boundaries, suggesting the presence of distinct yet related BEV scene patterns. These findings demonstrate that the DSE encodes semantically meaningful driving and traffic patterns, which correspond to distinct expert routing behaviors and provide evidence that the model’s soft-weighted MoE adapts dynamically to diverse scene contexts.
\vspace{-2mm}
%

\vspace{-2mm}
\section{Conclusion}
\vspace{-2mm}
\begin{figure}[t]
    \centering
    \begin{minipage}[t]{0.64\textwidth}
        \centering
        \includegraphics[width=1.015\textwidth]{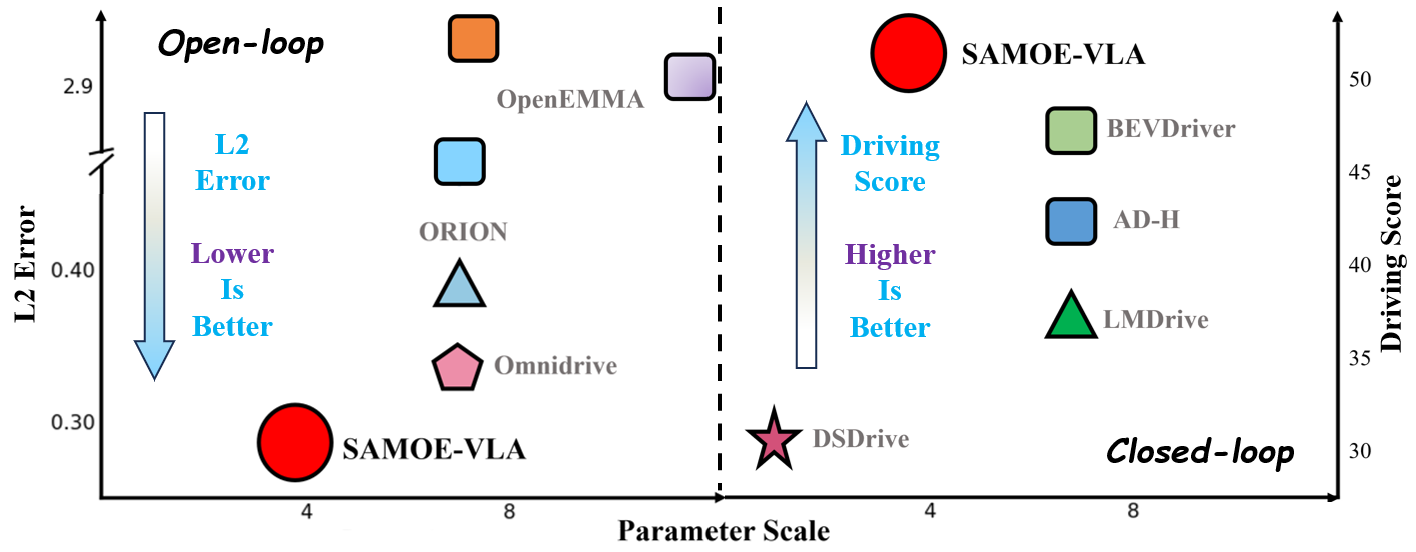}
        \vspace{-6mm}
        \caption{
        L2 error and driving score are compared across VLA models with different parameter scales. Lower L2 error and higher drive score indicate better performance. OpenEMMA includes three model sizes, while SAMoE-VLA reports total parameter count.
        }
        \label{fig:paramscla}
    \vspace{-3mm}
    \end{minipage}
    \hfill
    \begin{minipage}[t]{0.35\textwidth}
        \centering
        \includegraphics[width=\textwidth]{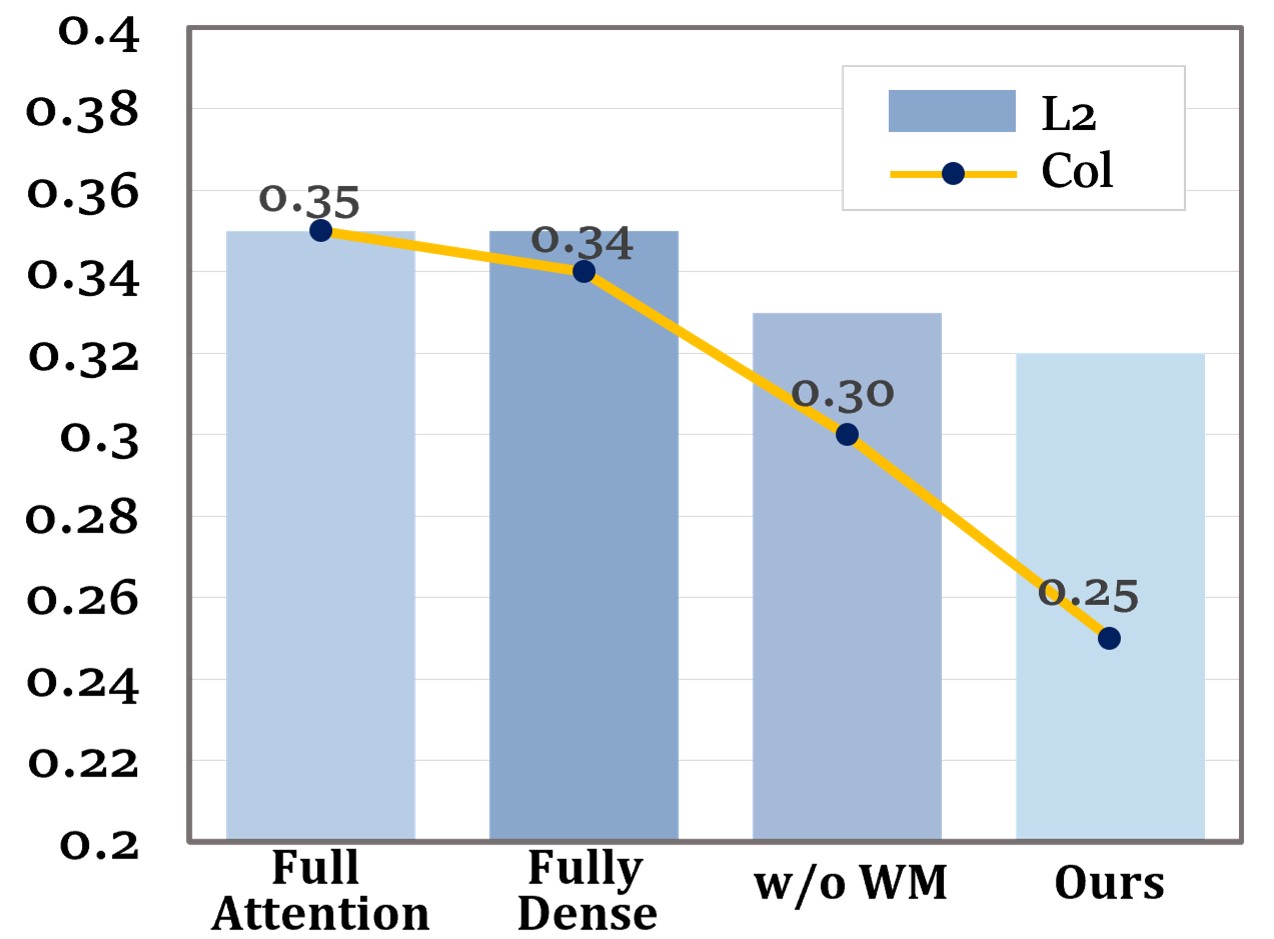}
        \vspace{-6mm}
        \caption{
       Comparison of L2 distance and collision rate for fully dense, full attention, without world modeling}
        \label{fig:CMCA}
    \end{minipage}
\vspace{-4mm}
\end{figure}

In this work, we introduce SAMoE-VLA, a scene-adaptive Mixture-of-Experts Vision-Language-Action model for Autonomous Driving.
To address the limitations of dense architectures and token-level routing in existing VLA and MoE-based approaches, we proposed two key mechanisms: Conditional Cross-Modal Causal Attention (CMCA), which unifies world, language, and action representations, and a BEV-guided Scene Adaptive MoE (SA-MoE), which employs deformable scene encoding to achieve differentiable and driving-scene-aware expert fusion. Comprehensive experiments on both open and close planning benchmark demonstrate that SAMoE-VLA achieves state-of-the-art performance, lower long-horizon prediction error, and superior robustness across complex driving scenarios, validating the effectiveness of soft expert weighting, scene adaptive routing, and world modeling based on CMCA.


\bibliographystyle{splncs04}
\bibliography{main}
\clearpage
\setcounter{page}{1}
\setcounter{section}{0}

\renewcommand{\thesection}{\Alph{section}}  
\begin{center}


{\Large \textbf{SAMoE-VLA: A Scene Adaptive Mixture-of-Experts Vision-Language-Action Model for Autonomous Driving}}
\vspace{5mm}

{\large {Supplementary Material}}

\vspace{5mm}
\end{center}
\section{The implementation of Flow-Matching}

\paragraph{Overview.}
As illustrate in~\ref{DSE}, we propose a conditional trajectory planning expert trained via flow matching. The system learns a time-conditioned velocity field to transport noise to ego-trajectory over a finite horizon.

\subsection{Perception and Conditioning}
Given multi-view images, a BEV encoder produces a tensor
\[
\mathbf{B}\in\mathbb{R}^{H\times W\times C_b}.
\]
Deformable Scene Encoder $\mathbf{DSE}:\mathbb{R}^{C_b}\!\to\!\mathbb{R}^{D}$ maps BEV feature to $D$-dimensional tokens 
\[
\mathbf{H}_{\mathrm{BEV}}=\mathbf{DSE}(\mathrm{reshape}(\mathbf{B}))\in\mathbb{R}^{(HW)\times D}.
\]
We also embed auxiliary contexts: ego-history $\mathbf{h}$, CAN bus $\mathbf{c}$, and low-level control features $\mathbf{l}$ via linear projections
\[
\mathbf{S}=\mathrm{concat}\!\big(\psi_{\mathrm{his}}(\mathbf{h}),\;\psi_{\mathrm{bus}}(\mathbf{c}),\;\psi_{\mathrm{ego}}(\mathbf{l})\big)\in\mathbb{R}^{L_s\times D}.
\]
The overall conditioning is $\mathcal{C}=\{\mathbf{z}, \mathbf{H}_{\mathrm{bev}},\mathbf{S}, \mathbf{L},\mathbf{W}\}$.
Where the $\mathbf{z}$ denotes BEV tokens,$\mathbf{L}$ denotes language instruction tokens , $\mathbf{W}$ denotes world tokens which are processed in world language expert.
\subsection{Flow Matching Formulation}
Let the planning horizon be $K$, ground-truth future actions be
\[
\mathbf{a}=(\mathbf{a}_1,\dots,\mathbf{a}_K)\in\mathbb{R}^{K\times d_a},\quad d_a=2.
\]
where $K$ in NuScenes is 6. 
We sample noise and time as
\[
\boldsymbol{\epsilon}\sim\mathcal{N}(\mathbf{0},\mathbf{I}),\qquad
t\sim \mathrm{Beta}(\alpha{=}1.5,\beta{=}1.0)\cdot 0.999+0.001.
\]
Define the convex interpolation and target velocity:
\[
\mathbf{x}_t = t\,\boldsymbol{\epsilon} + (1-t)\,\mathbf{a},\qquad
\mathbf{u}_t = \boldsymbol{\epsilon} - \mathbf{a}.
\]
We compute a sinusoidal time embedding $\gamma(t)\in\mathbb{R}^{D}$ over periods $[4\times 10^{-3},\,4]$ and construct action--time tokens by projecting actions and fusing time:
\[
\mathbf{E}_{\mathrm{act}}=\psi_{\mathrm{act}}(\mathbf{x}_t)\in\mathbb{R}^{K\times D},\quad
\mathbf{E}_{\mathrm{time}}=\gamma(t)\mathbf{1}_K\in\mathbb{R}^{K\times D},
\]
\[
\mathbf{E}_{\mathrm{suf}}=\mathrm{MLP}\big([\mathbf{E}_{\mathrm{act}}\Vert \mathbf{E}_{\mathrm{time}}]\big)\in\mathbb{R}^{K\times D}.
\]
The planning expert $f_\theta$ consumes the concatenated sequence
\[
\mathbf{X}=\big[\mathbf{S}\;\Vert\;\mathbf{E}_{\mathrm{suf}}\big]\in\mathbb{R}^{(L_s+K)\times D},
\]
and outputs hidden states whose last $K$ positions are projected by $\psi_{\mathrm{out}}:\mathbb{R}^{D}\!\to\!\mathbb{R}^{d_a}$ to produce the predicted velocity field $v_\theta$:
\[
v_\theta(\mathbf{x}_t, t, \mathcal{C})=\psi_{\mathrm{out}}\big(f_\theta(\mathbf{X})_{\mathrm{suffix}}\big)\in\mathbb{R}^{K\times d_a}.
\]

\paragraph{Training Objective.}
The flow-matching loss aligns $v_\theta$ with the target velocity:
\begin{equation}
\mathcal{L}_{\mathrm{FM}}
=\frac{1}{K}\sum_{k=1}^{K}\big\|v_\theta(\mathbf{x}_t, t, \mathcal{C})_k - \mathbf{u}_{t,k}\big\|_2^2.
\label{eq:fm-loss}
\end{equation}

\subsection{Inference via ODE Integration}
At test time, we discard ground-truth actions and solve the transport from noise to actions by explicit Euler:
\[
\mathbf{x}_1\sim\mathcal{N}(\mathbf{0},\mathbf{I}),\quad \Delta t=-\frac{1}{N},\quad t\leftarrow 1,
\]
\[
\text{repeat }N\text{ steps: }\quad
\mathbf{x}_{t+\Delta t}\leftarrow \mathbf{x}_t + \Delta t\cdot v_\theta(\mathbf{x}_t, t, \mathcal{C}),\quad t\leftarrow t+\Delta t.
\]
The final estimate $\hat{\mathbf{a}}=\mathbf{x}_0\in\mathbb{R}^{K\times d_a}$ is the planned trajectory.

\subsection{Implementation Details}
(i) BEV features are linearly projected to $D$-dimensional tokens and concatenated with state tokens to form $\mathcal{C}$. (ii) Time embedding uses sine-cosine features with periods in $[4\!\times\!10^{-3},4]$. (iii) Training samples $(\boldsymbol{\epsilon},t)$ per batch; the planner predicts $v_\theta$ over all $K$ steps, and the MSE in (\ref{eq:fm-loss}) is averaged across steps. (iv) For stability and efficiency, mixed precision and gradient checkpointing can be used.

\begin{algorithm}[t]
\caption{Training with Flow Matching}
\begin{algorithmic}[1]
\State Input: multi-view images, ego-history $\mathbf{h}$, CAN bus $\mathbf{c}$, control features $\mathbf{l}$, ground-truth $\mathbf{a}$.
\State Encode images to BEV $\mathbf{B}$; project to tokens $\mathbf{H}_{\mathrm{bev}}$; embed $\mathbf{S}$.
\State Sample $\boldsymbol{\epsilon}\!\sim\!\mathcal{N}(\mathbf{0},\mathbf{I})$, $t\!\sim\!\mathrm{Beta}(1.5,1.0)\cdot 0.999+0.001$.
\State Form $\mathbf{x}_t=t\boldsymbol{\epsilon}+(1-t)\mathbf{a}$,\; $\mathbf{u}_t=\boldsymbol{\epsilon}-\mathbf{a}$.
\State Build $\mathbf{E}_{\mathrm{suf}}$ with action--time fusion; sequence $\mathbf{X}=[\mathbf{S}\Vert\mathbf{E}_{\mathrm{suf}}]$.
\State Predict $v_\theta(\mathbf{x}_t,t,\mathcal{C})$ and minimize $\mathcal{L}_{\mathrm{FM}}$ in (\ref{eq:fm-loss}).
\end{algorithmic}
\end{algorithm}

\begin{algorithm}[t]
\caption{Inference by Euler Integration}
\begin{algorithmic}[1]
\State Input: multi-view images, ego-state/context $\mathcal{C}$; steps $N$.
\State Initialize $\mathbf{x}_1\!\sim\!\mathcal{N}(\mathbf{0},\mathbf{I})$, $\Delta t\!=\!-1/N$, $t\!=\!1$.
\For{$i=1$ to $N$}
  \State $\mathbf{x}_{t+\Delta t}\!\leftarrow\!\mathbf{x}_t+\Delta t\cdot v_\theta(\mathbf{x}_t,t,\mathcal{C})$;\; $t\!\leftarrow\!t+\Delta t$.
\EndFor
\State Output $\hat{\mathbf{a}}=\mathbf{x}_0$.
\end{algorithmic}
\end{algorithm}

\section{The implementation of World Modeling}
\label{para:world_modeling}

\paragraph{Overview}
Followed World language model HERMES\cite{zhou2025hermes}, We adopted a world language expert that predicts future 3D scene geometry by rendering LiDAR rays from Large Language Model (LLM)-conditioned BEV features.
Given BEV features via a BEV Encoder, compress them, and feed them into an LLM together with a compact set of learnable soft prompts.
The LLM returns a refined current BEV embedding and future-aware soft world tokens, which we use to synthesize per-frame future BEV embeddings.
We then lift each BEV into a unified voxel grid, and render depth along LiDAR rays to form predicted future point clouds.
The full system is implemented with memory-efficient mixed-precision and gradient checkpointing.
The expert is pretrained during first stage, and frozen in final stage.
\paragraph{Notation}
Let the BEV spatial shape be $(H, W)$ and channel dimension $C$.
We denote the unified voxel grid by $(Z, H, W)$ with vertical bins $Z$ and per-voxel size $\Delta = (\Delta_x,\Delta_y,\Delta_z)$ covering a 3D range $\mathrm{pc\_range}$.
We consider a frame set $\mathcal{T}=\{t, t\!+\!1, \ldots, t\!+\!K\}$ where $t$ is current and $K\!\ge\!1$ is the number of future steps.
For LiDAR rendering, a ray is a tuple $(\mathbf{o}_n,\mathbf{d}_n,r_n)$, with origin $\mathbf{o}_n\!\in\!\mathbb{R}^3$, unit direction $\mathbf{d}_n\!\in\!\mathbb{R}^3$, and ground-truth range $r_n\!\in\!\mathbb{R}_{+}$ when available.

 \subsection{BEV Feature Extraction}
From multi-view images, we compute the BEV embedding
\begin{align}
\mathbf{E}_t \in \mathbb{R}^{H_{\mathrm{bev}}\times W_{\mathrm{bev}}\times C}
\end{align}
using spatial cross attention(SCA) with learned queries and positional encodings:
\begin{align}
\mathbf{F}_{BEV} = \mathrm{SCA}(\text{img\_feats}, \text{BEV\_queries}, \text{pos})\\
\end{align}
For LLM conditioning, we downsample BEV channels with a lightweight module $\mathrm{DownSampleCross}$ :
\begin{align}
\mathbf{z},~\mathbf{E}^\mathrm{ae}_t = \mathrm{MLP}(\mathrm{DownSampleCross}(\mathbf{F}_{BEV})) .
\end{align}
Here $\mathbf{z}$ is the LLM input representation and $\mathbf{E}^\mathrm{ae}_t$ is the reconstruction target for self-regularization when enabled.

 \subsection{LLM-guided Future BEV Generation}
We compress $\mathbf{z}$ into $N_q$ learnable queries by channel-wise pooling and add temporal/control priors:
\begin{align}
\mathbf{Q} &= \mathrm{Pool}\big(\mathrm{permute}(\mathbf{z})\big)\in \mathbb{R}^{N_q\times C},\\
\mathbf{W}^{(k)} &= \mathbf{Q} + \mathbf{f}_{\mathrm{frame}}(k), \quad k\in\{1,\ldots,K\},
\end{align}
where $\mathbf{f}_{\mathrm{frame}}$ is learnable soft prompt embeddings,$\mathbf{W}^{(k)}$ is the world tokens.

We feed the LLM with the downsampled BEV tokens and world tokens:
\begin{align}
\big(\mathbf{E}^{\mathrm{llm}}_t, \{\mathbf{Q}^{(k)}\}_{k=1}^K, \ell_{\mathrm{chat}}\big) = \mathrm{LLM}\big(\mathbf{z}, \{\mathbf{W}^{(k)}\}_{k=1}^K\big),
\end{align}
where $\mathbf{E}^{\mathrm{llm}}_t$ is a refined current BEV embedding and $\mathbf{Q}^{(k)}$ are future-aware latent queries returned by the LLM; $\ell_{\mathrm{chat}}$ is language output.
Analogous to the HERMES, the language output is designed to articulate responses to user queries regarding the driving environment, specifically encompassing scene descriptions and answers to visually-derived questions.The language supervision and world supervision interact enhanced with each other and unifies the world understanding and future scene generation.
We then synthesize per-future-step BEV embeddings, followed by ego-attention refinement with $\mathbf{W}^{(k)}$:
\begin{align}
\widetilde{\mathbf{E}}^{\downarrow}_{t+k} &= \mathbf{E}^{\mathrm{llm}}_t + \mathbf{f}_{\mathrm{frame}}(k),\\
\mathbf{E}^{\downarrow}_{t+k} &= \mathrm{EgoAttn}\big(\widetilde{\mathbf{E}}^{\downarrow}_{t+k},~\mathbf{W}^{(k)}\big),  k=1,\ldots,K .
\end{align}
For completeness, the current-step BEV is given by $\mathbf{z}=\mathbf{E}^{\mathrm{llm}}_t$.

 \subsection{Lifting BEV to 3D Voxel Features}
Each $\mathbf{E}^{\downarrow}_{t+k}$ is lifted to a 3D voxel representation via a depth-aware BEV upsampler:
\begin{align}
\mathbf{V}_{t+k} &= \mathrm{BEVUpsample}\!\left(\mathbf{E}^{\downarrow}_{t+k}\right)
\in \mathbb{R}^{C\times Z \times H_{\mathrm{bev}}\times W_{\mathrm{bev}}},\\
\mathbf{U}_{t+k} &= \mathrm{Conv3D}\!\left(\mathbf{V}_{t+k}\right)
\in \mathbb{R}^{C_{\mathrm{out}}\times Z \times H_{\mathrm{bev}}\times W_{\mathrm{bev}}},
\end{align}
where $\mathrm{BEVUpsample}$ is a shallow 3D convolutions that reshapes the channel dimension into the vertical bins and upsamples from $Z/2$ to $Z$.
$\mathbf{U}_{t+k}$ is the unified 3D feature volume used for volumetric rendering.


\subsection{LiDAR Ray Sampling and SDF-based Rendering}

To recover fine-grained 3D geometry for each future frame $t+k$, we follow the LiDAR acquisition process and construct a set of rays from the target point cloud $P_{t+k}$. For each point $\mathbf{x}_n \in P_{t+k}$, we define a ray originating from the sensor center and passing through the point. Formally, we sample
\subsection{LiDAR Ray Sampling and SDF-based Rendering}

We construct LiDAR rays for each future frame $t+k$ from the target point cloud $P_{t+k}$.  
Each point $\mathbf{x}_n \in P_{t+k}$ defines one ray:

\begin{align}
\mathcal{R}_{t+k} = \{(\mathbf{o}_n, \mathbf{d}_n, r_n)\}_{n=1}^{N}. 
\end{align}

Here, $\mathbf{o}_n$ is the ray origin in the LiDAR coordinate frame:

\begin{align}
\mathbf{o}_n = \mathbf{0}. 
\end{align}

This sets the origin of every ray to the sensor center, consistent with the LiDAR emission model.

The ray direction $\mathbf{d}_n$ is computed by normalizing the vector from the origin to the observed point:

\begin{align}
\mathbf{d}_n = 
\frac{\mathbf{x}_n - \mathbf{o}_n}
     {\|\mathbf{x}_n - \mathbf{o}_n\|_2}. 
\end{align}

This yields a unit-length direction pointing from the LiDAR toward $\mathbf{x}_n$.

The raw LiDAR range $r_n$ is the Euclidean distance to the point:

\begin{align}
r_n = \|\mathbf{x}_n - \mathbf{o}_n\|_2. 
\end{align}

We apply range filtering $r_{\min} < r_n < r_{\max}$ and optionally mask the ego-vehicle region to remove invalid measurements.

\vspace{3pt}
\noindent\textbf{Ray-based SDF Rendering.}
Given the unified volumetric representation $\mathbf{U}_{t+k}$, we discretize each ray into $n$ sampled points:
\begin{align}
\mathbf{p}_i = \mathbf{o}_n + d_i \mathbf{d}_n,
\qquad i = 1,\dots,n,
\end{align}
where $d_i$ denotes the sampled distance along the ray. For each $\mathbf{p}_i$, we retrieve its feature embedding $\mathbf{f}_i$ from $\mathbf{U}_{t+k}$ using trilinear interpolation.

A shallow MLP $\phi_{\text{SDF}}$ predicts the signed distance value:
\begin{align}
s_i = \phi_{\text{SDF}}(\mathbf{p}_i, \mathbf{f}_i).
\end{align}
Following volumetric SDF rendering, the opacity at the $i$-th sample is computed as
\begin{align}
\alpha_i = \max\left(\frac{\sigma_t(s_i) - \sigma_t(s_{i+1})}{\sigma_t(s_i)},\, 0\right),
\sigma_t(x) = (1 + e^{-t x})^{-1},
\end{align}
where $\sigma_t(\cdot)$ is a learnable sigmoid modulation with parameter $t$.

The accumulated transmittance $T_i$ and the unbiased, occlusion-aware weight $w_i$ follow:
\begin{align}
T_i &= \prod_{j=1}^{i-1} (1 - \alpha_j), \\
w_i &= T_i\, \alpha_i.
\end{align}

The rendered depth for ray $(\mathbf{o}_n,\mathbf{d}_n)$ is then obtained via weighted integration:
\begin{align}
\widehat{r}_n^{(k)}
    = \sum_{i=1}^{n} w_i\, d_i.
\end{align}
Finally, the 3D point reconstructed by the renderer is
\begin{align}
\widehat{\mathbf{x}}_n^{(k)}
    = \mathbf{o}_n + \widehat{r}_n^{(k)} \mathbf{d}_n.
\end{align}

A global scale factor $s$ is applied consistently to all ray directions, sampled distances, and predicted depths to align the renderer’s internal metric space with real-world LiDAR geometry.

 \subsection{The pretrained Stage Training Objectives}
During the pretraining stage, we supervise rendered depth and point clouds using per-ray regression and set-level geometry consistency.
Let $\mathcal{X}^{(k)}=\{\mathbf{x}_n^{(k)}\}_{n=1}^{N}$ and $\widehat{\mathcal{X}}^{(k)}=\{\widehat{\mathbf{x}}_n^{(k)}\}_{n=1}^{N}$ be the ground-truth and predicted point sets:
\begin{equation}
\begin{aligned}
\mathcal{L}_{\mathrm{depth}} &= \frac{1}{KN}\sum_{k=1}^{K}\sum_{n=1}^{N}\big|\,r_n - \widehat{r}_n^{(k)}\,\big|,\\[6pt]
\mathcal{L}_{\mathrm{chamfer}} &= \frac{1}{K} \sum_{k=1}^{K} \Bigg[
  \frac{1}{|\widehat{\mathcal{X}}^{(k)}|} \sum_{\widehat{\mathbf{x}} \in \widehat{\mathcal{X}}^{(k)}} \min_{\mathbf{x} \in \mathcal{X}^{(k)}} \|\widehat{\mathbf{x}} - \mathbf{x}\|_2^2 \\[6pt]
&\quad + \frac{1}{|\mathcal{X}^{(k)}|} \sum_{\mathbf{x} \in \mathcal{X}^{(k)}} \min_{\widehat{\mathbf{x}} \in \widehat{\mathcal{X}}^{(k)}} \|\mathbf{x} - \widehat{\mathbf{x}}\|_2^2 \Bigg].
\end{aligned}
\end{equation}
The renderer also returns auxiliary losses $\mathcal{L}_{\mathrm{render}}$ (e.g., density/color regularizers).
When conversational pretraining is enabled, an LLM supervision $\mathcal{L}_{\mathrm{chat}}$ is added.
The full objective is
\begin{align}
\mathcal{L} = \lambda_{\mathrm{render}}\,\mathcal{L}_{\mathrm{render}}
+ \lambda_{\mathrm{depth}}\,\mathcal{L}_{\mathrm{depth}}\\
+ \lambda_{\mathrm{ch}}\,\mathcal{L}_{\mathrm{chamfer}}
+ \lambda_{\mathrm{chat}}\,\mathcal{L}_{\mathrm{chat}} .
\end{align}

 \subsection{Implementation Details}
We use half-precision for BEV tensors entering the LLM and enable gradient checkpointing for the LLM, BEV upsampler, and rendering convolutions to reduce memory.
The ego-attention in $\mathrm{DownSampleCross}$ is activated only when $N_q>0$, allowing the LLM's future-aware queries to modulate per-future BEV embeddings.
All per-frame volumes share the same renderer $f_\theta$ and voxelization parameters to ensure temporal consistency.

\begin{figure*}[t] 
    \centering
    \includegraphics[width=0.8\textwidth]{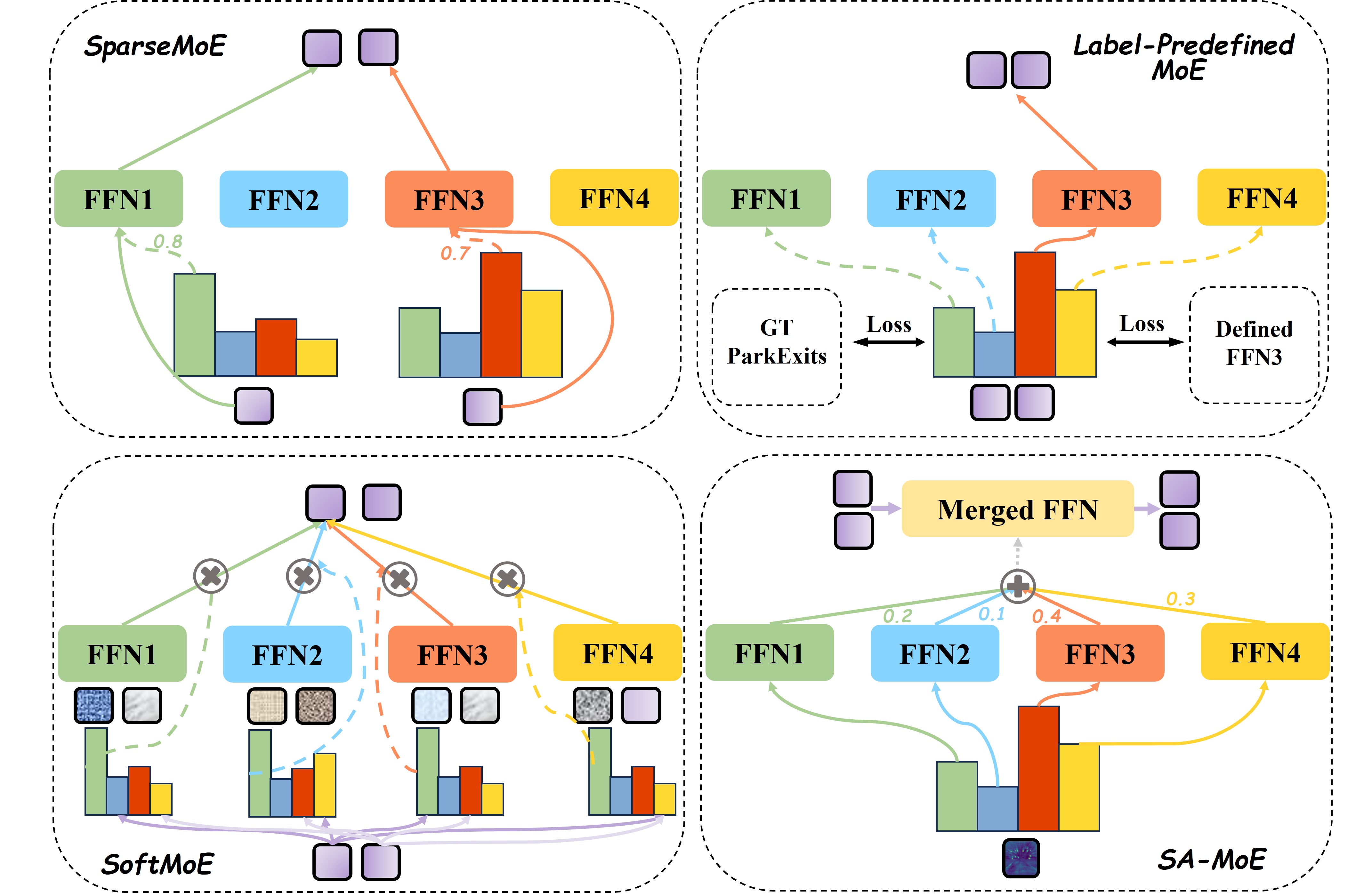} 
    \vspace{-2mm}
    \caption{Comparison between different MoE
}
    \label{fig:DSE}
    \vspace{-4mm}
\end{figure*}
\section{Comparison between Different MoE}

\paragraph{Preliminaries.}
Let $x_{b,t}\in\mathbb{R}^{d}$ be the $t$-th token in the $b$-th sample, and let each expert FFN with parameters $\{W_1^{(e)}\!\in\!\mathbb{R}^{d\times m},\,W_3^{(e)}\!\in\!\mathbb{R}^{d\times m},\,W_2^{(e)}\!\in\!\mathbb{R}^{m\times d}\}$.
Define the nominal FFN cost $C_{\text{ffn}}\approx (2dm+md)$ FLOPs per token.

\subsection{SparseMoE (Top-$k$ Token Routing)}
Sparse MoE routes each token independently to a small subset of experts.
Let $p_{b,t}\in\Delta^{E-1}$ be the gate over experts and $\mathcal{T}_k(b,t)$ the top-$k$ indices.
The layer output is
\[
y_{b,t}^{\text{sparse}}=\!\!\sum_{e\in\mathcal{T}_k(b,t)}\!\! p_{b,t,e}\, f^{(e)}(x_{b,t}).
\]
Routing is token-level, irregular, and requires dispatch/gather.
Complexity is $\mathcal{O}(BLk\,C_{\text{ffn}})$ per layer (sequence length $L$), with additional routing overhead and potential load imbalance.
Auxiliary balancing losses are commonly required.

\subsection{SoftMoE (Slot Aggregation Without Top-$k$)}
Soft Mixture-of-Experts replaces hard top-$k$ routing with continuous token-to-expert assignment. 
Each expert contains learnable slots $S\in\mathbb{R}^{E\times S\times d}$, and token–slot affinities are computed by
\begin{align}
    l_{b,t,e,s}=\langle x_{b,t},\, S_{e,s}\rangle .
\end{align}
A softmax over $(e,s)$ dispatches each token fractionally to all expert slots, which are then aggregated before passing through the expert FFN. 
This avoids load imbalance and exercises all experts every step, but incurs full expert computation, resulting in
\begin{align}
    \mathcal{O}(BES\,C_{\text{ffn}}),
\end{align}
and increased memory due to slot embeddings.

\subsection{SA-MoE}
We introduce a sample-level, BEV-conditioned, \emph{weight-merging} MoE that avoids token dispatch and executes a single FFN while preserving expert specialization.
SA-MoE runs a single FFN per token: $\mathcal{O}(BL\,C_{\text{ffn}})$, plus a per-sample merge cost
$\mathcal{O}\big(E(dm+md)\big)$ to form $\hat{W}_i^{(b)}$ (independent of $L$ and small for moderate $E$).
Compared with SparseMoE, SA-MoE removes token dispatch, padding/gather, and imbalance.
Compared with SoftMoE, it avoids the all-expert forward and slot memory.

\subsection{Qualitative Comparison}
\begin{itemize}
\item \textbf{Routing signal.}
SparseMoE routes by token content; SoftMoE routes by token-to-slot similarity; SA-MoE routes by BEV-derived \emph{scene} cues.
\item \textbf{Routing granularity.}
SparseMoE: token-level; SoftMoE: token-to-slot; SA-MoE: sample-level.
\item \textbf{Mixture locus.}
Sparse/Soft: mix \emph{outputs} in activation space; SA-MoE: mix \emph{parameters} to synthesize a sample-specific FFN.
\item \textbf{Efficiency.}
SparseMoE: $\propto k$ experts per token + dispatch overhead; SoftMoE: all experts each step; SA-MoE: one FFN per token + cheap merges, highly GPU-friendly.
\item \textbf{Specialization and stability.}
SparseMoE encourages sharp specialization but needs balancing.
SoftMoE is stable but compute-heavy.
SA-MoE inherits stability (softmax over experts) and yields coherent sample-level specialization via BEV cues, while keeping cost close to dense FFN.
\end{itemize}

\begin{table*}[t]
\centering
\small
\setlength{\tabcolsep}{4.2pt}  
\renewcommand{\arraystretch}{1.15}
\caption{Comparison of L2 distance and collision rate with different MoE designs.}
\begin{tabular}{lcccccccc}
\toprule
Method & \multicolumn{4}{c}{L2 $\downarrow$ (m)} & \multicolumn{4}{c}{Collision $\downarrow$ (\%)} \\
\cmidrule(lr){2-5} \cmidrule(lr){6-9}
 & 1s & 2s & 3s & Avg & 1s & 2s & 3s & Avg \\
\midrule
w/o MOE                              & 0.29 & 0.31 & 0.37 & 0.32 & 0.01 & 0.21 & 0.56 & 0.25 \\
sparseMOE (w/o bev bias)             & 0.24 & 0.26 & 0.40 & 0.30 & 0.02 & 0.22 & 0.84 & 0.36 \\
sparseMOE (with bev bias)            & 0.27 & 0.53 & 0.92 & 0.57 & 0.02 & 0.44 & 1.59 & 0.69 \\
SoftMOE (w/o bev bias)               & 0.39 & 0.42 & 0.56 & 0.46 & 0.09 & 0.44 & 1.17 & 0.57 \\
SoftMOE (with bev bias)              & 0.26 & 0.30 & 0.43 & 0.33 & 0.02 & 0.22 & 0.77 & 0.34 \\
Ours                   & \textbf{0.24} & \textbf{0.27} & \textbf{0.35} & \textbf{0.29} 
                                      & \textbf{0.01} & \textbf{0.18} & \textbf{0.60} & \textbf{0.26} \\
\bottomrule
\end{tabular}
\label{tab:abl_differnt_moe_design}
\end{table*}
\section{Additional Experiment Details}
\subsection{Additional Details Training Stages}
Our framework is trained in two main stages, with pretraining stage consisting of three sequential steps that establish multimodal perception and world modeling followed HERMES\cite{zhou2025hermes}.  

\textbf{Step 1: BEVEncoder Pretraining.} The BEV encoder $E$ and point cloud render $R$ are trained to map input images $I_t$ into point clouds via $P_t = R(E(I_t))$, using 12Hz data from the nuScenes training set.

\textbf{Step 2: BEV-Text Alignment and Refinement.} Vision-language alignment is established through in-projections for flattened BEV embeddings and out-projections for encoded BEV features, trained on NuInteract dense captions. To alleviate data scarcity, a masking-based augmentation is applied: one multi-view image is masked, and its caption is spliced from visible views, yielding approximately 200K image-text pairs—seven times the original nuScenes keyframe set. Subsequently, all parameters are unfrozen and the LLM is fine-tuned using LoRA, with supervision from scene descriptions in OmniDrive-nuScenes.

\textbf{Step 3: Understanding and Generation Unification.} Future point cloud generation modules are introduced and trained jointly with existing components using nuScenes keyframes, scene descriptions, and general conversation annotations from OmniDrive-nuScenes.

In the pretraining stage, the planning expert is frozen while the world-language expert is trained with a combined objective: language modeling loss for scene understanding and point cloud reconstruction loss for world prediction. Flow-matching loss $\mathcal{L}_{\text{flow}}(\theta)$ guides action interpolation toward true trajectories using multimodal context $\mathcal{C}$ (BEV tokens, language, world state, ego-motion), enabling smooth, context-aware forecasting.

The final stage has two phases: (1) train without MoE for stability, then (2) initialize the soft-weighted MoE sub-experts in the planning expert directly from the learned weights, enabling expert specialization while preserving pretraining knowledge.
\subsection{Additional Details About Comparison to Existing MoE Mechanism}
We further present the numerical results of different MoE implementation in Table~\ref{tab:abl_differnt_moe_design}. 
To isolate the contribution of Scene Adaptive MoE, we construct two ablation baselines with BEV bais. We present the details of these baseline,

\textbf{SoftMoE w/ BEV Bias.}
We extend the standard soft mixture-of-experts architecture by injecting BEV features into the slot-based routing mechanism. 
Each expert is associated with learnable slot embeddings $S \in \mathbb{R}^{E \times S \times d}$, and given token features 
$x \in \mathbb{R}^{B \times N \times d}$ and BEV features $h_{\mathrm{bev}}$, the routing logits are computed as
\begin{equation}
\mathrm{logits} 
= \langle x, S \rangle \cdot \omega_{\mathrm{slot}}
\;+\;
\big( \langle H_{\mathrm{BEV}}, S \rangle \odot \alpha \big)\cdot \omega_{\mathrm{bev}},
\end{equation}
where $\alpha \in \mathbb{R}^{E}$ is a learnable per-expert scaling factor.
The dispatch weights are obtained by normalizing logits across tokens and slots.
Each expert processes its assigned slot representation using
\begin{equation}
\mathrm{MLP}(x)
=
W_{2}\!\left(\sigma(W_{1}x)\odot W_{3}x\right).
\end{equation}
This baseline evaluates whether BEV-conditioned slot affinity improves soft expert merging.

\textbf{SparseMoE w/ BEV-Aware Gating.}
We further build a sparse MoE variant where BEV features directly bias the top-$k$ gating function.
The routing logits are computed as
\begin{equation}
\mathrm{router\_logits}
=
\omega_{\mathrm{tok}}\, W_{\mathrm{tok}} x
+
\omega_{\mathrm{bev}}\, W_{\mathrm{bev}} H_{\mathrm{BEV}},
\end{equation}
followed by a softmax and top-$k$ selection per token.
The selected routing probabilities are renormalized, and each active expert processes only its assigned tokens:
\begin{equation}
\mathrm{ExpertOutput}_{e}(x)
=
\mathrm{MLP}_{e}(x)\cdot p_{e}(x).
\end{equation}
The outputs from all activated experts are accumulated to obtain the final hidden representation.These two baselines allow us to disentangle the respective effects of BEV-aware soft slot routing and BEV-biased sparse top-$k$ gating.

\section{Theoretical Analysis of SAMoE}
In this section we provide a theoretical analysis and formal justification on the representation capacity, temporal causality stability, convergence advantages, and gradient stability of SAMoE.

\subsection{Preliminaries}

\subsubsection*{Scene-Optimal Parameter and Its Decomposition}

For a specific model parameter vector $\theta$, the flow-matching loss is defined as follows:

\begin{equation}
L_{\mathrm{flow}}(\theta)=
    \mathbb{E}_{x_0,\, t,\, \epsilon \mid s}
\left\|
v_\theta(x_t,t) - v^*(x_t,t)
\right\|^2
\end{equation}

Each scene is denoted by $s$, and represented by a bird's-eye-view feature
map $H_{\mathrm{BEV}}$. We define a \emph{scene-conditional optimum} 
parameter vector $\theta^*(s)$ as

\begin{equation}
\theta^*(s)
=
\arg\min_{\theta} 
\;
\mathbb{E}_{x_0,\, t,\, \epsilon \mid s}
\left\|
v_\theta(x_t,t) - v^*(x_t,t)
\right\|^2,
\label{eq:scene_optimal}
\end{equation}

where the conditional expectation is restricted to data belonging to 
scene $s$. Thus $\theta^*(s)$ is the parameter vector that minimizes the 
flow-matching loss for that specific scene. 

We only require the existence of at least one measurable selection of scene-optimal parameters $\theta^*(s)$. uniqueness is not needed, and all results hold for any such selection.

\subsubsection*{Mild Representability Assumption for Experts}

We consider a mixture-of-experts architecture with 
$E$ experts $\{F^{(e)}\}_{e=1}^E$, where each expert $e$ is parameterized by 
weight matrices $\{ W_i^{(e)} \}_{i=1}^M$ and has full parameter vector $\theta^{(e)}$.

\paragraph{Approximate scene-wise linearization.} 
Rather than assuming the experts' convex hull strictly contains the scene-optimal parameters,
we adopt a \emph{mild representability condition}: for each scene $s$, there exist 
mixture weights $\pi^*(s) \in \Delta^{E-1}$ such that the resulting linear combination
\begin{equation}
\Theta(\pi^*(s)) := \sum_{e=1}^E \pi_e^*(s)\, \theta^{(e)}
\end{equation}
provides a reasonable \emph{local linear approximation} to the scene-optimal parameter 
$\theta^*(s)$:
\begin{equation}
\|\theta^*(s) - \Theta(\pi^*(s))\| \le \epsilon_s,
\end{equation}
for some small $\epsilon_s$ depending on scene $s$. 

\paragraph{Justification.} 
This condition is \emph{sufficient but not necessary} for our theoretical analysis. 
It can be motivated by the following observations:
\begin{enumerate}
\item Joint MoE training encourages experts to specialize in different regions of the scene distribution.
\item Smoothness and parameter sharing imply that scene-optimal parameters vary on a low-curvature manifold.
\item By standard first-order Taylor or RCT-based arguments, any sufficiently smooth function over a low-curvature manifold can be locally approximated by a linear combination of nearby points (experts).
\end{enumerate}

\paragraph{Lemma (local linear approximation).}
Suppose the scene-optimal parameter manifold has local curvature bounded by $C$ and the set of experts $\{\theta^{(e)}\}$ covers the scene space with distance at most $\delta$. 
Then for any scene $s$, there exists a convex combination of experts $\Theta(\pi^*(s))$ such that
\begin{equation}
\|\theta^*(s) - \Theta(\pi^*(s))\| \le C \delta^2.
\end{equation}
This formalizes the sense in which experts can provide a \emph{local linear approximation} to $\theta^*(s)$ without assuming exact convex-span representability.

\subsubsection*{Two-sided (bi-Lipschitz) bounds on the model map.}
We now state and prove the two-sided Lipschitz-type inequalities
we will use repeatedly in the sequel.  Fix a data point $(x,t)$ and
consider the model map $\theta\mapsto v_\theta(x,t)\in\mathbb{R}^d$.

\medskip
\paragraph{Assumption A (Upper Lipschitz).}
There exists a constant $L>0$ such that for all admissible parameters
$\theta,\theta'$,
\begin{equation}\label{eq:upper_lipschitz}
\|v_\theta(x,t)-v_{\theta'}(x,t)\|_2 \le L \,\|\theta-\theta'\|_2.
\end{equation}
This is the standard Lipschitz continuity assumption on the parameter map.
In modern neural architectures this assumption is mild: each expert
is composed of affine layers and Lipschitz activation functions
(e.g., ReLU, GeLU), so the overall map $\theta\mapsto v_\theta$ is itself
Lipschitz with a constant determined by the product of layer-wise
spectral norms. Moreover, common training practices such as weight decay,
LayerNorm, residual connections, and gradient clipping implicitly control
the operator norms of intermediate layers, ensuring that the global Lipschitz
constant $L$ remains finite throughout training.

\medskip

\paragraph{Assumption B (Local Bi-Lipschitz / Mild Non-degeneracy).}
We assume that the model map $\theta \mapsto v_\theta(x,t)$ is continuously differentiable
and satisfies a local Jacobian lower bound \emph{along the regions of parameter space actually traversed by mixture combinations during training}:
\begin{equation}
J_\theta(x,t)^\top J_\theta(x,t) \succeq \lambda^2 I, \qquad 
\forall \theta \in \mathrm{conv}\bigl(\{\theta^{(e)}\}_{e=1}^E\bigr) \cap \mathcal{R}_{\text{train}},
\end{equation}
where $\mathcal{R}_{\text{train}}$ denotes the low-dimensional subspace of parameters explored by the mixture weights during training. 

Equivalently, the Jacobian has singular values at least $\lambda$ \emph{almost everywhere in the relevant region of parameter space}. 
This mild local curvature assumption ensures that perturbations in the mixture weights induce nontrivial changes in the model output, without requiring the model to be globally injective or globally bi-Lipschitz.

\medskip
\paragraph{Claim.} 
Under Assumptions A and B (local non-degeneracy along the training-relevant region), there exist constants
$L \ge c > 0$ (with $c = \lambda$) such that for all 
$\theta, \theta' \in \mathrm{conv}\bigl(\{\theta^{(e)}\}_{e=1}^E\bigr) \cap \mathcal{R}_{\text{train}}$,
\begin{equation}\label{eq:local_bi_lipschitz_pointwise}
c\,\|\theta-\theta'\|_2 \le \|v_\theta(x,t)-v_{\theta'}(x,t)\|_2
\le L\,\|\theta-\theta'\|_2.
\end{equation}

\paragraph{Proof.}
The upper bound is exactly \eqref{eq:upper_lipschitz} and holds globally.

For the lower bound, assume $\theta, \theta' \in \mathrm{conv}\bigl(\{\theta^{(e)}\}_{e=1}^E\bigr) \cap \mathcal{R}_{\text{train}}$
and set $\Delta\theta := \theta - \theta'$. Consider the line segment 
\(\gamma(u) = \theta' + u \Delta\theta\), $u \in [0,1]$, which lies entirely
within the convex hull of expert parameters and the training-relevant region. By the fundamental theorem of calculus
(vector-valued mean value / integral form)
\begin{equation}\label{eq:local_mean_value_integral}
v_\theta(x,t) - v_{\theta'}(x,t)
=
\int_0^1 J_{\gamma(u)}(x,t) \, du \;\; \Delta\theta.
\end{equation}
Using the local Jacobian lower bound along this segment,
\begin{align}
\|v_\theta - v_{\theta'}\|_2^2
&= \Delta\theta^\top \Big(\int_0^1 J_{\gamma(u)}^\top J_{\gamma(u)} \, du \Big) \Delta\theta
\nonumber\\
&\ge \lambda^2 \|\Delta\theta\|_2^2,
\end{align}
where $\lambda$ is the minimal singular value of $J_\theta$ in the training-relevant convex hull.
Taking the square root yields the desired lower bound $c = \lambda$.

\qed

\medskip

\paragraph{Conclusion.}
Under Assumptions A and B (local non-degeneracy), the model map $\theta \mapsto v_\theta(x,t)$
satisfies the two-sided (bi-Lipschitz) inequalities
\begin{align}\label{bi-lipschitz}
c\,\|\theta-\theta'\|_2 \le \|v_\theta(x,t)-v_{\theta'}(x,t)\|_2
\le L\,\|\theta-\theta'\|_2, 
\nonumber\\
\quad \forall \theta, \theta' \in \mathrm{conv}\bigl(\{\theta^{(e)}\}_{e=1}^E\bigr) \cap \mathcal{R}_{\text{train}},
\end{align}
with constants $L$ (global Lipschitz) and $c$ (local Jacobian lower bound). 
The upper bound $L$ controls the maximum sensitivity of the output to parameter changes globally,
while the lower bound $c$ ensures nontrivial output variation along convex combinations
of expert parameters in the training-relevant region, without requiring full global injectivity in the entire parameter space.

\subsubsection*{Representation Capacity Theory of SAMoE}

We analyze how a parameter-level soft merge approximates convex combinations
of expert outputs, and we derive a quantitative bound on the nonlinear residual
introduced by merging inside the parameters. This subsection provides a compact
form of the representation-capacity result used later to compare SAMoE with
token-level MoE and dense baselines.

\paragraph{Setup.}
Each expert $e\in\{1,\dots,E\}$ contains parameters $\{W^{(e)}_1,W^{(e)}_2,W^{(e)}_3\}$.
Given mixture weights $\pi\in\Delta^{E-1}$, the merged layer uses the convexly
combined parameters
\begin{equation}\label{eq:merge-weights}
W_i(\pi) = \sum_{e=1}^E \pi_e W_i^{(e)}.
\end{equation}
The nonlinear layer considered throughout the paper is
\begin{equation}\label{eq:layer-def}
F_W(x)=\sigma(xW_1)\odot(xW_3)\,W_2,
\end{equation}
with elementwise nonlinearity $\sigma$.

We assume:
(i) $\sigma$ is twice differentiable with bounded first and second derivatives;
(ii) input features satisfy $\|x\|_2\le B_x$;
(iii) all expert weight matrices are bounded by $\|W_i^{(e)}\|_F\le B_W$.
These are standard smoothness and boundedness assumptions for modern deep layers.

\vspace{0.5em}
\paragraph{Linear case: exact convexity.}
If the layer is linear, $f^{(e)}(x)=xW^{(e)}$, then for any $\pi$,
\begin{equation}\label{eq:linear-exact}
xW(\pi) = \sum_e \pi_e\, xW^{(e)}.
\end{equation}
Thus parameter-level merging is \emph{exactly} equivalent to output-level mixing.
Linear layers therefore incur no representation gap.

\vspace{0.5em}
\paragraph{Nonlinear case: defining the residual.}
Define the nonlinear merging error as
\begin{equation}
\Delta_\pi(x)=F_{W(\pi)}(x)-\sum_e\pi_e F_{W^{(e)}}(x).
\end{equation}
This term is zero for linear layers, and also zero for one-hot $\pi$, but
nonzero in general due to the nonlinearity in \eqref{eq:layer-def}.

\vspace{0.5em}
\paragraph{Representation–capacity bound.}
We show that the nonlinear residual is controlled by the pairwise dispersion
of experts:
\begin{equation}\label{eq:rct-cvpr}
\|\Delta_\pi(x)\|_2
\;\le\;
C\sum_{e\ne e'}\pi_e\pi_{e'}\|W^{(e)}-W^{(e')}\|_F^2,
\end{equation}
for a constant $C$ depending only on the layer dimensions, activation curvature,
and bounds on $x$ and $\{W^{(e)}\}$. When experts are similar or $\pi$ is sparse,
the right-hand side vanishes, recovering the exact convex representation.

\begin{proof}[Sketch of Proof]
Let $g(W)=\sigma(xW)$ denote the nonlinear layer with a smooth activation
$\sigma$. Since $\sigma$ has bounded second derivative, $g$ is twice
differentiable with Hessian norm uniformly bounded by a constant
$H_\sigma$ depending only on the activation curvature and $\|x\|$.

Consider the convex combination $W(\pi)=\sum_e \pi_e W^{(e)}$.
A second-order Taylor expansion of $g(W^{(e)})$ around $W(\pi)$ yields
\begin{equation}
g(W^{(e)}) \;=\; g(W(\pi))
\;+\; \langle \nabla g(W(\pi)),\, W^{(e)}-W(\pi)\rangle
\;+\; R_e,
\end{equation}

where the remainder satisfies
\begin{equation}
\|R_e\|
\;\le\;
C_0\,\|W^{(e)}-W(\pi)\|_F^{2}
\end{equation}
for some constant $C_0$ depending only on $H_\sigma$ and the layer
dimensions.

Using $\sum_e \pi_e (W^{(e)}-W(\pi)) = 0$, the linear terms cancel when
we form the mixture $\sum_e \pi_e g(W^{(e)})$, giving
\begin{equation}
\Delta_\pi(x)
= g(W(\pi))-\sum_e \pi_e g(W^{(e)})
= -\sum_e \pi_e R_e.
\end{equation}
Taking norms and substituting the bound on $R_e$,
\begin{equation}
\|\Delta_\pi(x)\|
\;\le\;
C_0 \sum_e \pi_e \|W^{(e)}-W(\pi)\|_F^2.
\end{equation}

Finally, expanding around the mean shows that
\begin{equation}
\sum_e \pi_e \|W^{(e)}-W(\pi)\|_F^2
= \frac12\sum_{e\neq e'} \pi_e \pi_{e'} \|W^{(e)}-W^{(e')}\|_F^2,
\end{equation}
which completes the proof after absorbing constants.
\end{proof}

\vspace{0.5em}
\paragraph{Interpretation.}
The decomposition
\begin{equation}
F_{W(\pi)}(x)
=
\underbrace{\sum_e\pi_e F_{W^{(e)}}(x)}_{\text{ideal convex combination}}
\;+\;
\underbrace{\Delta_\pi(x)}_{\text{nonlinear residual}}
\end{equation}
shows that the residual is small whenever:  
(1) experts lie close together in parameter space, or  
(2) $\pi$ is sparse (top-$1$ / top-$k$ routing), or  
(3) the merged point $W(\pi)$ lies near a low-variance region of the expert set.

This representation-capacity analysis provides the foundation for the
comparisons of token-level MoE and SAMoE, where token-level MoE and dense
baselines are shown to incur irreducible error when their routing weights
cannot match the scene-optimal mixture.

\subsubsection*{Fundamental Limitation of Token-Level Routing Under Local Information}

We analyze the expressive limitation of token-level MoE routing under
restricted information access.

\paragraph{Scene-Optimal Mixture.}
For each scene $s$, let
\begin{equation}
\pi^*(s) \in \Delta^{E-1}
\end{equation}
denote the scene-optimal expert mixture, and define
\begin{equation}
\theta^*(s)
=
\sum_{e=1}^{E}
\pi_e^*(s)\,\theta^{(e)}.
\end{equation}

We assume:
\begin{enumerate}
\item $\pi^*(s)$ lies in the interior of the simplex with non-zero probability;
\item $\pi^*(s)$ depends on global scene attributes.
\end{enumerate}

\paragraph{Assumption (Local Routing).}
The router of token-level MoE receives only a single token embedding
$x_t^{(i)}$ and does not access global scene features, BEV summaries,
or long-range context.

Thus routing is of the form
\begin{equation}
\pi^{(i)} = \mathrm{Router}(x_t^{(i)}).
\end{equation}

Let $\mathcal{G}_{\text{local}}$ denote the class of routing functions
measurable with respect to local token information.

\paragraph{Local-Information Lower Bound.}

Define the best achievable routing under local information:

\begin{equation}
\pi^{(i)}_{\text{best}}
=
\arg\min_{\pi \in \mathcal{G}_{\text{local}}}
\|\pi - \pi^*(s)\|.
\end{equation}

Since $\pi^*(s)$ depends on global scene attributes not observable from
$x_t^{(i)}$, the approximation error is strictly positive:

\begin{equation}
\inf_{\pi \in \mathcal{G}_{\text{local}}}
\|\pi - \pi^*(s)\|
\ge
\varepsilon_{\mathrm{local}}(s)
>
0.
\label{eq:local_lb}
\end{equation}

This lower bound holds whenever scene-level variables influencing
$\pi^*(s)$ are conditionally independent of $x_t^{(i)}$.

\paragraph{Top-$k$ Sparsification Constraint.}

Sparse token-level MoE further restricts routing to

\begin{equation}
S_{\pi^{(i)}} = \mathrm{Top}\text{-}k(\pi^{(i)}),
\qquad |S_{\pi^{(i)}}| = k,
\end{equation}

so that

\begin{equation}
\pi_e^{(i)} = 0
\quad \text{for } e \notin S_{\pi^{(i)}}.
\end{equation}

Let
\begin{equation}
\mathcal{S}_k
=
\bigcup_{|S|=k}
\mathrm{conv}\big(\{\mathbf{e}_e : e \in S\}\big)
\end{equation}
denote the union of all $k$-dimensional faces of the simplex.

If $\pi^*(s)$ lies in the interior of $\Delta^{E-1}$, then

\begin{equation}
\inf_{\pi \in \mathcal{S}_k}
\|\pi - \pi^*(s)\|
\ge
\varepsilon_{\mathrm{topk}}(s)
>
0.
\label{eq:topk_lb}
\end{equation}

This reflects the geometric fact that interior points cannot be
exactly represented by mixtures supported on only $k < E$ experts.

\paragraph{Unified Approximation Gap.}

Combining \eqref{eq:local_lb} and \eqref{eq:topk_lb},
any routing satisfying both local-information and top-$k$ constraints obeys

\begin{equation}
\|\pi^{(i)} - \pi^*(s)\|
\ge
\varepsilon_{\mathrm{tk}}(s)
:=
\varepsilon_{\mathrm{local}}(s)
+
\varepsilon_{\mathrm{topk}}(s)
>
0.
\end{equation}

\begin{proposition}[Token-Level Routing Gap]
Under local routing, the minimal deviation from the scene-optimal
mixture satisfies

\begin{equation}
\inf_{\pi \in \mathcal{G}_{\text{local}}}
\|\pi - \pi^*(s)\|
\ge
\varepsilon_{\mathrm{local}}(s) > 0.
\end{equation}

Under both local routing and top-$k$ sparsification,

\begin{equation}
\inf_{\pi \in \mathcal{G}_{\text{local}} \cap \mathcal{S}_k}
\|\pi - \pi^*(s)\|
\ge
\varepsilon_{\mathrm{topk}}(s) > 0.
\end{equation}
\end{proposition}

\paragraph{Corollary 1 (Sparse Token-Level MoE).}

Sparse MoE satisfies both constraints, hence incurs
an irreducible mixture approximation gap
$\varepsilon_{\mathrm{topk}}(s)$.

\paragraph{Corollary 2 (Soft Token-Level MoE).}

Soft MoE removes the top-$k$ constraint,
so $\varepsilon_{\mathrm{topk}}(s)=0$.
However, under local routing it still satisfies

\begin{equation}
\inf_{\pi \in \mathcal{G}_{\text{local}}}
\|\pi - \pi^*(s)\|
\ge
\varepsilon_{\mathrm{local}}(s) > 0.
\end{equation}

Thus Soft MoE reduces but does not eliminate
the representation gap.

\paragraph{Implication.}

The gap originates from an information granularity mismatch:
token-level routing operates on local embeddings,
whereas $\pi^*(s)$ depends on global scene-level attributes.

This mismatch directly leads to irreducible function
approximation error in token-level MoE variants,
which will be contrasted with the scene-level routing
mechanism of SAMoE in the next subsection.

\subsection{Comparison of Representation Capacity Across MoE Variants}

We compare the expressive power of four architectures:
(i) SAMoE (scene-level soft parameter merging),
(ii) Dense networks,
(iii) Sparse token-level MoE (top-$k$ routing),
(iv) Soft token-level MoE.

Let $\pi^*(s) \in \Delta^{E-1}$ denote the scene-optimal expert mixture,
and define the corresponding scene-optimal parameter

\begin{equation}
\theta^*(s) = \sum_{e=1}^{E} \pi_e^*(s)\,\theta^{(e)}.
\end{equation}

We assume $\theta^*(s)$ varies across scenes with non-zero variance.

\paragraph{1. Representation Capacity of SAMoE}

SAMoE produces scene-dependent parameters

\begin{equation}
\theta_{\text{SAMoE}}(s)
=
\sum_{e=1}^{E} \pi_e(s)\,\theta^{(e)},
\end{equation}

where the router observes global scene features.

From the Representation Capacity Theory (RCT) derived earlier,
the nonlinear merging residual satisfies

\begin{equation}
\|\Delta_{\pi}(x)\|
\le
C
\sum_{e \neq e'}
\pi_e \pi_{e'}
\|W^{(e)} - W^{(e')}\|_F^2 .
\end{equation}

Assume the router approximates $\pi^*(s)$ up to error $\delta(s)$:

\begin{equation}
\|\pi(s) - \pi^*(s)\|
\le
\delta(s).
\end{equation}

Then the approximation error satisfies

\begin{equation}
\mathbb{E}
\big[
\|f_{\theta_{\text{SAMoE}}(s)}(x)
-
f_{\theta^*(s)}(x)\|^2
\big]
\le
O(\delta(s)^2)
+
O(\text{expert dispersion}).
\end{equation}

If experts cover the parameter manifold and the router is consistent,
this error can be made arbitrarily small:

\begin{equation}
\inf_{f \in \mathcal{F}_{\text{SAMoE}}}
\mathbb{E}
\|f(x,s) - f^*(x,s)\|
=
0.
\end{equation}

Thus SAMoE can approximate the scene-optimal function arbitrarily well.

\paragraph{2. Representation Capacity of Dense Networks}

A dense model uses a single shared parameter vector
$\theta_D$ independent of $s$.

To represent all scenes perfectly, one would require

\begin{equation}
\theta_D \approx \theta^*(s)
\quad
\forall s.
\end{equation}

If $\theta^*(s)$ varies across scenes,
this condition is impossible.

Specifically, for two scenes $s_1, s_2$ such that

\begin{equation}
\|\theta^*(s_1) - \theta^*(s_2)\|
\ge
\gamma > 0,
\end{equation}

any shared parameter $\theta_D$ satisfies

\begin{equation}
\max_{s \in \{s_1, s_2\}}
\|\theta_D - \theta^*(s)\|
\ge
\frac{\gamma}{2}.
\end{equation}

Therefore,

\begin{equation}
\inf_{f \in \mathcal{F}_{\text{Dense}}}
\mathbb{E}
\|f(x,s) - f^*(x,s)\|
\ge
c \gamma.
\end{equation}

Dense networks exhibit an irreducible representation gap.

\paragraph{3. Representation Capacity of Sparse Token-Level MoE}

Sparse MoE performs token-level routing:

\begin{equation}
\pi^{(i)} = \mathrm{Router}(x^{(i)}),
\qquad
S_{\pi^{(i)}} = \mathrm{Top}\text{-}k(\pi^{(i)}).
\end{equation}

Each token realizes parameters

\begin{equation}
\theta^{(i)}
=
\sum_{e \in S_{\pi^{(i)}}}
\pi_e^{(i)} \theta^{(e)}.
\end{equation}

This restricts feasible mixtures to a union of $k$-dimensional faces
of the simplex.

If $\pi^*(s)$ lies in the interior of $\Delta^{E-1}$,
then

\begin{equation}
\|\pi^{(i)} - \pi^*(s)\|
\ge
\varepsilon_{\text{topk}}(s) > 0.
\end{equation}

Under the local-routing assumption
(the router observes only token-level embeddings),

\begin{equation}
\|\pi^{(i)} - \pi^*(s)\|
\ge
\varepsilon_{\text{local}}(s) > 0.
\end{equation}

Combining both constraints,

\begin{equation}
\|\pi^{(i)} - \pi^*(s)\|
\ge
\varepsilon_{\text{tk}}(s) > 0.
\end{equation}

Thus Sparse MoE incurs an irreducible representation gap:

\begin{equation}
\inf_{f \in \mathcal{F}_{\text{Sparse}}}
\mathbb{E}
\|f(x,s) - f^*(x,s)\|
\ge
c\,\varepsilon_{\text{tk}}(s).
\end{equation}

\paragraph{4. Representation Capacity of Soft Token-Level MoE}

Soft MoE removes the top-$k$ constraint:

\begin{equation}
\theta^{(i)}
=
\sum_{e=1}^{E}
\pi_e^{(i)} \theta^{(e)}.
\end{equation}

This eliminates $\varepsilon_{\text{topk}}(s)$.

However, under the local-routing assumption,
the router still lacks access to global scene information,
so

\begin{equation}
\|\pi^{(i)} - \pi^*(s)\|
\ge
\varepsilon_{\text{local}}(s) > 0.
\end{equation}

Hence Soft MoE reduces but does not eliminate
the representation gap.

\paragraph{5. Representation Hierarchy}

Under the stated assumptions, the function classes satisfy

\begin{equation}
\mathcal{F}_{\text{Dense}}
\subsetneq
\mathcal{F}_{\text{Sparse}}
\subsetneq
\mathcal{F}_{\text{Soft}}
\subsetneq
\mathcal{F}_{\text{SAMoE}}.
\end{equation}

The inclusions are strict whenever:

\begin{itemize}
\item scene-optimal parameters vary across scenes,
\item optimal mixtures lie in the interior of the simplex,
\item token routers lack explicit global scene access.
\end{itemize}

\subsection{Temporal Causality and Cross-Modal Coordination Disruption}

We show that token-level routing (Sparse and Soft MoE)
structurally disrupts temporal causality and cross-modal
coordination in flow-based driving planners,
while SAMoE preserves both properties.

Throughout this section we rely on the previously
established local bi-Lipschitz property:

\paragraph{Local Bi-Lipschitz Property.}
There exist constants $L \ge c > 0$ such that
for all
$\theta, \theta' \in
\mathrm{conv}(\{\theta^{(e)}\})
\cap \mathcal{R}_{\text{train}}$,

\begin{equation}
c\|\theta-\theta'\|_2
\le
\|v_\theta(x,t)-v_{\theta'}(x,t)\|_2
\le
L\|\theta-\theta'\|_2.
\label{eq:bi_lipschitz_used}
\end{equation}

This ensures that parameter perturbations induce
proportional perturbations in the vector field.

\subsubsection*{1. Temporal Causality Disruption}

The planner is modeled as a flow-matching ODE:

\begin{equation}
\frac{d y(t)}{dt}
=
v_{\theta_t}(y(t), t).
\label{eq:ode_main}
\end{equation}

\paragraph{Definition (Temporal Causality).}
The model preserves temporal causality if
small input perturbations between adjacent timesteps
lead to small deviations in the generated trajectory:
\begin{equation}
\|y_1(t)-y_2(t)\|
\le
C \|x_t - x_t'\|.
\end{equation}

\subsubsection*{Sparse Token-Level MoE}

Sparse MoE applies discrete Top-$k$ routing:

\begin{equation}
S_{\pi_t}
=
\mathrm{Top}\text{-}k(\mathrm{Router}(x_t)).
\end{equation}

Because Top-$k$ is piecewise constant,
there exist arbitrarily close inputs
$x_t, x_t'$ such that

\begin{equation}
\|x_t - x_t'\|_2 \to 0,
\qquad
\theta_t \neq \theta_t'.
\end{equation}

Hence
\begin{equation}
\|\theta_t - \theta_t'\|_2
\ge
\delta > 0.
\end{equation}

By the bi-Lipschitz lower bound:

\begin{equation}
\|v_{\theta_t}(x,t)
-
v_{\theta_t'}(x,t)\|_2
\ge
c\delta.
\end{equation}

Thus arbitrarily small input perturbations
induce finite jumps in the vector field.

\paragraph{Trajectory-Level Consequence via ODE Stability.}

Consider two ODE systems initialized at the same state
$y_1(0)=y_2(0)$:

\begin{equation}
\frac{d y_1}{dt}
=
v_{\theta_t}(y_1(t), t),
\qquad
\frac{d y_2}{dt}
=
v_{\theta_t'}(y_2(t), t).
\end{equation}

Define the trajectory error
\begin{equation}
e(t) := y_1(t)-y_2(t).
\end{equation}

Then
\begin{equation}
\frac{d e(t)}{dt}
=
v_{\theta_t}(y_1,t)-v_{\theta_t'}(y_2,t).
\end{equation}

Adding and subtracting $v_{\theta_t}(y_2,t)$ gives

\begin{equation}
\frac{d e(t)}{dt}
=
\underbrace{
v_{\theta_t}(y_1,t)-v_{\theta_t}(y_2,t)
}_{(A)}
+
\underbrace{
v_{\theta_t}(y_2,t)-v_{\theta_t'}(y_2,t)
}_{(B)}.
\end{equation}

\paragraph{State Lipschitz Bound.}

Assume the vector field is Lipschitz in the state variable:

\begin{equation}
\|v_{\theta}(y_1,t)-v_{\theta}(y_2,t)\|
\le
L_y \|y_1-y_2\|.
\end{equation}

Thus

\begin{equation}
\|(A)\|
\le
L_y \|e(t)\|.
\end{equation}

\paragraph{Parameter-Induced Vector Field Gap.}

From the previously established bi-Lipschitz property of the
vector field with respect to parameters,

\begin{equation}
\|v_{\theta_t}(y,t)-v_{\theta_t'}(y,t)\|
\ge
c\|\theta_t-\theta_t'\|.
\end{equation}

If the routing perturbation satisfies

\begin{equation}
\|\theta_t-\theta_t'\| \ge \delta,
\end{equation}

then for any state $y$

\begin{equation}
\|v_{\theta_t}(y,t)-v_{\theta_t'}(y,t)\|
\ge
c\delta.
\end{equation}

In particular, term (B) satisfies

\begin{equation}
\|(B)\|
\ge
c\delta.
\end{equation}

\paragraph{Early-Time Trajectory Deviation.}

At $t=0$, since $y_1(0)=y_2(0)$, we have

\begin{equation}
\frac{d e}{dt}(0)
=
v_{\theta_0}(y_0,0)-v_{\theta_0'}(y_0,0).
\end{equation}

Therefore

\begin{equation}
\left\|\frac{d e}{dt}(0)\right\|
\ge
c\delta.
\end{equation}

By continuity of the vector field, there exists a small
$\tau>0$ such that for all $t\in[0,\tau]$

\begin{equation}
\left\|\frac{d e}{dt}(t)\right\|
\ge
\frac{c}{2}\delta.
\end{equation}

Integrating over $[0,t]$ yields

\begin{equation}
\|e(t)\|
=
\left\|
\int_0^t
\frac{d e}{ds}\,ds
\right\|
\ge
\int_0^t
\left\|\frac{d e}{ds}\right\| ds
\ge
\frac{c}{2}\delta t.
\end{equation}

Hence for any $t\in(0,\tau]$

\begin{equation}
\|y_1(t)-y_2(t)\|
\ge
\frac{c}{2}\delta t.
\end{equation}

\paragraph{Implication.}

Thus even when the input states become arbitrarily close
($\|x_t-x_t'\|\to0$), a routing-induced parameter perturbation
$\|\theta_t-\theta_t'\|\ge\delta$ leads to a trajectory deviation
satisfying

\begin{equation}
\|y_1(t)-y_2(t)\|
\ge
\Omega(\delta).
\end{equation}

This shows that small routing perturbations can induce
non-vanishing trajectory deviations, implying instability
at the trajectory level.
\subsubsection*{Soft Token-Level MoE}

Soft MoE removes the discrete Top-$k$ routing,
but the routing remains token-dependent.
The effective parameters at time $t$ become

\begin{equation}
\theta_t
=
\sum_e
\pi_{e,t}\theta^{(e)},
\end{equation}

where $\pi_{e,t}$ are routing weights satisfying
$\sum_e \pi_{e,t}=1$.

\paragraph{Routing Weight Sensitivity.}

The parameter difference between two consecutive timesteps is

\begin{equation}
\theta_t - \theta_{t-1}
=
\sum_e
(\pi_{e,t}-\pi_{e,t-1})\theta^{(e)}.
\end{equation}

If expert parameters are bounded,
there exists a constant $c'>0$ such that

\begin{equation}
\|\theta_t-\theta_{t-1}\|_2
\ge
c'\|\pi_t-\pi_{t-1}\|_2.
\end{equation}

\paragraph{Vector Field Perturbation.}

Using the previously established bi-Lipschitz property
of the vector field with respect to parameters,

\begin{equation}
\|v_{\theta_t}(y,t)-v_{\theta_{t-1}}(y,t)\|
\ge
c\|\theta_t-\theta_{t-1}\|,
\end{equation}

which implies

\begin{equation}
\|v_{\theta_t}(y,t)-v_{\theta_{t-1}}(y,t)\|
\ge
cc'\|\pi_t-\pi_{t-1}\|.
\end{equation}

\paragraph{Trajectory Consequence.}

Suppose the routing weights vary significantly between
consecutive timesteps, i.e.,

\begin{equation}
\|\pi_t-\pi_{t-1}\|_2 \ge \epsilon
\end{equation}

for some $\epsilon>0$ on a set of timesteps with positive measure.
Then the induced vector field perturbation satisfies

\begin{equation}
\|v_{\theta_t}-v_{\theta_{t-1}}\|
\ge
cc'\epsilon .
\end{equation}

Applying the short-time ODE stability argument, the resulting trajectories satisfy

\begin{equation}
\|y(t)-y'(t)\|
\ge
\Omega(\epsilon).
\end{equation}

\paragraph{Implication.}

In practice, token-level soft routing often exhibits
rapid temporal fluctuations in mixture weights due to
sensitivity to small variations in token embeddings across
frames. Such routing oscillations can propagate through the
bi-Lipschitz vector field, leading to non-smooth or jittery
trajectory behavior.
\subsubsection*{2. Cross-Modal Coordination Disruption}

Let modalities be indexed by $m$.
Under token-level MoE routing, each modality produces
its own routing weights based on modality-specific tokens.
The resulting effective parameters are

\begin{equation}
\theta^{(m)}
=
\sum_e
\pi_e^{(m)}\theta^{(e)},
\end{equation}

where $\pi_e^{(m)}$ denotes the routing weight of expert $e$
for modality $m$.

Since routing depends on modality-specific token embeddings,
different modalities may produce different routing decisions,
leading to

\begin{equation}
\theta^{(m_1)} \neq \theta^{(m_2)}.
\end{equation}

When modality-specific tokens lead to meaningfully different
routing weights, the parameter discrepancy satisfies

\begin{equation}
\|\theta^{(m_1)}-\theta^{(m_2)}\|
\ge
\delta_m
\end{equation}

for some $\delta_m>0$.

\paragraph{Vector Field Discrepancy.}

Using the previously established bi-Lipschitz property of
the vector field with respect to parameters,

\begin{equation}
\|v_{\theta^{(m_1)}}(y,t)-v_{\theta^{(m_2)}}(y,t)\|
\ge
c\delta_m.
\end{equation}

\paragraph{Trajectory Consequence.}

Consider two trajectories generated from the same initial
state but using modality-specific parameters:

\begin{equation}
\frac{d y_{m_1}}{dt}
=
v_{\theta^{(m_1)}}(y_{m_1},t),
\qquad
\frac{d y_{m_2}}{dt}
=
v_{\theta^{(m_2)}}(y_{m_2},t).
\end{equation}

Applying the same short-time ODE stability argument
(Appendix~E.1), we obtain

\begin{equation}
\|y_{m_1}(t)-y_{m_2}(t)\|
\ge
\Omega(\delta_m).
\end{equation}

\paragraph{Implication.}

Therefore, when modality-dependent routing leads to
non-negligible parameter discrepancies, the resulting
vector field mismatch can induce trajectory-level divergence,
potentially disrupting cross-modal coordination in downstream
decision-making.
\subsubsection*{3. Stability of SAMoE}

SAMoE performs scene-level parameter merging:

\begin{equation}
\theta(s)
=
\sum_e
\pi_e(s)\theta^{(e)}.
\end{equation}

All modalities share the same parameter:

\begin{equation}
\theta^{(m)} = \theta(s).
\end{equation}

Hence no parameter discrepancy exists across modalities,
and cross-modal parameter divergence is eliminated at the source.

\paragraph{Temporal Stability of the Vector Field.}

Assume the scene encoder is Lipschitz continuous:

\begin{equation}
\|\theta(s_t)-\theta(s_{t-1})\|_2
\le
L_s \|s_t - s_{t-1}\|_2.
\end{equation}

Using the upper bound of the local bi-Lipschitz property
\eqref{eq:bi_lipschitz_used},

\begin{equation}
\|v_{\theta(s_t)}(x,t)
-
v_{\theta(s_{t-1})}(x,t)\|_2
\le
L \|\theta(s_t)-\theta(s_{t-1})\|_2,
\end{equation}

which yields

\begin{equation}
\|v_{\theta(s_t)} - v_{\theta(s_{t-1})}\|_2
\le
L L_s
\|s_t - s_{t-1}\|_2.
\end{equation}

Thus the induced vector field is Lipschitz
with respect to scene perturbations.

\paragraph{Trajectory-Level Stability via ODE Analysis.}

Consider two trajectories generated from scenes
$s_t$ and $s_{t-1}$:

\begin{equation}
\frac{d y_1}{dt}
=
v_{\theta(s_t)}(y_1,t),
\qquad
\frac{d y_2}{dt}
=
v_{\theta(s_{t-1})}(y_2,t).
\end{equation}

Define $e(t)=y_1(t)-y_2(t)$.
Then

\begin{equation}
\frac{d e(t)}{dt}
=
v_{\theta(s_t)}(y_1,t)
-
v_{\theta(s_{t-1})}(y_2,t).
\end{equation}

Add and subtract $v_{\theta(s_t)}(y_2,t)$:

\begin{equation}
\frac{d e(t)}{dt}
=
\underbrace{
v_{\theta(s_t)}(y_1,t)
-
v_{\theta(s_t)}(y_2,t)
}_{\text{state Lipschitz}}
+
\underbrace{
v_{\theta(s_t)}(y_2,t)
-
v_{\theta(s_{t-1})}(y_2,t)
}_{\text{scene-induced difference}}.
\end{equation}

Using state Lipschitz constant $L_y$:

\begin{equation}
\|(A)\|
\le
L_y \|e(t)\|.
\end{equation}

Using the vector-field Lipschitz bound above:

\begin{equation}
\|(B)\|
\le
L L_s \|s_t - s_{t-1}\|.
\end{equation}

Hence

\begin{equation}
\frac{d}{dt}\|e(t)\|
\le
L_y \|e(t)\|
+
L L_s \|s_t - s_{t-1}\|.
\end{equation}

Solving this linear differential inequality yields

\begin{equation}
\|e(t)\|
\le
\frac{L L_s}{L_y}
\left(e^{L_y t}-1\right)
\|s_t - s_{t-1}\|.
\end{equation}

Therefore, for finite horizon $t$,

\begin{equation}
\|y_1(t)-y_2(t)\|
\le
\tilde C
\|s_t - s_{t-1}\|,
\end{equation}

for some constant $\tilde C > 0$.

\paragraph{Conclusion.}

Unlike token-level MoE, SAMoE induces
a scene-level Lipschitz parameterization.
The resulting vector field is Lipschitz continuous,
and ODE stability guarantees
trajectory continuity.

Hence SAMoE preserves temporal causality
and avoids trajectory-level instability.
\begin{theorem}[Trajectory-Level Structural Disruption]
Under local routing, token-level MoE
induces parameter discontinuities (Sparse)
or oscillations (Soft).
By the local bi-Lipschitz mapping and ODE stability,
these parameter perturbations propagate
into finite trajectory deviations,
violating temporal causality
and cross-modal coordination.

In contrast, SAMoE maintains
continuous scene-level parameterization,
thereby preserving trajectory-level stability.
\end{theorem}

\subsection{Convergence Advantage of SAMoE}

We analyze optimization convergence under stochastic gradient
descent (SGD). Let the training objective be

\begin{equation}
\min_\theta \mathcal L(\theta),
\end{equation}

where $\mathcal L$ is possibly non-convex.

\paragraph{Assumptions.}

(A1) $\mathcal L$ is $L$-smooth:
\begin{equation}
\|\nabla \mathcal L(\theta_1) - \nabla \mathcal L(\theta_2)\|
\le
L\|\theta_1-\theta_2\|.
\end{equation}

(A2) The stochastic gradient $g_t$ satisfies
\begin{equation}
\mathbb E[g_t|\theta_t]
=
\nabla \mathcal L(\theta_t),
\end{equation}
\begin{equation}
\mathbb E\|g_t-\nabla \mathcal L(\theta_t)\|^2
\le
\sigma^2.
\end{equation}

\paragraph{SGD Descent Inequality.}

SGD update:
\begin{equation}
\theta_{t+1}=\theta_t-\eta g_t.
\end{equation}

By $L$-smoothness,

\begin{equation}
\mathcal L(\theta_{t+1})
\le
\mathcal L(\theta_t)
-
\eta \langle \nabla \mathcal L(\theta_t), g_t\rangle
+
\frac{L\eta^2}{2}\|g_t\|^2.
\end{equation}

Taking expectation and using the variance bound yields

\begin{equation}
\mathbb E[\mathcal L(\theta_{t+1})]
\le
\mathcal L(\theta_t)
-
\frac{\eta}{2}
\|\nabla \mathcal L(\theta_t)\|^2
+
\frac{L\eta^2}{2}\sigma^2,
\end{equation}

for $\eta \le 1/L$.

Summing over $t=0,\dots,T-1$ gives

\begin{equation}
\frac{1}{T}
\sum_{t=0}^{T-1}
\mathbb E\|\nabla \mathcal L(\theta_t)\|^2
\le
\frac{2(\mathcal L(\theta_0)-\mathcal L^*)}{\eta T}
+
L\eta\sigma^2.
\label{eq:sgd_bound}
\end{equation}

Choosing $\eta = O(1/\sqrt{T})$ yields

\begin{equation}
\frac{1}{T}
\sum_{t=0}^{T-1}
\mathbb E\|\nabla \mathcal L(\theta_t)\|^2
=
O\!\left(\frac{1}{\sqrt{T}}\right)
+
O\!\left(\frac{\sigma^2}{\sqrt{T}}\right).
\end{equation}

Therefore, the convergence behavior depends critically on
the stochastic gradient variance $\sigma^2$.

\paragraph{Routing-Induced Gradient Variance.}

In MoE models, the effective parameter depends on the input:

\begin{equation}
\theta(x)
=
\sum_e \pi_e(x)\theta^{(e)}.
\end{equation}

Consequently the gradient variance decomposes as

\begin{equation}
\sigma^2
=
\sigma^2_{\text{data}}
+
\sigma^2_{\text{routing}},
\end{equation}

where $\sigma^2_{\text{routing}}$ arises from variation of
the effective parameter $\theta(x)$ within a mini-batch.

We quantify this term via parameter dispersion

\begin{equation}
\sigma^2_{\text{routing}}
\propto
\mathbb E
\|\theta(x)-\bar\theta\|^2,
\end{equation}

where $\bar\theta$ is the batch-mean parameter.

\begin{theorem}[Variance Reduction of Scene-Level Routing]
Under Assumptions (A1)-(A2), SGD satisfies the bound
\eqref{eq:sgd_bound}. Moreover, routing mechanisms with
lower parameter dispersion
$\mathbb E\|\theta(x)-\bar\theta\|^2$
lead to smaller gradient variance
$\sigma^2_{\text{routing}}$ and thus a tighter SGD bound.

In practice, the routing variance typically follows the
qualitative ordering

\begin{equation}
\sigma^2_{\text{Sparse}}
\;\gtrsim\;
\sigma^2_{\text{Soft}}
\;\gtrsim\;
\sigma^2_{\text{SAMoE}} .
\end{equation}

Therefore scene-level routing as used in SAMoE reduces
routing-induced gradient variance and improves optimization
stability.
\end{theorem}

\paragraph{Proof Sketch.}

\textbf{Step 1: Variance Decomposition.}

Let the stochastic gradient be

\begin{equation}
g(x)
=
\nabla_\theta \ell(\theta(x); x).
\end{equation}

Assume the gradient is Lipschitz in $\theta$:

\begin{equation}
\|
\nabla_\theta \ell(\theta_1;x)
-
\nabla_\theta \ell(\theta_2;x)
\|
\le
G
\|\theta_1-\theta_2\|.
\end{equation}

Then routing variance satisfies

\begin{equation}
\sigma^2_{\text{routing}}
\le
G^2
\mathbb E
\|\theta(x)-\bar\theta\|^2 .
\end{equation}

Thus gradient variance is controlled by parameter dispersion.

\textbf{Step 2: Sparse MoE.}

For Top-$k$ routing,

\begin{equation}
\theta(x)
=
\sum_{e\in S_{\pi(x)}} \theta^{(e)} .
\end{equation}

Because routing decisions are discrete, the mapping
$x \mapsto \theta(x)$ is piecewise constant.
Small input perturbations can therefore cause abrupt
changes in expert selection.

In heterogeneous mini-batches spanning multiple routing
regions, such discontinuities may produce large variation
in effective parameters $\theta(x)$ across samples,
leading to increased routing variance.

\textbf{Step 3: Soft MoE.}

Soft MoE replaces discrete routing with continuous
mixture weights:

\begin{equation}
\theta(x)
=
\sum_e
\pi_e(x)\theta^{(e)} .
\end{equation}

The mapping $x\mapsto\theta(x)$ is therefore continuous,
removing jump discontinuities in parameter assignments.
Although mixture weights may still vary across tokens,
parameter dispersion is typically smaller than in sparse
Top-$k$ routing.

\textbf{Step 4: SAMoE.}

SAMoE performs routing at the scene level:

\begin{equation}
\theta(s)
=
\sum_e \pi_e(s)\theta^{(e)} .
\end{equation}

All tokens within the same scene share the same parameter.
Assume the scene encoder is Lipschitz:

\begin{equation}
\|\theta(s_i)-\theta(s_j)\|
\le
L_s
\|s_i-s_j\|.
\end{equation}

In driving scenarios, scene embeddings are typically
computed from aggregated multi-modal features (e.g.,
BEV representations or pooled token features).
Such aggregated representations exhibit substantially
lower variance across tokens within the same frame
than individual token embeddings.

Consequently parameter dispersion satisfies

\begin{equation}
\mathbb E
\|\theta(s)-\bar\theta\|^2
\le
L_s^2
\mathrm{Var}(s),
\end{equation}

which is typically smaller than token-level dispersion.


\textbf{Conclusion.}

Scene-level routing reduces intra-batch parameter
variation, leading to smaller routing-induced gradient
variance. Substituting this into the SGD bound
\eqref{eq:sgd_bound} yields a tighter optimization bound,
which empirically translates to more stable and often
faster convergence during training.

\subsection{Gradient Stability of SAMoE.}

We now analyze the intrinsic gradient stability of SAMoE.
The goal is to show that the stochastic gradient varies
in a controlled and Lipschitz manner with respect to scene variation.

\textbf{Setting.}

In SAMoE, the effective parameter is scene-dependent:
\begin{equation}
\theta(s)
=
\sum_e \pi_e(s)\theta^{(e)},
\end{equation}
where $s$ denotes the scene representation,
and all tokens within the same scene share the same $\theta(s)$.

Let the per-sample loss be $\ell(\theta(s); x)$.

\textbf{Assumptions.}

(A1) The loss gradient is Lipschitz in the parameter:
\begin{equation}
\|
\nabla_\theta \ell(\theta_1; x)
-
\nabla_\theta \ell(\theta_2; x)
\|
\le
G
\|\theta_1-\theta_2\|.
\end{equation}

(A2) The scene-conditioned parameter mapping is Lipschitz:
\begin{equation}
\|\theta(s_1)-\theta(s_2)\|
\le
L_s
\|s_1-s_2\|.
\end{equation}

\textbf{Proposition (Gradient Stability of SAMoE).}

Under (A1)-(A2), the gradient field of SAMoE satisfies
\begin{equation}
\|
\nabla_\theta \ell(\theta(s_1); x)
-
\nabla_\theta \ell(\theta(s_2); x)
\|
\le
G L_s
\|s_1-s_2\|.
\label{eq:samoe_grad_stability}
\end{equation}

\textit{Proof.}

By Assumption (A1),
\begin{equation}
\|
\nabla_\theta \ell(\theta(s_1); x)
-
\nabla_\theta \ell(\theta(s_2); x)
\|
\le
G
\|\theta(s_1)-\theta(s_2)\|.
\end{equation}

Applying Assumption (A2),
\begin{equation}
\|\theta(s_1)-\theta(s_2)\|
\le
L_s
\|s_1-s_2\|.
\end{equation}

Combining the two inequalities yields
\begin{equation}
\|
\nabla_\theta \ell(\theta(s_1); x)
-
\nabla_\theta \ell(\theta(s_2); x)
\|
\le
G L_s
\|s_1-s_2\|.
\end{equation}
\hfill $\square$

\textbf{Implications.}

Equation~\eqref{eq:samoe_grad_stability}
shows that the gradient field of SAMoE
is Lipschitz continuous with constant $G L_s$
with respect to scene variation.

This implies:

1. Small perturbations in scene representation
lead to proportionally small gradient changes.

2. The gradient does not exhibit abrupt jumps.

3. The stochastic optimization trajectory
evolves under a smooth vector field.

Consequently, SAMoE enjoys intrinsically stable
gradient dynamics during training.

\section{Algorithm}
\begin{algorithm}[htbp]
\caption{Stage I(step3) — Pretraining of World--Language Expert}
\label{alg:stage1_formula}
\begin{algorithmic}[1]

\Require Dataset $\mathcal{D}$; BEV encoder $\mathcal{E}_{\mathrm{BEV}}$; world--language transformer $\mathcal{T}$; frozen planning expert.

\For{each batch $(\mathbf{F}_{\mathrm{BEV}}, w_{1:T})\sim\mathcal{D}$}
    \State $\mathbf{z} = \mathcal{E}_{\mathrm{BEV}}(\mathbf{F}_{\mathrm{BEV}})$
    \State Mask ego-state \& action tokens: 
        \[
        \mathbf{x} = [\mathbf{z},\, \text{[MASK]}_{\text{state+action}},\, w_{1:T}]
        \]
    \State Autoregressive LM:
        \[
        p_\theta(w_t \mid w_{<t}, \mathbf{z})
        = \mathcal{T}(w_{<t}, \mathbf{z})[t]
        \]
    \State Future scene prediction:
        \[
        \hat d_{f,r} = \mathcal{T}_{\text{depth}}(\mathbf{z})[f,r]
        \]
    \State Pretraining loss:
        \[
        \mathcal{L}_{\text{LM}} 
        = -\sum_{t=1}^{T}\log p_\theta(w_t\mid w_{<t},\mathbf{z})
        \]
        \[
        \mathcal{L}_{\text{depth}}
        = 10\sum_{f=1}^{F}\lambda_f \frac{1}{R_f}
        \sum_{r=1}^{R_f} | d_{f,r} - \hat d_{f,r} |
        \]
        \[
        \mathcal{L}_{\text{pre}}
        =
        \mathcal{L}_{\text{LM}}
        +
        \mathcal{L}_{\text{depth}}
        \]
    \State $\theta \leftarrow \theta - \eta\nabla_{\theta}\mathcal{L}_{\text{pre}}$
\EndFor

\end{algorithmic}
\end{algorithm}

\begin{algorithm}[htbp]
\caption{DSE-based Routing Preparation}
\label{alg:dse_routing}
\begin{algorithmic}[1]

\Require BEV feature maps $\mathbf{F}_{\mathrm{BEV}}$

\For{each BEV feature map $\mathbf{F}_{\mathrm{BEV}}$}
    \State Construct near-field distance map $M_{\mathrm{near}}$.
    \State Predict deformable offsets
    \[
        \Delta = \mathcal{P}\!\left([\mathbf{F}_{\mathrm{BEV}}; M_{\mathrm{near}}]\right).
    \]
    \State Compute deformable scene features
    \[
        \mathbf{S}_{\mathrm{BEV}}
        = \mathrm{Norm}\!\left(
            \mathrm{Flatten}\!\left(
                \mathrm{DeformConv}(\mathbf{F}_{\mathrm{BEV}},\Delta)
            \right)
        \right).
    \]
    \State Extract routing-aware hidden states
    \[
        \mathbf{H}_{\mathrm{BEV}}
        = \mathrm{MHA}(\mathbf{Q}, \mathbf{S}_{\mathrm{BEV}}, \mathbf{S}_{\mathrm{BEV}}).
    \]
\EndFor

\end{algorithmic}
\end{algorithm}

\begin{algorithm}[htbp]
\caption{Final Training Stage of SAMoE-VLA}
\label{alg:training_pipeline}
\begin{algorithmic}[1]

\Require Pretrained world VLM; BEV encoder; MoE planning expert; dataset $\mathcal{D}$.

\For{each sample $(\mathcal{I}_n,l,S,A_t)$}

    \State Compute BEV features 
    \[
        \mathbf{F}_{\mathrm{BEV}}
        = \mathrm{SCA}(\text{img\_feats}, \text{BEV\_queries}, \text{pos}).
    \]

    \State Run DSE routing (Alg.~\ref{alg:dse_routing}) to get $\mathbf{H}_{\mathrm{BEV}}$.

    \State Tokenize language instructions $\mathbf{L}$.
    \State Encode BEV tokens  
    \[
        \mathbf{z}=\mathrm{MLP}(\mathrm{DownSampleCross}(\mathbf{F}_{\mathrm{BEV}})).
    \]

    \State World tokens: 
    \[
        \mathbf{W}^{(k)}
        =\mathrm{Pool}(\mathbf{z})+\mathbf{f}_{\mathrm{frame}}(k),\ k=1,\ldots,K.
    \]

    \State Sample $\varepsilon\!\sim\!\mathcal{N}(0,I)$, $\tau\!\sim\!U(0,1)$ and set  
    \[
        x_\tau=(1-\tau)A_t+\tau\varepsilon.
    \]

    \For{each layer in SAMoE-VLA}

        \State $\mathbf{V}_{\mathrm{WL}}=\mathrm{Norm}\big(WQKV_{\mathrm{WL}}(\mathbf{z},\mathbf{L},\mathbf{W})\big)$
        \State $\mathbf{V}_{\mathrm{P}}=\mathrm{Norm}\big(WQKV_{\mathrm{P}}(\mathbf{S},x_\tau)\big)$

        \State Concatenate values: 
        $\mathbf{V}=\mathrm{Concat}(\mathbf{V}_{\mathrm{WL}},\mathbf{V}_{\mathrm{P}})$

        \State CMCA attention: 
        $\mathbf{O}=\mathrm{softmax}(\mathbf{S}')\,\mathbf{V}$

        \State Apply WL FFN:
        \[
            \mathrm{FFN}_{\mathrm{WL}}(\mathbf{O}_{[:L_{\mathrm{WL}}]}+\mathbf{R}_{\mathrm{WL}}^{*})
        \]

        \If{layer is MoE and step is 2}
            \State Routing weights:
            \[
                \pi=\mathrm{softmax}\big(
                \mathrm{Linear}(\mathrm{MeanPool}(\mathbf{H}_{\mathrm{BEV}}))
                \big)
            \]
            \State Expert fusion:
            \[
                \tilde{\mathbf{W}}_i=\sum_e \pi_e\,\mathbf{W}_i^{(e)},\quad i=1,2,3
            \]
            $\mathrm{FFN}(\mathrm{P}) = \tilde{\mathbf{W}}$
        \EndIf
        \State Planning FFN:
        \[
            \mathrm{FFN}_{\mathrm{P}}(\mathbf{O}_{[:L_{\mathrm{P}}]}+\mathbf{R}_{\mathrm{P}}^{*})
        \]

    \EndFor

    \State Predict flow velocity $v_\theta(x_\tau,\mathcal{C},\tau)$

    \State Flow-matching loss:
    \[
        \mathcal{L}_{\mathrm{flow}}
        =\|v_\theta(x_\tau,\mathcal{C},\tau)-(A_t-\varepsilon)\|^2
    \]

    \State Update planning expert.

\EndFor

\end{algorithmic}
\end{algorithm}

\section{Challenging Scenarios: Selection from \textsc{nuScenes}}
\label{sec:challenging-scenarios}

To evaluate the behavior and safety of token-level Mixture-of-Experts (MoE) mechanisms in embodied driving, we construct three focused subsets from the \textsc{nuScenes} corpus that emphasize scene-level difficulties underrepresented in the full dataset. The selection is quantitative and metric-driven: for each sample we compute per-sample statistics (agent count, yaw rate, and minimum inter-agent distance) over the full split (N = 34,149) and apply simple, interpretable thresholds to isolate challenging examples.

\subsection{Selection criteria}
The three subsets target distinct failure modes relevant to scene-level decision-making:
\begin{itemize}
  \item \textbf{Complex intersection (high agent density).} Samples with average agent count $>40$ to capture scenes with many traffic participants and high interaction complexity.
  \item \textbf{Narrow turning (high maneuvering).} Samples with average yaw rate $>0.05\ \mathrm{rad/s}$ to capture aggressive or complex turning trajectories.
  \item \textbf{Close overtake (close proximity).} Samples with average minimum inter-agent distance $<8\ \mathrm{m}$ to capture scenarios requiring tight spacing, braking, or precise lateral control.
\end{itemize}

\subsection{Subset statistics}
Table~\ref{tab:subset_stats} summarizes the resulting subsets and the full \textsc{nuScenes} statistics used for thresholding.

\begin{table}[ht]
  \centering
  \caption{Summary statistics for the selected challenging subsets and the full \textsc{nuScenes} split. ``Samples'' denotes number of selected test cases.}
  \label{tab:subset_stats}
  \begin{tabular}{lrrrr}
    \toprule
    Subset & Samples & Avg Agent Count & Avg Yaw Rate (rad/s) & Avg Min Distance (m) \\
    \midrule
    List 1 (high-density)      & 38    & 52.105263 & 0.076436 & 6.322976 \\
    List 2 (high-yaw)          & 27    & 46.259259 & 0.191537 & 12.770937 \\
    List 3 (close-prox.)       & 42    & 33.404762 & 0.014230 & 5.857343 \\
    \textsc{nuScenes} (full)   & 34,149 & 34.149960 & 0.037195 & 8.143113 \\
    \bottomrule
  \end{tabular}
\end{table}

These subsets are deliberately rare in the full corpus: List~1 comprises $38/34149 \approx 0.1113\%$, List~2 comprises $27/34149 \approx 0.0791\%$, and List~3 comprises $42/34149 \approx 0.1230\%$. Relative to the dataset averages, the subsets deviate markedly in the intended directions: List~1 contains $\approx 1.53\times$ the mean agent count, List~2 exhibits a yaw rate of $\approx 5.15\times$ the dataset mean, and List~3’s minimum inter-agent distance is $\approx 0.72\times$ the dataset mean (a reduction of $\approx 2.29\ \mathrm{m}$). These deltas confirm that the chosen thresholds successfully isolate denser, more dynamic, and closer-proximity interactions respectively.

\subsection{Evaluation protocol}
For each subset we generate closed-loop trajectories and report the following metrics:
\begin{itemize}
  \item \textbf{L2 error:} Euclidean distance between predicted and ground-truth trajectories (m).
  \item \textbf{Collision rate:} Fraction of trajectories incurring any collision during the evaluation rollout.
  \item \textbf{Success rate:} Fraction of trajectories satisfying both L2 error $<0.4\ \mathrm{m}$ and zero collisions.
\end{itemize}
These metrics quantify tracking fidelity, safety, and task completion under challenging scene-level constraints.

\subsection{Motivation and expected value}
Typical \textsc{nuScenes} driving frames are dominated by relatively straightforward motion; therefore, aggregate evaluations can mask failure modes that emerge only under strong scene-level interactions. The targeted subsets act as a discriminative benchmark to reveal whether token-level MoE routing—originally developed for tokenized language modeling—can robustly support scene-level, safety-critical decision-making in autonomous driving. Consequently, improvements observed on these subsets provide stronger evidence that a method meaningfully enhances performance in complex, safety-relevant scenarios beyond gains on the dataset average.

\section{More Qualitative Results}
This section presents additional qualitative driving scenarios to further evaluate the scene-adaptive capability of our MoE-enabled model. These cases span diverse and safety-critical environments, including signalized intersections, narrow urban lanes, multi-lane roadways, and highly interactive scenes involving both vehicles and pedestrians. 

\textbf{Scene-adaptive behavior across diverse maneuvers.}  
As illustrated in Fig.~\ref{fig:case1}--\ref{fig:case6}, our model demonstrates consistent controllability in a wide range of maneuvers, such as waiting at a red-light queue (Fig.~\ref{fig:case1}), making safe left and right turns through narrow or constrained lanes (Fig.~\ref{fig:case2}, Fig.~\ref{fig:case3}, Fig.~\ref{fig:case6}), approaching complex multi-lane intersections (Fig.~\ref{fig:case4}), and performing lane-change maneuvers near blocked traffic (Fig.~\ref{fig:case5}). In all these settings, the predicted trajectories closely follow the ground-truth motion, avoiding collisions and adapting to the dynamic constraints of each scene.

\textbf{Effectiveness of MoE in complex interactive scenarios.}
We further compare the MoE-enabled model with the non-MoE baseline in challenging scenes exhibiting strong multi-agent interactions and structural constraints. As shown in Fig.~\ref{fig:case7}, the MoE model successfully reasons about lane geometry and traffic flows, producing a valid straight-driving trajectory, whereas the non-MoE model deviates into an invalid region. Similarly, in dense right-turn scenarios with close-proximity vehicles (Fig.~\ref{fig:case9}), the MoE model maintains accurate predictions over time, while the non-MoE baseline accumulates substantial errors. These results highlight the model’s superior ability to integrate scene understanding and trajectory prediction, especially under complex and highly interactive conditions.

\begin{figure*}[t]
    \centering
    \includegraphics[width=\textwidth]{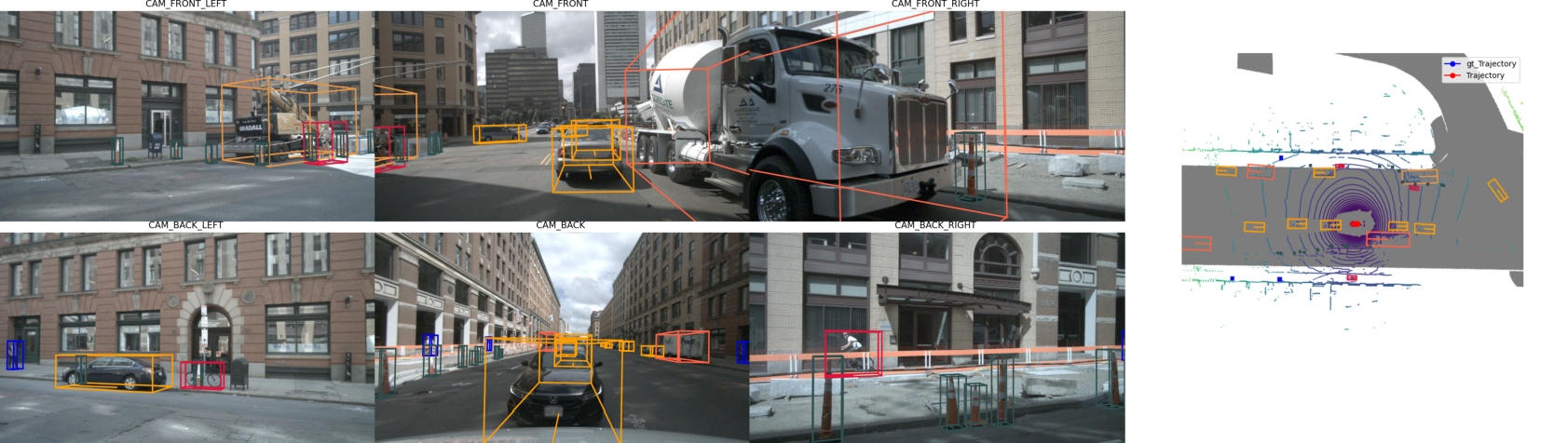}
    \caption{At a signalized urban intersection, the ego vehicle is positioned in the left lane, where traffic is stationary while waiting for the red light. A concrete mixer truck occupies the adjacent right lane and is unable to change lanes. In this scenario, our model correctly keeps the ego vehicle stationary, closely matching the ground-truth trajectory.}
    \label{fig:case1}
\end{figure*}

\begin{figure*}[t]
    \centering
    \includegraphics[width=\textwidth]{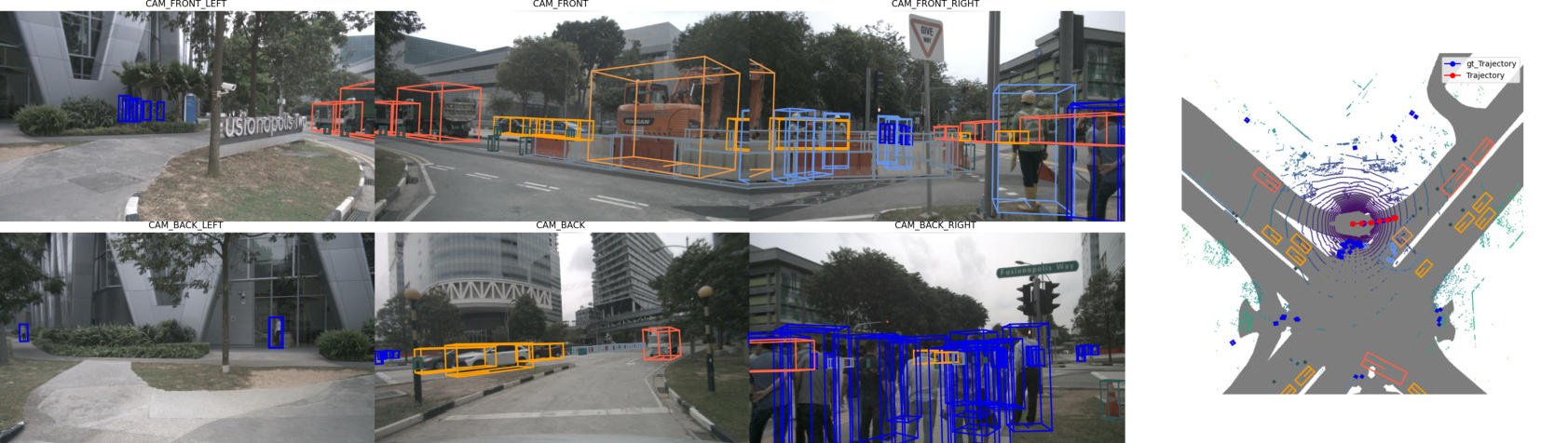}
    \caption{This case locates at an urban four-way intersection. The ego vehicle approaches a left-turn–only lane and initiates a turn into a narrow lane bordered by roadside barriers, with many pedestrians nearby. Our model executes the left turn at a cautious speed, avoiding collisions with both the barriers and pedestrians, and its predicted trajectory closely matches the ground-truth path.}
    \label{fig:case2}
\end{figure*}

\begin{figure*}[t]
    \centering
    \includegraphics[width=\textwidth]{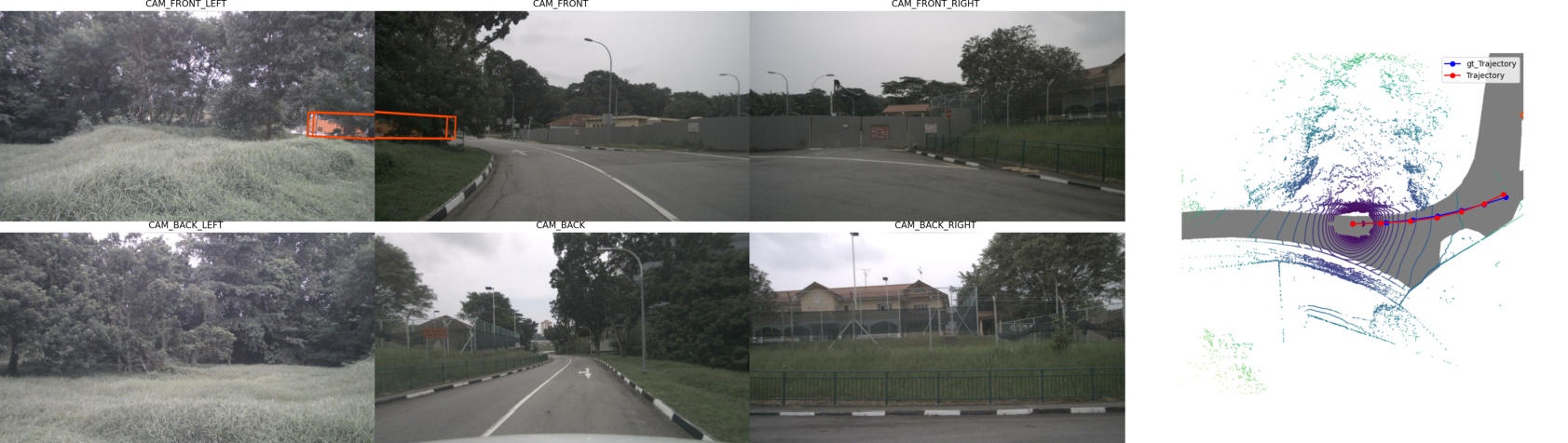}
    \caption{This case locates on a suburban two-lane road. The ego vehicle prepares to make a left turn on a narrow bidirectional roadway. In this scenario, our model successfully executes the left-turn maneuver, producing a predicted trajectory that closely matches the ground-truth path.}
    \label{fig:case3}
\end{figure*}

\begin{figure*}[t]
    \centering
    \includegraphics[width=\textwidth]{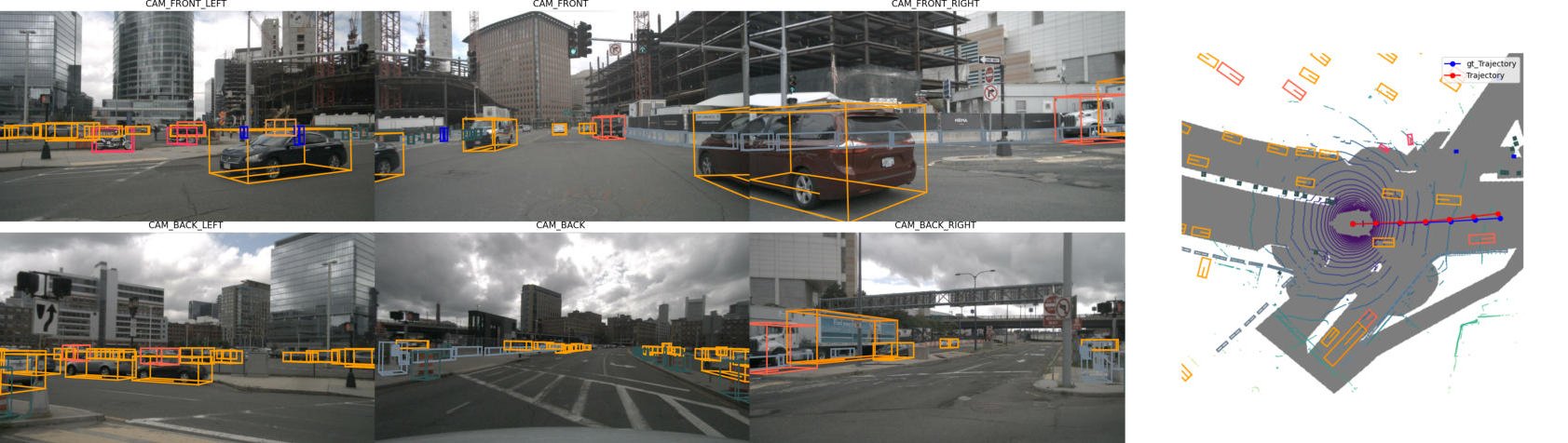}
    \caption{This case locates at a complex urban intersection. With a green traffic signal, the ego vehicle proceeds straight through the intersection while vehicles from multiple directions are either moving or waiting. In this challenging setting, our model successfully completes the straight-driving maneuver without any collisions, and its predicted trajectory closely follows the ground-truth path.}
    \label{fig:case4}
\end{figure*}

\begin{figure*}[t]
    \centering
    \includegraphics[width=\textwidth]{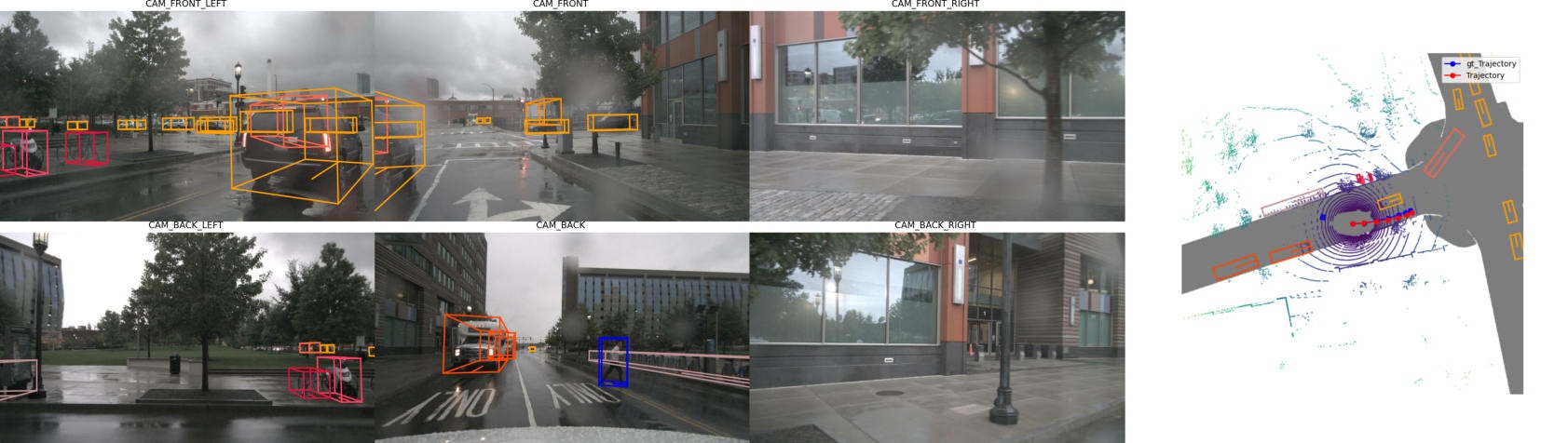}
    \caption{Case study of a lane-change maneuver near an intersection. The ego vehicle approaches a junction on a two-lane roadway, initially traveling in the left lane, which is blocked by a vehicle waiting at a red light. Two large vehicles are present behind in the adjacent right lane. In this scenario, the ego vehicle intends to merge into the right lane. Our model successfully performs the lane change, producing a predicted trajectory that closely matches the ground-truth path.}
    \label{fig:case5}
\end{figure*}

\begin{figure*}[t]
    \centering
    \includegraphics[width=\textwidth]{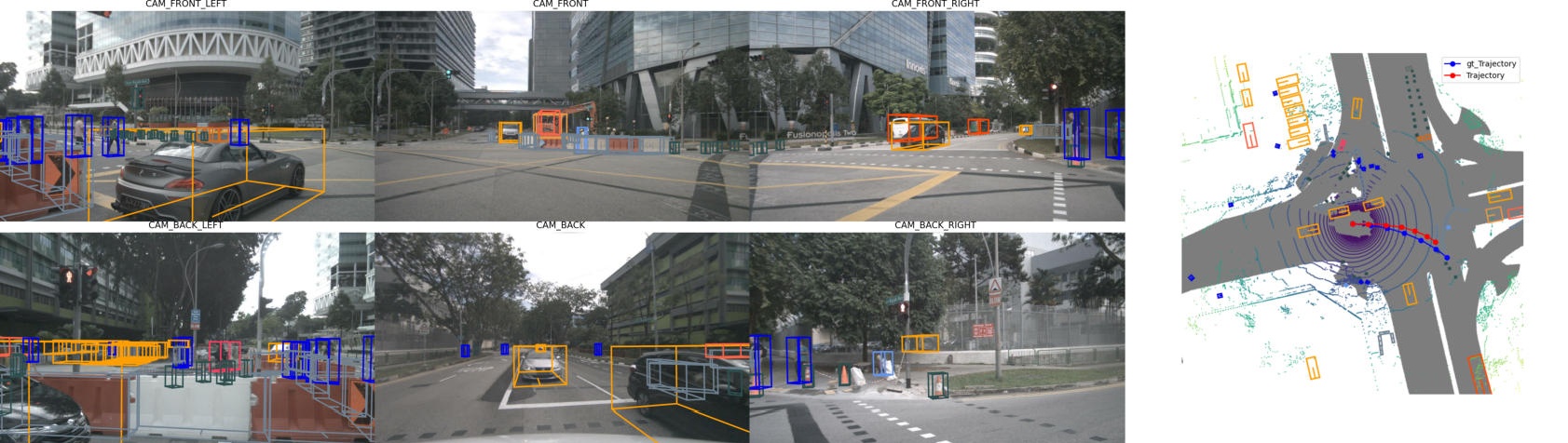}
    \caption{Case study at a large and complex urban intersection. With a green signal, the ego vehicle has passed the stop line and prepares to make a right turn into the adjacent roadway. Two nearby vehicles are present in the left lane, and additional traffic is approaching the intersection from other directions. In this scenario, our model generates a right-turn trajectory that closely follows the ground-truth path.}
    \label{fig:case6}
\end{figure*}

\begin{figure*}[t]
    \centering
    \includegraphics[width=\textwidth]{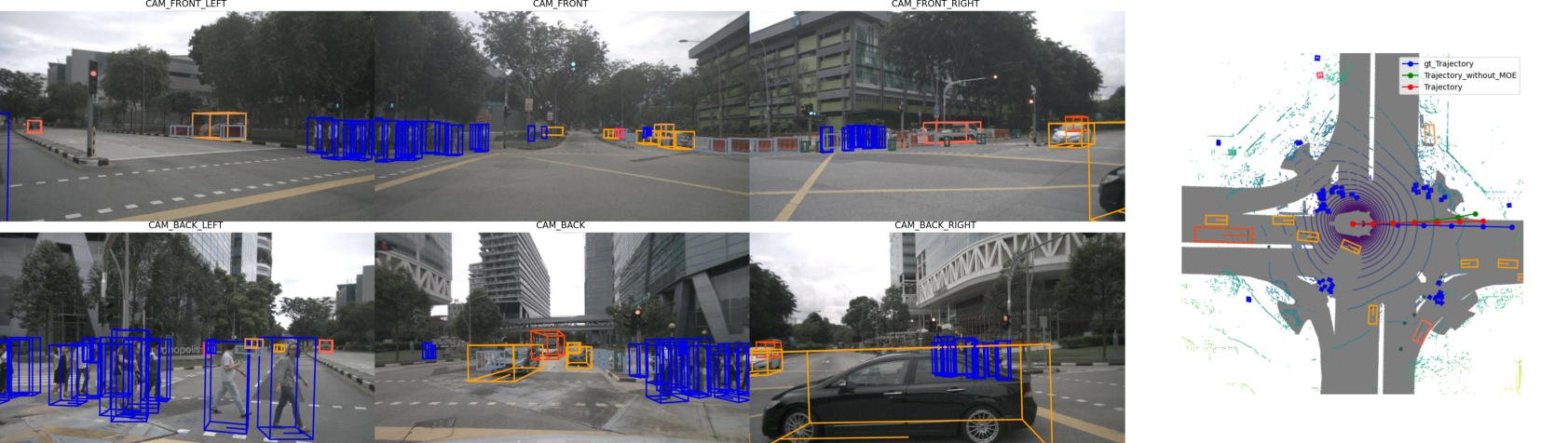}
    \caption{Case study at an urban intersection with a green signal. The ego vehicle proceeds straight through the intersection, where traffic is complex due to the presence of multiple vehicles and pedestrians, and the incoming road narrows to a single lane. We compare the trajectories generated by the MoE-enabled model and the non-MoE model. The MoE-enabled model produces a trajectory that closely matches the ground-truth path and correctly enters the target lane, whereas the non-MoE model deviates significantly and leads the vehicle into an invalid area.}
    \label{fig:case7}
\end{figure*}

\begin{figure*}[t]
    \centering
    \includegraphics[width=\textwidth]{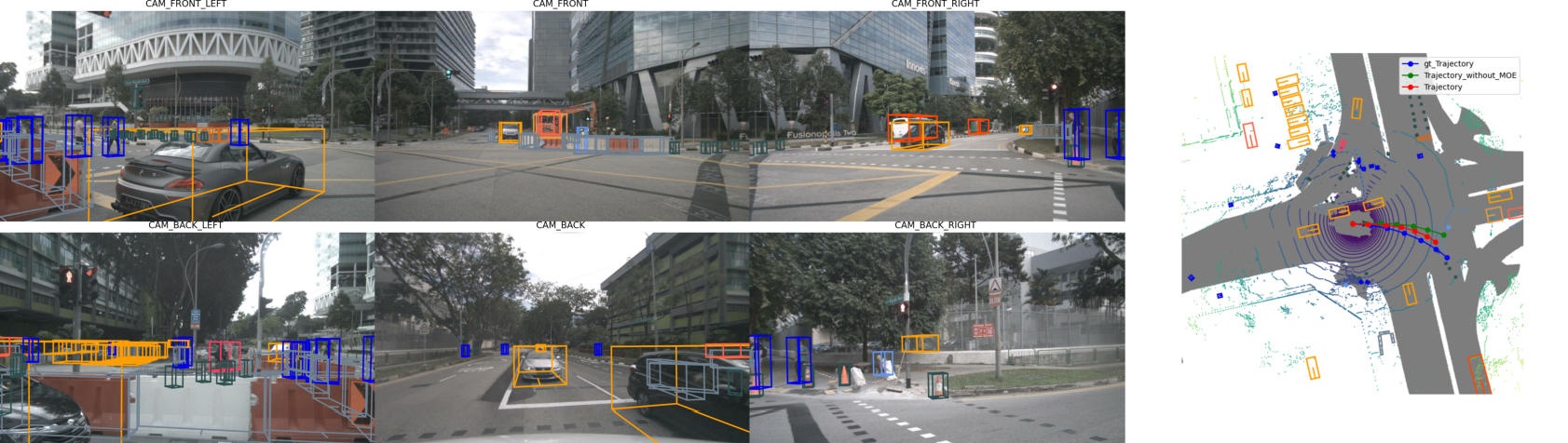}
    \caption{Case study at a complex urban intersection with many vehicles and pedestrians. The ego vehicle prepares to make a right turn, with two nearby vehicles in close proximity. In this scenario, the MoE-enabled model produces a trajectory that closely follows the ground-truth path, while the non-MoE model exhibits a larger prediction error that increases over time.}
    \label{fig:case9}
\end{figure*}

\begin{table*}[t]

\centering
\caption{Notation table for all symbols used in the SAMoE-VLA framework.}
\label{tab:mainsys}
\renewcommand{\arraystretch}{1.2}
\begin{tabular}{p{4cm} p{9.5cm}}
\hline
\textbf{Symbol} & \textbf{Description} \\
\hline

$\theta$ & Model parameter vector in flow matching. \\

$\mathbf{A} \in \{0,1\}^{B \times L \times L}$ &
Binary attention mask enforcing CMCA.\\
$\mathcal{C}$ & Index set of conditioning tokens. \\

$\mathbf{S}_{b,h,i,j}$ & Original similarity score for head $h$, sequence $b$, query $i$, key $j$. \\

$\mathbf{S}'_{b,h,i,j}$ & Masked similarity score after CMCA eager-attention masking. \\

$\mathbf{V}$ & Concatenated value states from two experts. \\

$\mathbf{O}$ & Attention output after softmax$(\mathbf{S}')\mathbf{V}$. \\

$B$ & Batch size. \\
$L$ & Total token length (conditioning tokens + action tokens). \\
$L_p$ & Token length for the planning expert (state + action). \\

\hline
\multicolumn{2}{l}{\textbf{Deformable Scene Encoder (DSE)}} \\
\hline

$\mathbf{F}_{\text{BEV}} \in \mathbb{R}^{B \times H \times W \times C}$ &
BEV feature map extracted by the BEV encoder. \\

$H, W, C$ & Spatial height, width, and feature channels of BEV map. \\

$\mathbf{M}_{\text{near}} \in [0,1]^{H \times W}$ &
Normalized distance map used to emphasize near-field regions. \\

$(c_x,c_y)$ & Ego vehicle center coordinates in BEV. \\

$\mathrm{dist}(\cdot,\cdot)$ & Euclidean distance function. \\

$K$ & Kernel size for deformable convolution. \\

$\mathbf{\Delta} \in \mathbb{R}^{B \times 2K^2 \times H_c \times W_c}$ &
Learned spatial offsets for deformable convolution sampling grid. \\

$\mathcal{P}(\cdot)$ & Convolutional predictor generating $\Delta$. \\

$\mathrm{DeformConv}(\cdot)$ & Deformable convolution operator. \\

$\mathbf{S}_{\text{BEV}}$ &
normalized scene embedding after deformable encoding. \\

$N$ & Number of flattened BEV tokens after deformable encoding. \\

$\mathbf{Q} \in \mathbb{R}^{1\times T \times C}$ &
Learnable query vectors for scene-aware routing attention. \\

$T$ & Number of routing queries. \\

$\mathbf{H}_{\text{BEV}} \in \mathbb{R}^{B \times T \times C}$ &
BEV-guided hidden representation for expert routing. \\

\hline
\multicolumn{2}{l}{\textbf{Scene-Adaptive MoE}} \\
\hline

$E$ & Number of planning experts in the MoE. \\

$\mathbf{r} \in \mathbb{R}^{B \times E}$ &
Routing logits computed from $\mathbf{H}_{\text{BEV}}$. \\

$\boldsymbol{\pi} \in \mathbb{R}^{B \times E}$ &
Softmax-normalized expert weights. \\

$\pi_e$ & Weight assigned to expert $e$ for dynamic parameter fusion. \\

$\mathbf{W}_{1}^{(e)}, \mathbf{W}_{2}^{(e)}, \mathbf{W}_{3}^{(e)}$ &
Feed-forward parameters of expert $e$. \\

$\tilde{\mathbf{W}}_i$ &
Fused FFN weights: $\tilde{\mathbf{W}}_i=\sum_{e=1}^E \pi_e \mathbf{W}_i^{(e)}$, $i \in \{1,2,3\}$. \\

$\mathbf{O}{[:L_p]}$ &
Planning-token portion of the CMCA attention output. \\

$\mathbf{R}_{\text{plan}}^*$ &
Precomputed Residual connection for planning expert. \\

$\mathbf{H}_{\mathrm{plan}}$ &
Hidden states processed by the MoE FFN. \\

$\mathbf{H}_{\mathrm{plan}}'$ &
Output hidden states after fused FFN with SiLU activation. \\

$\sigma(\cdot)$ & Activation function (SiLU). \\

$\odot$ & Element-wise multiplication. \\

\hline
\end{tabular}
\end{table*}

\begin{table*}[t]
\centering
\caption{Summary of Notation Used in Flow-Matching Implementation.}
\label{tab:flowmatchingimplementation}
\renewcommand{\arraystretch}{1.2}
\begin{tabular}{p{4cm} p{14.5cm}}
\hline
\textbf{Symbol} & \textbf{Description} \\
\hline

$\mathbf{B}\in\mathbb{R}^{H\times W\times C_b}$ 
& BEV feature map extracted from multi-view images. \\

$\mathbf{DSE}:\mathbb{R}^{C_b}\to\mathbb{R}^{D}$ 
& Deformable Scene Encoder mapping BEV to $D$-dimensional. \\

$\mathbf{H}_{\mathrm{BEV}}\in\mathbb{R}^{(HW)\times D}$ 
& Flattened BEV tokens after DSE encoding. \\

$\mathbf{h}$, $\mathbf{c}$, $\mathbf{l}$ 
& Ego-history, CAN bus readings, and low-level control features. \\

$\psi_{\mathrm{his}},\psi_{\mathrm{bus}},\psi_{\mathrm{ego}}$ 
& Linear projections for auxiliary context embeddings. \\

$\mathbf{S}\in\mathbb{R}^{L_s\times D}$ 
& Concatenated state/context tokens. \\

$\mathbf{z}$ 
& BEV latent tokens used as part of conditioning. \\

$\mathbf{L}$ 
& Language instruction tokens. \\

$\mathbf{W}$ 
& World tokens processed by world-language expert. \\

$\mathcal{C}=\{\mathbf{z},\mathbf{H}_{\mathrm{bev}},\mathbf{S},\mathbf{L},\mathbf{W}\}$ 
& Overall conditioning context for the planning expert. \\

\midrule

$\mathbf{a}\in\mathbb{R}^{K\times d_a}$ 
& Ground-truth future actions (with $d_a=2$). \\

$K$ 
& Planning horizon (6 in nuScenes). \\

$\boldsymbol{\epsilon}\sim\mathcal{N}(\mathbf{0},\mathbf{I})$ 
& Gaussian noise sampled for flow-matching interpolation. \\

$t\sim\mathrm{Beta}(1.5,1.0)$ 
& Random time used to define convex interpolation. \\

$\mathbf{x}_t=t\boldsymbol{\epsilon}+(1-t)\mathbf{a}$ 
& Noisy intermediate action at time $t$. \\

$\mathbf{u}_t=\boldsymbol{\epsilon}-\mathbf{a}$ 
& Target instantaneous velocity in flow matching. \\

$\gamma(t)\in\mathbb{R}^{D}$ 
& Sinusoidal time embedding with period range $[4\times10^{-3},4]$. \\

$\psi_{\mathrm{act}}$ 
& Linear projection transforming actions to $D$-dimensional tokens. \\

$\mathbf{E}_{\mathrm{act}}\in\mathbb{R}^{K\times D}$ 
& Action tokens derived from $\mathbf{x}_t$. \\

$\mathbf{E}_{\mathrm{time}}\in\mathbb{R}^{K\times D}$ 
& Repeated time embedding for each of the $K$ steps. \\

$\mathbf{E}_{\mathrm{suf}}\in\mathbb{R}^{K\times D}$ 
& Suffix/action–time tokens after fusion through an MLP. \\

\midrule

$\mathbf{X}\in\mathbb{R}^{(L_s+K)\times D}$ 
& Input sequence to planning expert $f_\theta$. \\

$f_\theta$ 
& Transformer-based planning expert producing hidden states. \\

$\psi_{\mathrm{out}}:\mathbb{R}^{D}\to\mathbb{R}^{d_a}$ 
& Projection head mapping hidden states to predicted velocities. \\

$v_\theta(\mathbf{x}_t,t,\mathcal{C})\in\mathbb{R}^{K\times d_a}$ 
& Model-predicted time-conditioned velocity field. \\

$\mathcal{L}_{\mathrm{FM}}$ 
& Flow-matching loss between $v_\theta$ and $\mathbf{u}_t$. \\

\midrule

$\mathbf{x}_1\sim\mathcal{N}(\mathbf{0},\mathbf{I})$ 
& Noise initialization during inference. \\

$\Delta t=-1/N$ 
& Reverse-time step for Euler ODE solver. \\

$\mathbf{x}_{t+\Delta t}=\mathbf{x}_t+\Delta t\cdot v_\theta(\mathbf{x}_t,t,\mathcal{C})$ 
& Euler integration update rule. \\

$\hat{\mathbf{a}}=\mathbf{x}_0$ 
& Final predicted trajectory returned by the ODE solver. \\

\bottomrule
\end{tabular}
\end{table*}

\section{Additional Computational Efficiency and Deployment Analysis}
\label{sec:appendix_efficiency}

To complement the planning accuracy and closed-loop evaluations presented in the main paper, we further analyze the computational efficiency and memory characteristics of the proposed Scene-Adaptive MoE (SA-MoE) under a realistic deployment setting. Considering the practical constraints of autonomous driving systems, all measurements are conducted on a \emph{single GPU} with mixed-precision inference (FP16), which reflects a single-card on-vehicle or edge deployment scenario rather than multi-GPU data-center configurations.

\subsection{Experimental Setup}

We benchmark three modules under identical configurations:
(i) a conventional token-level sparse MoE block,
(ii) the proposed scene-adaptive soft-weighted SA-MoE layer,
and (iii) a single expert block extracted from the MoE architecture for reference.
All results report average inference latency, token throughput, estimated FLOPs, and GPU memory statistics.

\subsection{Latency and Throughput Comparison}

Under identical parameter budgets ($\sim$276M parameters), the conventional sparse MoE block achieves an average inference time of $46.17$ ms with a throughput of $44{,}356$ tokens/s. In contrast, SA-MoE reduces latency to $43.89$ ms and increases throughput to $46{,}667$ tokens/s, yielding a relative speedup of $1.05\times$.

Although the parameter counts are nearly identical (276.84M vs.\ 276.83M), SA-MoE achieves substantially lower computational cost per token. The estimated FLOPs per token decrease from $6.92\times10^{7}$ (sparse MoE) to $3.51\times10^{7}$ (SA-MoE), corresponding to a $1.97\times$ reduction. This reduction arises because SA-MoE performs expert weight fusion once per layer and executes a single merged FFN, whereas token-level sparse routing activates multiple experts per token and incurs additional routing overhead.

\subsection{Memory Behavior}

A notable difference appears in memory allocation patterns. The sparse MoE block exhibits a memory increase of $47.8$ MB during inference, while SA-MoE shows only $0.2$ MB additional steady-state memory usage. The relative steady-state memory delta is thus reduced by over two orders of magnitude.

Although SA-MoE records a higher transient CUDA peak memory (due to temporary expert weight merging buffers), the post-execution memory footprint returns to the same level as the sparse MoE implementation. This behavior indicates that SA-MoE does not increase persistent deployment memory requirements, which is critical for embedded automotive hardware.



\subsection{Computational Analysis of the Deformable Scene Encoder (Router Module)}

We further isolate and benchmark the \emph{Deformable Scene Encoder (DSE)}, which serves as the scene-adaptive routing module in SA-MoE. This component is responsible for extracting BEV-conditioned routing representations and generating expert fusion weights. Since routing is computed once per forward pass and reused across transformer layers, its computational overhead is critical for real-time deployment.

Under the same single-GPU mixed-precision (FP16) setting, the DSE module contains $23.07$M parameters. The average inference latency is measured at $2.883$ ms, corresponding to a throughput of $710{,}323$ tokens per second.

The estimated total FLOPs of the DSE forward pass are $4.73\times10^{10}$, yielding approximately $2.31\times10^{7}$ FLOPs per token. Although the absolute FLOP count appears non-negligible, it is important to emphasize that DSE is executed only once per planning step and its output is reused across all SA-MoE layers. Consequently, its amortized per-layer computational contribution remains small.

In terms of memory behavior, no measurable persistent memory increase is observed (0.0 MB delta between pre- and post-execution states). The CUDA peak memory reaches $1168.3$ MB during execution and returns to $1124.3$ MB afterward, indicating that temporary buffers are released properly and that DSE introduces no additional steady-state memory burden.

These results confirm that the Deformable Scene Encoder introduces minimal latency and negligible persistent memory overhead, validating the design choice of scene-level routing. The router remains computationally lightweight relative to the full planning expert, making SA-MoE suitable for single-card autonomous driving deployment scenarios.







\subsection{Hardware-Aware Efficiency Analysis under Autonomous Driving Deployment for SA-MoE}

While Mixture-of-Experts (MoE) architectures have become a standard scaling strategy in large language models (LLMs), their design is typically optimized for large-scale distributed training and inference environments. In most LLM deployments, MoE layers are executed across multiple GPUs or nodes, where different experts are distributed across devices and token routing is implemented through collective communication primitives such as \texttt{all-to-all}. This distributed setting allows sparse expert activation to reduce the effective compute per token while maintaining large model capacity.

However, the deployment constraints of autonomous driving systems differ fundamentally from those of large-scale cloud inference. Production autonomous vehicles typically operate under strict hardware budgets and latency constraints, where perception, planning, and control modules must share limited computational resources. In practice, most autonomous driving stacks are deployed on a single high-performance GPU (e.g., automotive-grade NVIDIA platforms), rather than multi-GPU clusters. Under this hardware regime, the efficiency characteristics of conventional MoE architectures change significantly.

In token-level sparse MoE designs, such as those commonly used in LLMs, routing decisions are made independently for each token and each layer. Although only a subset of experts is activated per token, the system must still compute routing scores and dynamically gather expert outputs at the token granularity. When executed on a single GPU, this leads to several inefficiencies. First, token-wise dynamic routing introduces irregular memory access patterns and kernel fragmentation, which reduce GPU utilization. Second, repeated routing computation across layers increases control overhead. Third, the fine-grained dispatch mechanism prevents efficient batching of expert computations, limiting the throughput benefits that sparse activation aims to achieve.

Soft MoE methods partially alleviate routing instability by replacing discrete expert selection with soft weighted aggregation. Nevertheless, Soft MoE still operates at the token level and requires computing outputs for multiple experts simultaneously before aggregation. As a result, the computational cost remains close to dense MoE in practice, and the hardware efficiency improvements are limited under single-device deployment.

In contrast, the proposed \textbf{Scene-Adaptive Mixture-of-Experts (SA-MoE)} is explicitly designed with the hardware characteristics of autonomous driving platforms in mind. Instead of performing token-level routing, SA-MoE introduces a scene-level routing mechanism implemented by the Deformable Scene Encoder (DSE). The router processes the BEV scene representation once that are shared across tokens within the planning module. This design leads to several practical advantages.

First, \textbf{routing overhead is amortized across the entire sequence}. Since the routing weights are computed per scene rather than per token, the total routing computation becomes negligible relative to the expert networks themselves. As demonstrated in the empirical analysis, the DSE module introduces only $2.883$ ms of latency with $23.07$M parameters, representing a small fraction of the total MoE computation.

Second, \textbf{expert computation becomes structurally regular}. Because all tokens share the same scene-level expert mixture, the forward pass can be implemented as a deterministic weighted combination of experts rather than dynamic token dispatch. This eliminates the need for token-wise gather/scatter operations and improves GPU kernel efficiency, particularly on single-device inference.

Third, \textbf{memory behavior becomes more predictable}. Token-level sparse MoE often requires temporary buffers for token regrouping and expert batching, which increases peak memory consumption and fragmentation. In contrast, SA-MoE maintains stable memory usage since expert outputs are computed in-place and fused using precomputed scene-level weights.

These hardware-aware design choices align naturally with the computational constraints of autonomous driving systems. While traditional LLM MoE architectures assume multi-device distributed environments, SA-MoE targets the realistic deployment scenario where Mixture of Experts must run on a single GPU with strict real-time requirements.
To further understand the practical efficiency difference between conventional sparse MoE and the proposed SA-MoE under autonomous driving deployment, we analyze the computational behavior from the perspective of a typical automotive GPU platform. Modern autonomous driving systems are commonly deployed on embedded GPU accelerators such as the NVIDIA Orin series, which provides approximately $275$ TOPS of AI compute with $32$ GB LPDDR5 memory bandwidth of roughly $204$ GB/s. In contrast to large-scale LLM serving infrastructure, such platforms operate under a \textbf{single-device constraint}, meaning all experts and routing logic must reside on the same GPU and execute sequentially or partially concurrently.

In conventional sparse MoE architectures, the feed-forward layer is replaced with $N$ experts, and each token activates the top-$k$ experts through a routing function. Formally, the output of a sparse MoE layer can be written as
\[
y_i = \sum_{j \in \text{TopK}(g(x_i))} g_j(x_i) \, E_j(x_i),
\]
where $x_i$ denotes the $i$-th token representation, $E_j(\cdot)$ is the $j$-th expert network, and $g_j(x_i)$ is the router weight. Although only $k$ experts contribute to the final output for each token, the system must still maintain all $N$ experts in GPU memory and dynamically dispatch tokens to the selected experts.

Under single-GPU execution, this routing mechanism leads to a fundamental hardware inefficiency. Since tokens routed to different experts are typically interleaved within the sequence, the GPU must either (1) perform repeated kernel launches for each expert or (2) reorganize tokens through gather/scatter operations before batched computation. In both cases, all experts must be resident in memory and participate in the scheduling process, which introduces additional overhead in both memory bandwidth and kernel execution.

Let the hidden dimension be $d$ and the intermediate FFN dimension be $m$. The computational cost of a single expert can be approximated as
\[
\text{FLOPs}_{\text{expert}} \approx 2dm .
\]
For a sparse MoE layer with $N$ experts and top-$k$ routing, the theoretical compute per token is
\[
\text{FLOPs}_{\text{sparse}} \approx k \cdot 2dm .
\]
However, in a single-GPU environment, the effective runtime cost is closer to
\[
\text{FLOPs}_{\text{runtime}} \approx k \cdot 2dm + \mathcal{O}(N \cdot d),
\]
where the second term represents routing, token dispatch, and memory movement overhead proportional to the total number of experts. As $N$ increases, this overhead becomes increasingly dominant because each expert must remain resident in GPU memory and participate in the routing pipeline.

In contrast, the proposed SA-MoE replaces token-level routing with \textbf{scene-level expert fusion}. The routing weights are computed once using the Deformable Scene Encoder (DSE), producing a scene-conditioned expert mixture
\[
y_i = \sum_{j=1}^{N} \alpha_j(s) \, E_j(x_i),
\]
where $s$ denotes the scene representation and $\alpha_j(s)$ are normalized fusion weights. Since the mixture weights are shared across tokens, expert outputs can be computed sequentially or fused efficiently without dynamic token dispatch.

Under this formulation, the routing complexity becomes independent of the token sequence length:
\[
\text{FLOPs}_{\text{router}} \approx \mathcal{O}(d_s d),
\]
where $d_s$ is the BEV scene feature dimension. In our implementation, the router module (DSE) contains only $23.07$M parameters and introduces an average latency of $2.883$ ms, which is negligible relative to the full MoE computation.

More importantly, the removal of token-level routing eliminates the gather/scatter overhead and irregular memory access patterns that degrade GPU utilization. On single-device inference hardware, this leads to more efficient kernel execution and improved throughput. Empirically, as shown in the performance analysis, SA-MoE achieves approximately $1.05\times$ lower latency and nearly $2\times$ reduction in FLOPs per token compared to the baseline sparse MoE block, while maintaining nearly identical parameter capacity.

From a system perspective, this design better aligns with the hardware constraints of embedded autonomous driving platforms. While traditional sparse MoE architectures are optimized for distributed multi-GPU inference, SA-MoE leverages scene-level routing to reduce dynamic scheduling overhead, enabling more efficient execution on a single GPU accelerator.


\section{Notation Table}
As shown in Table~\ref{tab:mainsys}, Table~\ref{tab:flowmatchingimplementation} and Table~\ref{tab:notation_worldmodeling}, we further provide notation tables to increase the readability of the paper and standardize the representation of symbols in the paper.

\begin{table*}[t]
\centering
\caption{Notation used in implementation of World Modeling.}
\label{tab:notation_worldmodeling}
\renewcommand{\arraystretch}{1.2}
\begin{tabular}{p{4cm} p{14.5cm}}
\hline
\textbf{Symbol} & \textbf{Description} \\
\hline
$\mathbf{F}_{BEV}$ &  BEV feature map. \\
$\mathrm{DownSampleCross}$ & Lightweight BEV downsampling module. \\
$\mathbf{E}^{\mathrm{ae}}_t$ & Reconstruction target for autoencoding regularization. \\[3pt]

$N_q$ & Number of world queries. \\
$\mathbf{Q}$ & Pooled BEV queries for LLM input. \\
$\mathbf{f}_{\mathrm{frame}}(k)$ & Learnable soft prompt for the $k$-th future frame. \\
$\mathbf{W}^{(k)}$ & Future-aware world tokens per frame. \\
$\mathbf{E}^{\mathrm{llm}}_t$ & LLM-refined current BEV. \\
$\mathbf{Q}^{(k)}$ & Future latent queries returned by LLM. \\
$\ell_{\mathrm{chat}}$ & LLM text output. \\
$\widetilde{\mathbf{E}}^{\downarrow}_{t+k}$ & Coarse synthesized BEV for future frame $t+k$. \\
$\mathbf{E}^{\downarrow}_{t+k}$ & Ego-attention refined future BEV tokens. \\[3pt]

$\mathbf{V}_{t+k}$ & Lifted 3D voxel features before convolution. \\
$\mathbf{U}_{t+k}$ & Unified 3D voxel grid for SDF rendering. \\
$Z$ & Vertical voxel bins. \\
$\Delta=(\Delta_x,\Delta_y,\Delta_z)$ & Voxel size. \\
$\mathrm{pc\_range}$ & 3D coverage range of the voxel grid. \\[3pt]

$P_{t+k}$ & Ground-truth point cloud for frame $t+k$. \\
$\mathcal{R}_{t+k}$ & Set of LiDAR rays for frame $t+k$. \\
$\mathbf{o}_n,\mathbf{d}_n,r_n$ & Ray origin, direction, and ground-truth range. \\[3pt]

$\mathbf{p}_i$ & Sampled 3D point along a ray. \\
$\mathbf{f}_i$ & Interpolated voxel feature at $\mathbf{p}_i$. \\
$s_i$ & Predicted SDF value at sample $i$. \\
$\phi_{\mathrm{SDF}}$ & MLP predicting SDF. \\
$\sigma_t(\cdot)$ & Learnable sigmoid modulation in SDF rendering. \\
$\alpha_i$ & Opacity between adjacent samples. \\
$T_i$ & Ray transmittance. \\
$w_i$ & Accumulated rendering weight. \\
$\widehat{r}_n^{(k)}$ & Rendered depth for frame $t+k$. \\
$\widehat{\mathbf{x}}_n^{(k)}$ & Reconstructed 3D point. \\
$s$ & Global scale factor aligning SDF metric to LiDAR. \\[3pt]

$\mathcal{X}^{(k)},\widehat{\mathcal{X}}^{(k)}$ & Ground-truth and predicted point sets. \\
$\mathcal{L}_{\mathrm{depth}}$ & Per-ray depth regression loss. \\
$\mathcal{L}_{\mathrm{chamfer}}$ & Set-level Chamfer distance. \\
$\mathcal{L}_{\mathrm{render}}$ & Auxiliary rendering losses. \\
$\mathcal{L}_{\mathrm{chat}}$ & LLM conversational supervision. \\
$\lambda_{\mathrm{*}}$ & Loss weights. \\
\bottomrule
\end{tabular}
\end{table*}

\end{document}